\documentclass[11pt,a4paper]{article}

\usepackage{amsmath}
\usepackage{fullpage}
\usepackage{amsfonts,amssymb,amsmath,amsthm,xspace}
\usepackage{graphicx}
\usepackage{color}
\usepackage{cite}
\usepackage{url}
\usepackage[footnotesize]{caption}
\usepackage[tight,footnotesize]{subfigure}

\usepackage[noend]{algpseudocode}
\usepackage{algorithmicx}
\usepackage{algorithm}

\usepackage{pdfsync}

\newcommand{\arXiv}{arxiv:}
\newcommand{\bb}{\mathbb}
\newcommand{\Rbb}{\mathbb{R}}
\newcommand{\Cbb}{\mathbb{C}}
\newcommand{\Nbb}{\mathbb{N}}

\newcommand{\tinv}[1]{{\textstyle\frac{1}{#1}}}
\newcommand{\ud}{\mathrm{d}} 
\renewcommand{\leq}{\leqslant}
\renewcommand{\geq}{\geqslant}
\renewcommand{\tilde}{\widetilde}

\DeclareMathOperator{\st}{{s.\!t.}\xspace}
\DeclareMathOperator{\Ker}{Ker}

\DeclareMathOperator{\supp}{supp}

\DeclareMathOperator{\diag}{diag}
\DeclareMathOperator{\prox}{prox}

\DeclareMathOperator{\Id}{\mathbb{I}\hspace{-1.24mm}\mathbb{I}}

\DeclareMathOperator{\intset}{int}

\DeclareMathOperator*{\argmin}{\arg\min}

\newcommand{\bs}{\boldsymbol}
\newcommand{\cl}{\mathcal}
\newcommand{\ie}{\emph{i.e.}, }
\newcommand{\eg}{\emph{e.g.}, }

\newcommand{\imap}{{\mathsf n}}
\newcommand{\imapv}{\boldsymbol{\mathsf n}}

\newcommand{\cmapv}{\boldsymbol{u}}

\renewcommand{\vec}{\boldsymbol}

\newcommand{\bbb}{|\!|\!|}

\newtheorem{theorem}{Theorem}[section]

\newtheorem{lemma}[theorem]{Lemma}

\theoremstyle{definition}

\newcommand{\scp}[2]{\langle #1, #2 \rangle}

\title{Compressive Optical Deflectometric Tomography:\\
A Constrained Total-Variation Minimization Approach.\\[-1cm]}
\date{}

\begin{document}

\maketitle

\centerline{\scshape Adriana Gonz\'alez\footnote{
$\{$adriana.gonzalez, laurent.jacques,
christophe.devleeschouwer$\}$@uclouvain.be, pantoine@lambda-x.com}}
\medskip
{\footnotesize
 \centerline{ICTEAM/IMCN, Universit\'e catholique de Louvain. Louvain-la-Neuve, Belgium.}
}
\medskip

\centerline{\scshape Laurent Jacques }
\medskip
{\footnotesize
 \centerline{ICTEAM, Universit\'e catholique de Louvain. Louvain-la-Neuve, Belgium.}
} 

\medskip

\centerline{\scshape Christophe De Vleeschouwer}
\medskip
{\footnotesize
 \centerline{ICTEAM, Universit\'e catholique de Louvain. Louvain-la-Neuve, Belgium.}
}

\medskip

\centerline{\scshape Philippe Antoine}
\medskip
{\footnotesize
 \centerline{IMCN, Universit\'e catholique de Louvain. Louvain-la-Neuve, Belgium.}
}
\medskip
{\footnotesize
 \centerline{Lambda-X SA, Rue de l'Industrie 37. Nivelles, Belgium.}
}

\bigskip
\centerline{\today}

\begin{abstract}
  Optical Deflectometric Tomography (ODT) provides an accurate
  characterization of transparent materials whose complex surfaces
  present a real challenge for manufacture and control. In ODT, the
  refractive index map (RIM) of a transparent object is reconstructed
  by measuring light deflection under multiple orientations. We show
  that this imaging modality can be made \emph{compressive}, \ie a
  correct RIM reconstruction is achievable with far less observations
  than required by traditional minimum energy (ME) or Filtered Back Projection (FBP)
  methods. Assuming a cartoon-shape RIM model, this reconstruction is
  driven by minimizing the map Total-Variation under a fidelity
  constraint with the available observations. Moreover, two other
  realistic assumptions are added to improve the stability of our
  approach: the map positivity and a frontier condition. Numerically,
  our method relies on an accurate ODT sensing model and on a
  primal-dual minimization scheme, including easily the sensing
  operator and the proposed RIM constraints. We conclude this paper by
  demonstrating the power of our method on synthetic and experimental
  data under various compressive scenarios. In particular, the
  potential compressiveness of the stabilized ODT problem is demonstrated by
  observing a typical gain of 24\,{\rm dB} compared to ME and of 30\,{\rm dB} compared to FBP at only 5\%
  of 360 incident light angles for moderately noisy sensing.   
\end{abstract}

\pagenumbering{arabic}

\section{INTRODUCTION}
\label{sec:intro}

Optical Deflectometric Tomography (ODT) is an imaging modality that
aims at reconstructing the spatial distribution of the refractive
index of a transparent object from the deviation of the light passing
through the object. By reconstructing the refractive index map (RIM)
we are able to optically characterize transparent materials like
optical fibers or multifocal intra-ocular lenses (IOL), which is of
great importance in manufacture and control processes.

ODT is attractive for its high sensitivity and effective detection of
local details and object contours. The technique is insensitive to
vibrations and it does not require coherent light. Compared
to interferometry, ODT can be applied to objects with higher
refractive index difference with respect to the surrounding solution.

The deflectometer used for the data acquisition is based on the
phase-shifting Schlieren method~\cite{Emmanuel_2010}. For each
orientation of the sample, two-dimensional~(2-D) mappings
of local light deviations are measured. As these light deviations are
proportional to the transverse gradient of the RIM integrated along
the light ray, ODT is able to reconstruct the RIM from the angle
measurements.

First works on deflectometric
tomography~\cite{beghuin2010optical,Faris1987} have focused on the use
of common reconstruction techniques like the Filtered Back Projection
(FBP). They have proved that FBP provides an accurate estimation of
the RIM when sufficient amount of object orientations is
available. However, when we consider a scenario with a limited number
of light incident angles, and in the presence of noise in ODT
observations, FBP induces the apparition of spurious artifacts in the
estimated RIM, lowering the reconstruction quality.

Inspired by the Compressed Sensing (CS)
paradigm~\cite{Candes2004,donoho2006cs}, which demonstrates that few
measurements are enough for an accurate reconstruction of low
complexity (\eg sparse)
signals, recent works in ODT have started to exploit sparsity to
reconstruct the RIM from few acquisitions, \ie in the presence of an
ill-posed inverse problem. Foumouo et al.~\cite{Emmanuel_2010} and
Antoine et al.~\cite{optimess2012} have used a sparse representation
in a B-splines basis and regularized the problem using the
$\ell_1$-norm. The reconstruction was performed using iterative
schemes and the results show that, although the method is capable of
reproducing the shape of spatially localized objets, the image
dynamics is not well preserved and the borders are smoothened.

Although sparsity based methods are new in ODT, they have been used in
other applications, such as Magnetic Resonance Imaging
(MRI)~\cite{lustig2007sparse}, Absorption Tomography
(AT)~\cite{Sidky2011,Ritschl2011}, Radio
Interferometry~\cite{wiaux2009compressed} and Phase-Contrast
Tomography~\cite{cong2011,cong2012}, for image reconstruction under
partial observation models. More details are given in
Sec.~\ref{sec:state-of-the-art}.

An additional problem that rises in all physical applications is the
estimation of the sensing operator that fits better the physical
acquisition. Most operators present a non ideal behavior, which
conditions the numerical methods to solve the inverse problem. In
tomographic problems, this operator requires to map spatial data in a
Cartesian grid to Fourier data in a Polar grid. Previous works have
used \emph{gridding} techniques to interpolate data from a polar to a
cartesian or pseudo polar grid before applying the Fourier
Transform~\cite{SPIE_2011,zhang2007}. However, the error introduced
when using these techniques is not bounded and introduces an
uncontrolled distortion.

\subsection{Contribution} 

In this work, we show that the ODT can be made both compressive and robust to
Gaussian noise. In the context of a simplified 2-D
  sensing model, we propose a constrained method based on the
minimization of the Total-Variation (TV) norm that provides high
quality reconstruction results even when few acquisitions are
available and in the presence of high level of Gaussian noise.

This is motivated by assuming the RIM composed of slowly varying areas separated by
sharp transitions corresponding to material interfaces. Such a
behavior is known to be represented by spatial functions having
\emph{bounded variations} (small TV norm), such as those following a cartoon-shape
model. This also distinguishes our work from two previous studies 
focused on continuous RIM decomposition with a B-splines
basis. Deflection integrals were there estimated in the spatial domain using complex numerical
methods leading to a smoother RIM
estimation \cite{optimess2012,Emmanuel_2010}.

To account for the noise and the raw data consistency, we
add an $\ell_2$ data fidelity term adapted to Gaussian and uniformly
distributed noise. Moreover, the proposed method offers the flexibility to work with more than one
constraint. In contrast with~\cite{iTWIST12}, we add here two more constraints based on
general prior information on the RIM in order to converge to an
optimal solution: \emph{(i)} the RIM is positive everywhere and \emph{(ii)} the
object is completely contained in the experimental field of
view. These extra constraints provably guarantee the unicity of the ODT solution.

As for the operator, we use the NFFT algorithm\footnote{This toolbox
  is freely available here
  \url{http://www-user.tu-chemnitz.de/~potts/nfft/}}, a fast
computation of the non-equispaced DFT. This algorithm provides an
efficient estimation of the polar Fourier transform with a controlled
distortion regarding the true polar transform.

The compressive ODT problem is solved by means of the primal-dual
algorithm proposed by Chambolle et al.~\cite{ChambollePock_2011},
complemented by an adaptive choice of its parameters improving the global
convergence speed, as proposed in a
recent work of Goldstein et al. \cite{goldstein2013}. The
results are compared with a minimum energy (ME) method and a
  common FBP approach, showing clear gain compared to these two approaches in terms of compressiveness, noise robustness
and reconstruction quality.

\subsection{Outline} 

In Sec.~\ref{sec:problem} we provide a brief background on optical
deflectometric tomography, describing also the experimental setup used
for the data acquisition. Then, the ODT discrete model is presented in
Sec.~\ref{sec:processing}. In Sec.~\ref{sec:state-of-the-art} we depict
related works on tomographic reconstruction, which provide a basis on
the methods adopted to recover the RIM: the commonly used FBP method,
a standard minimum energy (ME) approach (or penalized least square)
and the proposed regularized method coined TV-$\ell_2$. These
methods are described in Sec.~\ref{sec:reconstruction}. In
Sec.~\ref{sec:noise_estimation} we present the identification and
estimation of the noise in both synthetic and experimental data, and
the analysis of the noise impact when comparing common absorption
tomography and deflection tomography. Sec.~\ref{sec:gcp} presents the
numerical methods used to recover the RIM from the noisy
measurements by means of the proposed regularized
formulation. Finally, in Sec.~\ref{sec:experiments} some reconstruction
results are shown, focusing first on the comparison between common
tomographic and ODT reconstructions, and then on the comparison of the
reconstruction methods on the basis of reconstruction quality and
convergence for both synthetic and experimental data.

\subsection{Conventions} 

Most of domain dimensions are denoted
by capital roman letters, \eg $M, N, \ldots$ Vectors and matrices are associated to bold
symbols while lowercase light letters are associated to scalar
values, \eg $\bs \Phi \in \Rbb^{M\times N}$ or $\bs u \in \Rbb^M$. The $i^{\rm th}$
component of a vector $\bs u$ reads either $u_i$ or $(\bs u)_i$,
while the vector $\bs u_i$ may refer to the $i^{\rm
  th}$ element of a vector set. The
identity matrix in $\Rbb^D$ reads $\bs\Id^D$, or simply $\Id$
when its dimension is clear from the context. The vectors of
  ones and zeros in $\Rbb^D$ are denoted $\bs 1$ and $\bs 0$, respectively. The set of indices in
$\Rbb^D$ is $[D]=\{1,\,\cdots,D\}$. Scalar product between two vectors
$\bs u,\bs v \in \Rbb^{D}$ for some dimension $D\in\Nbb$ is denoted
equivalently by $\bs u^T \bs v = {\bs u \cdot \bs v}=\scp{\bs u}{\bs
  v}$. For any $p\geq 1$, the $\ell_p$-norm of $\bs u$ is $\|\bs u\|_p^p = \sum_i |u_i|^p$ with
$\|\!\cdot\!\|=\|\!\cdot\!\|_2$. The notation $\|\bs u\|_0 = \# \supp
\bs u$ represents the cardinality of the support $\supp \bs u = \{i:
u_i \neq 0\} \subset [D]$.  For a subset $\cl S \subset [D]$, given
$\bs u \in\Rbb^D$ and $\bs \Phi\in \Rbb^{D'\times D}$, $\bs u_{\cl
  S}\in \Rbb^{\#\cl S}$ (or $\bs \Phi_{\cl S}$) denotes the vector
(resp. the matrix) obtained by retaining the components
(resp. columns) of $\bs u$ (resp. $\bs \Phi$) belonging to $\cl
S$. Alternatively, we have $\bs u_{\cl S} = \bs R_{\cl S} \bs u$ or
$\bs \Phi_{\cl S} = \bs \Phi \bs R^T_{\cl S} $ where $\bs R_{\cl S} :=
(\Id_{\cl S})^T \in \{0,1\}^{\#\cl S \times D}$ is the
\emph{restriction} operator. The kernel (or \emph{null space}) of a matrix $\bs \Phi$ is
$\rm ker (\bs \Phi) := \{ \bs x \in \Rbb^D : \bs \Phi \bs x
= 0\}$. The norm of an operator $\bs K$ is defined as
$\bbb \bs K \bbb = \max \{\| \bs K \bs x \| : \bs x \in \Rbb^N \
\textrm{with} \ \| \bs x \| = 1 \}$.  We denote $\Gamma^0(\cl V)$ the
class of proper, convex and lower-semicontinuous functions from a
finite dimensional vector space $\cl V$ (\eg $\Rbb^D$) to
$(-\infty,+\infty]$~\cite{combettes2011proximal,dupeICIP2011}. The non-negativity thresholding function is $(\bs x)_+$, which
is defined componentwise as $(x_i)_+ = (x_i + |x_i|)/2$. The (convex) indicator
function $\imath_{\cl C}(\bs x)$ of the set $\cl C$
is equal to $0$ if $\bs x\in \cl C$ and $+\infty$ otherwise. The interior of a set $\Omega$ is $\intset \Omega$, which
consists of all points of $\Omega$ that do not belong to its boundary
$\partial \Omega$. 
The superscript $^*$ denotes the conjugate transpose (or adjoint) of a
matrix (or the complex conjugate for
scalars), while $^\star$ denotes the convex conjugation of a convex function.

\section{OPTICAL DEFLECTOMETRIC TOMOGRAPHY} 
\label{sec:problem}

In this section, the principles of optical deflectometric tomography
are explained and completed with a brief description of the
experimental setup used for actual deflectometric acquisition.

\subsection{Principles} 
\label{sec:principles}

Optical Deflectometric Tomography aims at infering the refractive index spatial
distribution (or refractive index map -- RIM) of a transparent
object. This is achieved by measuring, under various incident angles,
the deflection angles of multiple parallel light rays when passing
through this transparent object (see Fig.~\ref{fig:deflectometry}-(top)).
The (indirect) observation of the RIM, allowing its further
reconstruction, is guaranteed by the relation between the total
light ray deflection and the integration of the RIM transverse gradient
along the light ray path~\cite{beghuin2010optical}.

In this work, we restrict ourselves to two-dimensional (2-D) ODT by assuming
that the refractive index $\imap$ of the observed object is constant
along the $\bs e_3$-axis for a convenient coordinate system $\{\bs
e_1,\bs e_2,\bs e_3\}$, \ie $\partial\imap(\bs r)/\partial r_3 = 0$ (with $\bs
r=(r_1,r_2,r_3)^T\in\Rbb^3$) and deflections occur in the $\bs e_1$-$\bs
e_2$ plane. This assumption is validated in the experimental
  setup described later in this section, where the light probes a very thin 2-D slice
  of the 3-D sample and where the particular geometry of the test objects
  makes the refractive index variation along the $\mathbf{e}_3$ direction negligible, \ie $|\tfrac{\partial}{\partial r_3} \imap| \ll
\max(|\tfrac{\partial}{\partial r_1}\imap|, |\tfrac{\partial}{\partial
  r_2} \imap|)$.

Given the refractive index $\imap:\vec{r}=(r_1,r_2)^T\in\Rbb^2\to \imap(\vec r)$, for a particular light ray trajectory $\cl R = \{\vec{r}(s):s\in \Rbb\} \subset \Rbb^2$ parametrized by $\vec r(s)=(r_1(s),r_2(s))^T\in\Rbb^2$ with $s$ describing its \emph{curvilinear} parameter\footnote{By a slight abuse of notation, $\bs r$ denotes any points in $\Rbb^2$ while $\bs r(s)$ represents a particular curve in $\Rbb^2$ parametrized by $s\in\Rbb$.}, the relation between light deflections and the refractive index $\imap$ is provided by the \emph{light ray equation}
\begin{equation}
	\label{eq:eikonal}
	\tfrac{\ud}{\ud s}\big( \imap\tfrac{\ud}{\ud s} \bs r(s)\big) = \bs \nabla \imap,
\end{equation}
established from the eikonal equation~\cite[Section~3.2, p.~129]{born1999principles}.

For small deflection angles, we can adopt a (first
  order) approximation and assume the trajectory $\cl R$ is a straight
  line. The error committed by removing the trajectory
dependence in the deflection angle can be estimated by a \emph{ray tracing} method (not
developed here) relying on the Snell law. In our tests, for simple continuous object models, the absolute error between the two
deflection models is lower than $0.7^\circ$ for deflection angles
smaller than $7^\circ$ (the angular acceptance of the optical deflectometer). Moreover, in \cite{optimess2012}, the authors have qualitatively shown
that, for the same range of deflection angles and for limited
refracted index variations, the model mismatch stays limited. 

\begin{figure}
  \centering
  \includegraphics[height=5cm]{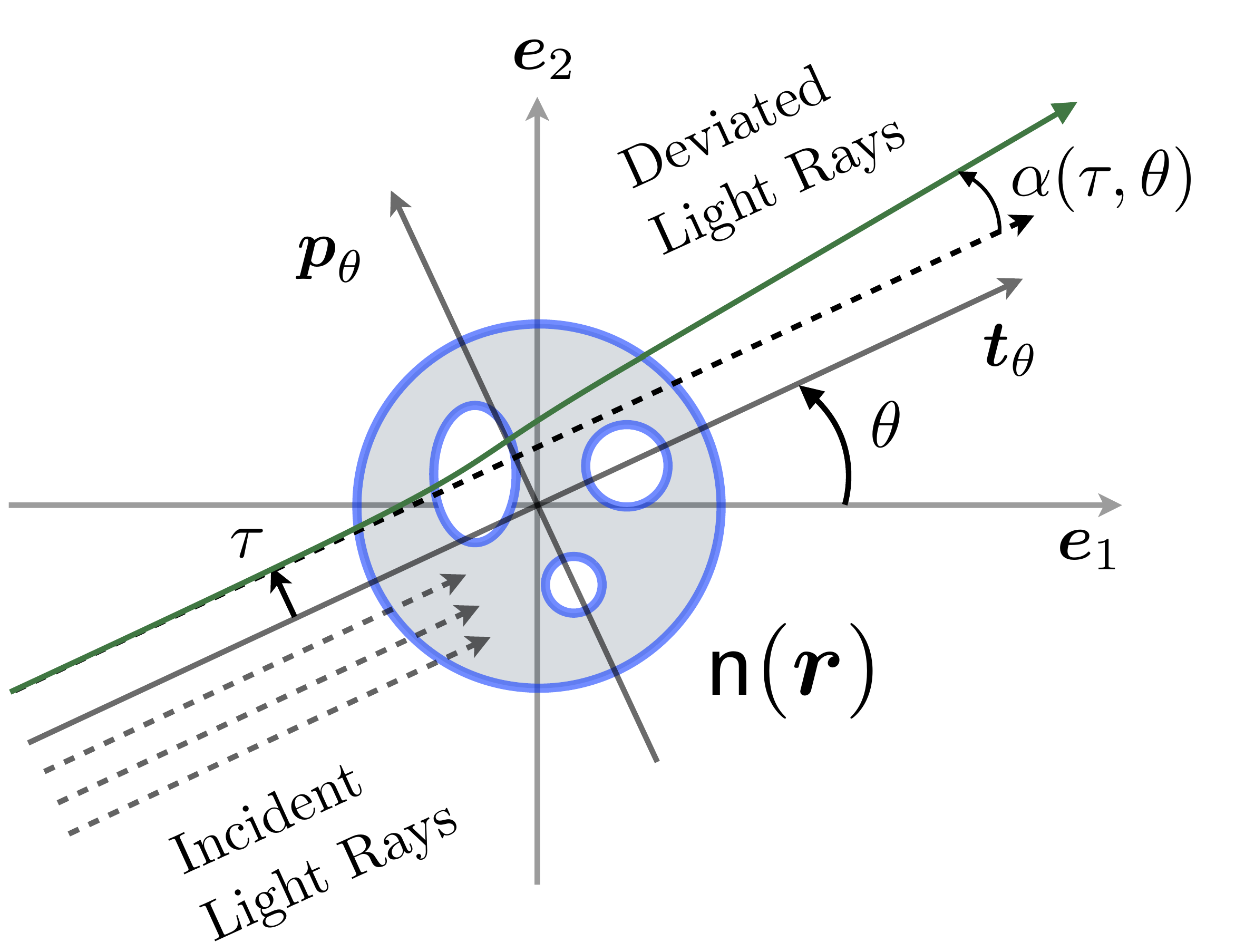} \\[2mm]
  \includegraphics[height=3.8cm]{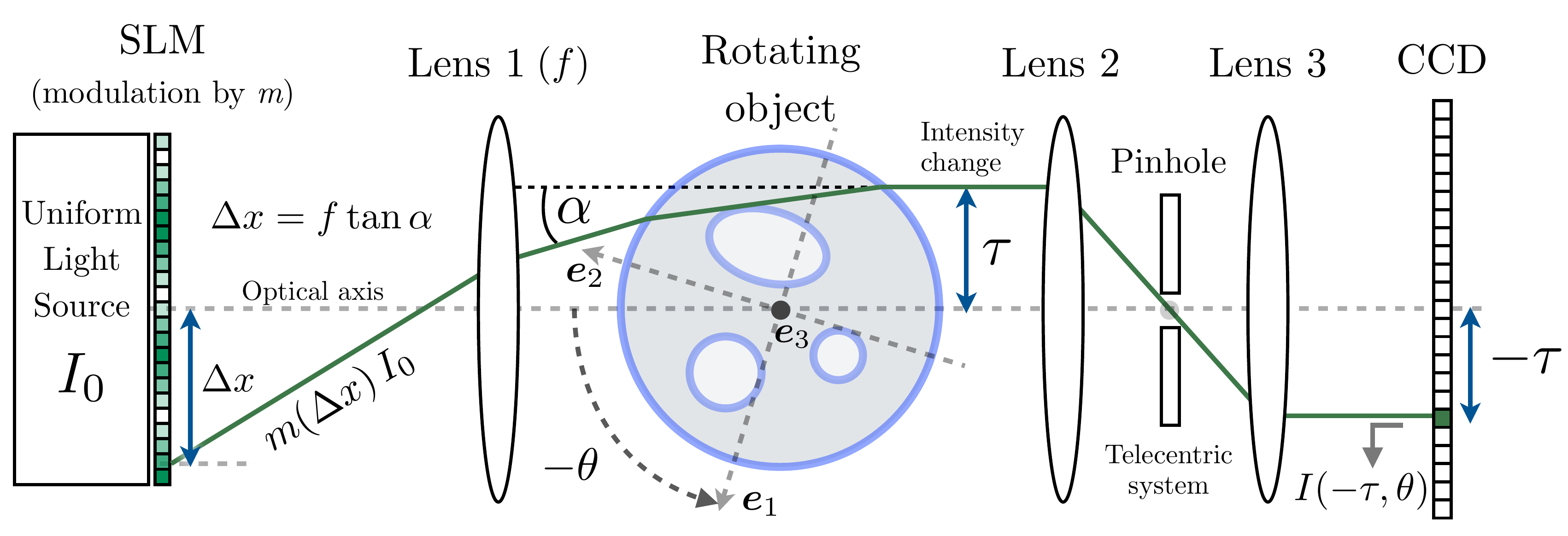}
  \caption{(top) The Deflectometric model. (bottom) A scheme of a
    Phase-Shifting Schlieren Deflectometer (in refraction
    mode) measuring light ray deflection angles by encoding light
    deviation into light intensity variations. }
  \label{fig:deflectometry}
\end{figure}

In this simplified model, the 2-D RIM is measured by $\Delta(\theta,\tau;\imap)$,
the sinus of the deflection angle $\alpha(\theta,\tau;\imap)$ of a light ray $\cl R(\theta,\tau) = \{\vec{r} \in \Rbb^2: \vec{r}\cdot\vec p_{\theta} = \tau\}$, where $\tau\in\bb R$ is the \emph{affine} parameter determining the distance between $\cl R$ and the origin, $\theta\in [0, 2\pi)$ is the incident angle of $\cl R$ with the $\bs e_1$ axis, and $\vec p_\theta = (-\sin \theta, \cos\theta)^T$ is the (constant) perpendicular vector to the light ray direction $\vec t_\theta =(\cos \theta, \sin\theta)^T$.

The simplified ODT model depicted in
Fig.~\ref{fig:deflectometry}-(top) is then obtained from the light ray equation \eqref{eq:eikonal} as
\begin{equation}
\label{eq:DeflectionAngle}
\Delta(\tau,\theta) := \tinv{\imap_{\rm r}}\,\int_{\Rbb} \vec\nabla \imap(\bs r_{\tau,\theta}(s)) 
\cdot \vec p_{\theta}\, \ud s =\ \tinv{\imap_{\rm r}}\int_{\Rbb^2} \big(\vec\nabla \imap(\vec{r})
\cdot \vec p_{\theta}\big)
\ \delta(\tau - \vec{r}\cdot\vec p_{\theta})\,\ud^2\vec r,
\end{equation}
where $\imap_{\rm r}$ is the (constant) reference refractive index of the
surrounding medium, $\bs r_{\tau,\theta}(s) = s \bs t_{\theta} + \tau \bs p_\theta
\in \cl R$ and the Dirac distribution turns the line integral into an
integration over $\Rbb^2$. In short, the above equation relates the
deflection angle $\Delta(\tau,\theta)$ to the Radon transform
of the transverse gradient of $\imap$ within the straight line
trajectory approximation. Interestingly, this first order
deflectometric model is also used in computer graphics for rendering of refractive gas
flows \cite{ihrke_deflecto_09}.

As for traditional parallel absorption tomography~\cite[Section~3.2, p.~56]{Kak1988_Imaging}, there
exists a Deflectometric Fourier Slice Theorem (DFST) that relates the
1-D (radial) Fourier transform of the deflection angle along the
\emph{affine parameter} $\tau$, \ie
\begin{equation}
\label{eq:radial-fourier-deflecto}
	y(\omega,\theta)\ :=\ \int_{\Rbb} \Delta(\tau,\theta)\, e^{-2\pi i\, \tau\omega} \ud \tau,
\end{equation}
to the 2-D Fourier transform of the RIM. Mathematically, the DFST
establishes the following equivalence \cite{SPIE_2011} (proved in Appendix~\ref{app:DFST-proof}):
\begin{equation}
\label{eq:measures}
y(\omega,\theta)\ =\ \tfrac{2\pi i \omega}{\imap_{\rm r}}\,\widehat{\imap}\big(\omega\,\vec p_{\theta}\big),
\end{equation}
where $\widehat \imap(\vec k) = \int_{\Rbb^2} \imap(\vec r)\,e^{-2\pi i
  \vec k \cdot \vec r}\,\ud^2\vec r$ stands for the 2-D Fourier
transform of $\imap$. Hereafter,
the value $\omega$ is called \emph{affine frequency}.

\medskip

We may remark from (\ref{eq:measures}) that there are different ways
to cover the 2-D frequency plane with deflectometric
measurements. This can be observed from the following symmetry
relations for $\imap \in \Rbb$, $\tau,\omega\in\Rbb$, $\theta\in
[0,2\pi)$:
\begin{subequations}
\begin{align}
\label{eq:OD-eqsym-1}
\Delta(\tau,(\theta+\pi)\!\!\mod 2\pi)&=-\Delta(-\tau,\theta),\\
\label{eq:OD-eqsym-2}
y(\omega,(\theta+\pi)\!\!\mod 2\pi)&=-y(-\omega,\theta),\\
\label{eq:OD-eqsym-3}
y(\omega,\theta)&=y^*(-\omega,\theta),
\end{align}
\end{subequations}
where the two first relations come from the change $\bs p_{\theta +
  \pi} = -\bs p_\theta$ in (\ref{eq:DeflectionAngle}) and
(\ref{eq:measures}), and the last equation is due to the Fourier
conjugate symmetry of real spatial functions.

In particular, from the symmetry (\ref{eq:OD-eqsym-3}), we can
restrict $\omega$ to positive values by taking $\theta$ in the whole
circle $[0, 2\pi)$. Or alternatively, from (\ref{eq:OD-eqsym-2}) and
(\ref{eq:OD-eqsym-3}), we can deduce
$y(\omega,\theta)=-y^*(\omega,(\theta+\pi)\!\!\mod 2\pi)$ and restrict
$\theta$ to the half circle $[0,\pi)$ with $\omega\in\Rbb$.  We insist 
on the fact that this symmetry is only preserved when the first order 
approximation is validated. \medskip

When comparing the relation (\ref{eq:measures}) with the standard
tomographic Fourier Slice theorem (FST)~\cite[Section~3.2, p.~56]{Kak1988_Imaging}, the main difference is provided by
the transverse gradient in the deflectometric relation
\eqref{eq:DeflectionAngle}, which results in multiplying by $2\pi i \omega/\imap_{\rm r}$ the RIM Fourier transform.  In
particular, from (\ref{eq:DeflectionAngle}) or from
(\ref{eq:measures}) (since $\omega$ vanishes on the frequency origin)
we see that the ODT sensing is blind to constant RIM.  As
we will see in Sec.~\ref{sec:reconstruction}, this implies that the RIM can only be
estimated relatively to the known reference RIM~$\imap_{\rm r}$.

Let us conclude by insisting on the impact of the
equivalence \eqref{eq:measures}.  Similarly to the use of the Fourier
Slice theorem in common (absorption) tomography, \eqref{eq:measures}
is of great importance for defining a discrete ODT sensing model which
can be computed efficiently in the Fourier domain given a discretized
refractive index map $\imap$.

\subsection{Deflection Measurements} 
\label{sec:defl-meas}

Experimentally, the deflection angles can be measured by
phase-shifting Schlieren deflection tomography (schematically
represented in Fig.~\ref{fig:deflectometry}-(bottom)). We briefly
explain this system for the sake of completeness in order to set the
experimental background surrounding the actual deflection measurement
process. More details can be found in~\cite{joannes2003phase,
  beghuin2010optical, Emmanuel_2010, SPIE_2011}.

This system proceeds by encoding light ray deflection $\alpha$ in
intensity variations. A transparent object is illuminated with an
incoherent uniform light source $I_0$ modulated by a sinusoidal
pattern $m$ using a Spatial Light Modulator (SLM). From classical
optics, the light deviation angle $\alpha$ is related to a phase shift
$\Delta x = f \tan \alpha$, where $f$ is the focal length of Lens
1. This phase shift is associated with the intensity variation thanks
to the modulation $m$ as $I(-\tau,\theta) = m(\Delta x) I_0$. These
intensity variations are processed by phase shifting methods for
recovering the deflection measurements $\alpha(\tau,\theta)$
for each couple of parameters $(\tau,\theta)\in \Rbb\times
[0,2\pi]$. Up to some linear coordinate rescaling, the \emph{affine}
parameter $\tau$ corresponds to the horizontal pixel coordinate in the
2-D CCD detector collecting light (assuming the object refractive
index constant along the CCD vertical direction). This correspondence
is implicitly allowed by the telecentric system formed by the
combination of Lens 2 and 3. The pinhole guarantees that only parallel
light rays outgoing from the object are collected. Rather than
rotating the whole incident light beam around the object, it is this
one which is rotated by an angle $-\theta$ along an axis parallel to
the CCD pixel vertical direction \cite{beghuin2010optical}. Finally,
since the system is invariant under time inversion, \ie under light
progression inversion, measuring the deflection angle $\alpha$ in
Fig.~\ref{fig:deflectometry}-(bottom) is equivalent to measuring the
same angle in Fig.~\ref{fig:deflectometry}-(top).

\section{DISCRETE FORWARD MODEL} 
\label{sec:processing}

In order to reconstruct efficiently the RIM from ODT measurements,
recorded data must be treated appropriately considering, jointly, the
data discretization, the polar geometry of the ODT sensing and
unavoidable measurement noises. We present here the discrete
formulation of the ODT sensing and the construction of the forward
model from the recorded data.

\subsection{Discrete domains}
\label{sec:discrete-domains}

Let us first assume that the object of interest is fully contained in
a square field-of-view (or FoV) $\Omega\subset \Rbb^2$ centered on the
spatial origin. The physical dimensions of this FoV can be provided by
the Deflectometric device itself. In other words, the RIM is constant
and equal to the reference index $\imap_{\rm r}$ outside of
$\Omega$. This involves also that the deflection measurement vanishes,
\ie $\Delta(\tau,\omega)=0$, if $|\tau|$ is bigger than the typical
width of $\Omega$ in a section of direction~$\theta+\pi/2$.

We can consider a spatial sampling of $\Omega$ as follows. We define a
$N_0\times N_0$ 2-D Cartesian grid of $N:=N_0^2$ pixels as
$$ 
\cl C_N = \{ \bs r_{m,n} := (m\,\delta r, n\,\delta r): -N_0/2 \leq m,n <
N_0/2\}, 
$$
where the spatial spacing $\delta r$ is adjusted to $\Omega$ and to
the resolution by imposing $\Omega = [-\tinv{2}N_0\,\delta r,
\tinv{2}N_0\,\delta r]\times [-\tinv{2}N_0\,\delta r,
\tinv{2}N_0\,\delta r]$. 

\medskip

Second, as the deflectometric experiments provide evenly sampled
variables $\tau$ and $\theta$, $\Delta$ is measured on a (signed)
regular polar coordinate grid
$$ 
\cl P_M := \{(\tau_s,\theta_t): - (N_\tau/2) \leq s < (N_\tau/2),\, 0
\leq t < N_{\theta}\},\ \tau_s := s \,\delta\tau,\ \theta_t := t
\,\delta\theta, 
$$
of size $M := N_\tau N_\theta$, with $N_\tau$ the number of parallel
light rays passing through the object (assumed even), $N_{\theta}$ the
number of incident angles in ODT sensing, $\delta\tau$ and $\delta
\theta = \pi/N_\theta$ the distance between two consecutive affine
parameters and angles respectively\footnote{Notice that $\delta\theta$
  is not set to $2\pi/N_\theta$ since $\tau$ is allowed to be
  negative.}.

The value $\delta \tau$ can be known experimentally from the pixel
size of the CCD detector in a Schlieren Deflectometer
(Sec.~\ref{sec:defl-meas}). Moreover, the value $\delta\tau$ and the
resolution $N_\tau$ are also related to the FoV $\Omega$ so that
$\delta \tau N_\tau \approx \delta r N_0$. Since there is no reason to
ask more resolution to the sampling $\cl C_N$ than the available in
the affine variations of the ODT measurements, we will work with
$\delta \tau \approx \delta r$ and $N_\tau \approx N_0$.

\medskip

Third, in this discretized setting, the affine frequency $\omega$ in
(\ref{eq:measures}) must also be sampled with $N_{\tau}$ values. As
described in the next section, this comes from the replacement of the
(radial) Fourier transform in (\ref{eq:radial-fourier-deflecto}) by
its discrete counterpart. This leads to a (signed) frequency polar grid of same
size
$$ 
\widehat{\cl P}_M := \{(\omega_{s'},\theta_t): -(N_\tau/2) \leq s' < (N_\tau/2),\, 0 \leq t
< N_{\theta}\},\ \omega_{s'} := s'\,\delta \omega,\ \theta_t := t
\,\delta\theta, 
$$
with $\delta \omega = 1/(N_\tau \delta\tau)$. As it will become
clearer in the following, only half of this polar grid will be
necessary to bring independent observations of the RIM, \ie we will
often work on
$$ 
\widehat{\cl P}^+_M := \{(\omega_{s'},\theta_t): 0 \leq s' < (N_\tau/2),\,
0 \leq t < N_{\theta}\},
$$
with $\#\widehat{\cl P}^+_M = M/2$.

\subsection{Discretized Functions}
\label{sec:discr-funct}

From the discrete domains defined above, the continuous RIM $\imap$
observed in the experimental FoV is discretized into a set of $N$
values $\imap(\bs r_{m,n})$ from the coordinates $\bs r_{m,n} \in \cl
C_N$. This description can always be arranged into a one dimensional
$N$-length vector $\imapv \in \Rbb^N$, given a convenient mapping
between the components indices of $\imapv=(\imap_1,\cdots,\imap_N)^T$
and the pixel coordinates in~$\cl C_N$. This ``vectorization''
provides us a simplified representation of any linear operator acting
on the sampled RIM $\imap$ as a matrix multiplying $\imapv$.

For the different functions discretized on $\cl P_M$ or on $\widehat{\cl
  P}_M$, we use the same vectorization trick, namely, a function $u$
defined on $\cl C_N$ and sampled on $\cl P_M$ is associated to a
vector $\bs u\in\Rbb^{M}$ with the right correspondence between the
components of $\bs u$ and the polar coordinates in $\cl P_M$ (and
similarly for function defined in the Fourier domain and sampled on
$\widehat{\cl P}_M$).

Therefore, the ODT observations $\{\Delta(\tau,\theta):
(\tau,\theta)\in \cl P_M\}$ are gathered in a vector $\bs z\in
\Rbb^{M}$ with $z_{j}=\Delta(\tau_s,\theta_t)$ for a certain index
mapping $j=j(s,t) \in [M]$. In this discrete representation, the
equivalent transformation of (\ref{eq:radial-fourier-deflecto}) reads
\begin{equation}
  \label{eq:F-rad-def}
\bs y_{\rm comp} = (\sqrt{N_\tau}\,\delta\tau)\, \bs F^{\rm rad} \bs z,
\end{equation}
where $\bs y_{\rm comp}\in \Cbb^{M}$ is associated to a (vectorized) sampling of
$y$ on $\widehat{\cl P}_M$, and $\bs F^{\rm rad}: \Rbb^{M} \to \Cbb^{M}$
performs a 1-D DFT on the radial $\tau$-variations of $\bs z$, \ie
$$
(\bs F^{\rm rad} \bs z)_{k(s',t)} = \tinv{\sqrt{N_\tau}}\,\sum_{s} z_{j(s,t)}\,e^{-2\pi i s s'/N_\tau},
$$
for a vectorized index $k=k(s',t)$. In other words, if $\delta\tau$ is
sufficiently small, \eg if $\Delta(\tau,\theta)$ is band-limited with
a cut-off frequency smaller than $1/\delta\tau=N_\tau \delta
\omega$, then
\begin{align*}
y_{k(s',t)}&= \tfrac{\sqrt{N_\tau}\,\delta\tau}{\sqrt{N_\tau}}\,\sum_{s} \Delta(\tau_s,\theta_t)\,e^{-2\pi i (s\delta\tau)
(s'/\delta\tau N_\tau)}\\ 
&= \sum_{s}
\Delta(\tau_s,\theta_t)\,e^{-2\pi i \tau_s
\omega_{s'}}\,\delta\tau\ \approx\ y(\omega_{s'},\theta_t),
\end{align*}
using a Riemann sum approximation of
(\ref{eq:radial-fourier-deflecto}) and knowing that $\Delta$ vanishes
on~$\Omega^c$.

Despite the fact that $\bs y_{\rm comp}$ belongs $\Cbb^M \simeq
\Rbb^{2M}$, this vector brings only $M$ independent real observations
of $\bs\imap\in\Rbb^N$. This is due to the central symmetry
(\ref{eq:OD-eqsym-3}) induced by the realness of $\bs z$ and which
allows us to consider $\bs y_{\rm comp}$ only on $\widehat{\cl P}^+_M$.

This is clarified by the definition of the useful operator $\bs
\Theta: \Cbb^M \to \Rbb^M$ which perform the two following linear
operations. First, it restricts any vector $\bs \xi\in\Cbb^M$ to the
indices associated to the half grid $\widehat{\cl P}^+_M$. Second, it
appends the $M/2$ imaginary values of the restricted vector to its
$M/2$ real values in order to form a $M$-length vector in
$\Rbb^M$. The adjoint operation $\bs \Theta^*$, which is also the
inverse of $\bs \Theta$ for vectors in $\Cbb^M$ respecting the
Hermitian symmetry, is obtained easily by first reforming a
$M/2$-length complex vector from the separated real and imaginary
parts, and by inserting the results in $\Cbb^M$ according to the
indices of $\widehat{\cl P}^+_M$ and by completing the information in
$\widehat{\cl P}_M\setminus \widehat{\cl P}^+_M$ with the central symmetry
(\ref{eq:OD-eqsym-3}).

Consequently, thanks to $\bs \Theta$, we can form the real vector
\begin{equation}
  \label{eq:F-rad-def-real}
\textstyle \bs y = \bs \Theta\,\bs y_{\rm comp} = (\sqrt{N_\tau}\,\delta\tau)\, (\bs \Theta\,\bs F^{\rm rad})\, \bs
z\quad \in \Rbb^M, 
\end{equation}
with $(\bs \Theta\,\bs F^{\rm rad}):\Rbb^M\to\Rbb^M$. We call $\bs y$
the Frequency Deflectometric Measurements (FDM) vector. This will
be our direct source of ODT observations instead of
$\bs z$.

\subsection{Forward Model}
\label{sec:forward-model}

We can now explain how we use the DFST relation (\ref{eq:measures})
for defining the forward model that links any discrete 2-D RIM
representation to its FDM vector.

In a previous work~\cite{SPIE_2011}, the data available in
the frequency polar grid $\widehat{\cl P}_M$ were first interpolated to a
Cartesian frequency grid in order to reconstruct the 2-D RIM using a
Fast Fourier Transform (FFT). However, the
polar to Cartesian frequency interpolation introduced a hardly
controlled colored distortion.

We use here another operator $\bs F: \Rbb^{N} \to
\Cbb^{M}$ performing a non-equispaced Discrete Fourier Transform
(NDFT) for directly relating functions sampled on the Cartesian
spatial grid $\cl C_N$ to those sampled on the polar frequency grid
$\widehat{\cl P}_M$.  

More precisely, given a function $f:\Omega\to \Cbb$ sampled on $\cl
C_N$, the NDFT\footnote{This NDFT formulation is strictly equivalent to the one
given in \cite{KeinerNFFT} where $\delta\tau=\delta r=1$.} computes
\begin{equation}
  \label{eq:2DNFFT}
   \widehat f(\bs k) = \sum_{m=0}^{N_0-1}\sum_{n=0}^{N_0-1} f(\bs r_{m,n}) e^{-2\pi i\, \bs k \cdot \bs r_{m,n}}, 
\end{equation}
on the $M$ nodes $\bs k$ of $\widehat{\cl P}_M$. Gathering all the values of
$\widehat f$ and $f$ into vectors $\bs{\widehat{f}}\in\Cbb^{M}$ and $\bs f \in
\Rbb^{M}$ respectively, this relation is conveniently summarized as
\begin{equation}
  \label{eq:2DNFFT-matrix}
  \bs{\widehat f} = \bs F \bs f,\quad \text{with}\ F_{ij}=e^{-2\pi i\, \bs k_i \cdot \bs r_j}, 
\end{equation}
where the matrix $\bs F \in \Cbb^{M\times N}$ stands for the linear
NDFT operation. Its internal entry indexing follows the one of the
components of $\bs f$ and $\bs{\widehat f}$. We explain in
Appendix~\ref{app:NFFT} how the Non-equispaced Fast Fourier Transform
(NFFT) algorithm allows us to compute efficiently in $\cl O(N \log
(N/\epsilon))$ the multiplications $\bs F \bs u$ and $\bs F^* \bs v$
for any $\bs u \in \Rbb^N$ and $\bs v \in \Cbb^{M}$, with a controlled
distortion $\epsilon$ with respect to the true NDFT.

The action of $\bs F$ on a discretized RIM $\imapv$ is related to the
continuous Fourier transform of $\imap$ as follows. Let $j=j(s',t)\in
[M]$ be the $j^{\rm th}$ point $\bs k_j$ of the grid
$\widehat{\cl P}_M$ associated to the polar coordinates
$(\omega_{s'},\theta_t)$. Then, for a sufficiently small $\delta r$,
\begin{align}
\label{eq:NDFT-continuous-equiv}
	(\bs F \imapv)_{j(s',t)}&= \sum_{m=0}^{N_0-1}\sum_{n=0}^{N_0-1} \imap(\bs r_{m,n})
  e^{-2\pi i\, \bs k_j \cdot \bs r_{m,n}}\ \approx \tinv{(\delta r)^2}\,\widehat{\imap}(\omega_{s'} \bs p_{\theta_t}). 
\end{align}

To take into account the multiplication by $2\pi i\,\omega/\imap_{\rm
  r}$ in (\ref{eq:measures}) and the existence of a factor $1/(\delta
r)^2$ in the equivalence \eqref{eq:NDFT-continuous-equiv}, we
introduce the diagonal operator $\bs D\in \nobreak \Rbb^{M\times M}$
defined as
$$
\bs D = \tfrac{2\pi\,i (\delta r)^2}{\imap_{\rm r}}\,\diag(\omega_{(1)},\,\cdots,
\omega_{(M)}),
$$ 
where $\omega_{(j)}$ refers to the $\omega$-coordinate of the $j^{\rm
  th}$ point of $\widehat{\cl P}_M$, \ie if as before $j=j(s',t)$ is
associated to $(\omega_{s'},\theta_t)\in \widehat{\cl P}_M$ then $\omega_{(j)} =
\omega_{s'}$. The operator $\bs D$ models the effect of the transverse gradient
in the Fourier domain.

In parallel to the discussion ending Sec.~\ref{sec:discr-funct}, we
also restrict the action of $\bs D \bs F$ to the domain $\widehat{\cl
  P}^+_M$. Consequently, using the operator $\bs \Theta$
(Sec.~\ref{sec:discr-funct}), the final linear forward model linking
the real FDM to the 2-D NDFT of the discrete RIM $\bs \imap$ reads
\begin{equation}
\label{eq:disc-measures-noise}
\bs y \ =\ (\bs \Theta \bs D \bs F)\,\imapv + \bs \eta\quad \in \Rbb^M.
\end{equation}

The additional noise $\bs \eta\in \Rbb^{M}$ integrates the different
distortions explained and estimated in Sec.~\ref{sec:noise_estimation}, \ie the numerical computations, the model discretization, the discrepancy to the first order approximation
\eqref{eq:DeflectionAngle}, and the actual noise corrupting the
observation of $\bs z$.

\medskip

Notice that in the absence of noise ($\bs \eta = 0$), the model
\eqref{eq:disc-measures-noise} could be easily turned into a classical
tomographic model where the DC frequency is not observed. Indeed,
forgetting the formal action of $\bs \Theta$, if the frequency
origin has been carefully removed from $\widehat{\cl P}_M$, then $\bs D$
is invertible with $\bs D^{-1} = -\tfrac{i\,\imap_{\rm r}}{2\pi
  (\delta r)^2}\,\diag(\omega^{-1}_{(1)},\, \cdots,
\omega^{-1}_{(M)})$, and we can solve the common tomographic problem
\begin{equation}
\label{eq:tomo_transform}
	\tilde{\bs y} := \bs D^{-1} \bs y \ =\ \bs F\,\imapv.
\end{equation}

However, as we present in Sec.~\ref{sec:tomo_vs_deflect}, this
transformation is not suited to noisy ODT sensing. Even for a simple
additive white Gaussian noise $\bs \eta$ (or AWGN), the multiplication
by $\bs D^{-1}$ breaks the homoscedasticity of $\bs \eta$, \ie the
variance of each $(\bs D^{-1} \bs \eta)_j$ varies with $j \in
[M]$. This interferes with common reconstruction techniques used in
classical tomography. Obviously, a noise whitening could be realized
for stabilizing such methods but at least for AWGN, this strictly
amounts to solve directly the model \eqref{eq:disc-measures-noise}.

\medskip

\section{RELATED WORKS}
\label{sec:state-of-the-art}

This section describes some recently proposed methods for tomographic
reconstruction in the domains of differential phase-contrast
tomography and common absorption tomography.

In differential phase-contrast tomography, the refractive index
distribution is recovered from
phase-shifts measurements. These are composed by the
derivative of the refractive index map, inducing the apparition of the
\emph{affine} frequency $\omega$ when using the FST, as it happens in
the ODT sensing model (Sec.~\ref{sec:problem}). In this
application, Pfeiffer et al.~\cite{Pfeiffer2007} have used the FBP
algorithm to reconstruct the refractive index map from a fully covered
set of projections. Cong et al.~\cite{cong2011,cong2012} have used
different iterative schemes based on the minimization of the TV norm
to reconstruct the refractive index distribution over a region of
interest. These methods are accurate and provide similar results, but
the iterative scheme based on the TV norm has proved to be better than
FBP when the amount of acquisitions decreases.

In common Absorption Tomography (AT) we deal with the reconstruction
of the absorption index distribution from intensity measurements. As
these measurements are directly related to the absorption index, the
AT sensing model does not include the \emph{affine} frequency
$\omega$. In this domain, several works have exploited sparsity based
methods. Most recent works in AT have focused on promoting a small TV
norm~\cite{Ritschl2011, Sidky2011}. Sidky et al.~\cite{Sidky2011} use
a Lagrangian formulation for the tomographic reconstruction problem,
promoting a small TV norm under a Kullback-Leiber data divergence and
a positivity constraint. They aim at reconstructing a breast phantom
from 60 projections with Poisson distributed noise. For this, they use
the primal-dual optimization algorithm proposed by Chambolle et
al.~\cite{ChambollePock_2011}. The method results in high quality
reconstruction compared to FBP but with a convergence result that is
highly dependent on the Lagrangian parameter choosen.

Ritschl et al.~\cite{Ritschl2011} use a constrained optimization
formulation to reconstruct the absorption index from low amount of
clinical data in the presence of metal implants and Gaussian
noise. This problem is solved by means of an alternating method that
allows then optimizing separately the raw data consistency function
and the sparsity cost function, without the need of prior information
on the observations. The fast convergence of the method is based on
the estimation of the optimization steps. The gradient descent method
is used to minimize the TV norm and the consistency term is minimized
via an algebraic reconstruction technique. The method is proven to
give better results than FBP.

These works have proved that some tomographic applications can be made
compressive in the CS sense~\cite{Candes2004, donoho2006cs}. In short,
accurate tomographic image
reconstruction are reachable from a number of samples that is smaller than the
desired image resolution, if the image is sparse in some basis, \ie
the image expansion in that basis contains only a small number of
nonzero coefficients. However, most of
these works have considered Cartesian frequency grids while actual sensing often occurs in polar or non-equispaced grids. Besides, they attack different
problems than ODT, where the sensing model and type of measurement
change from one to the other. One important difference between AT and
ODT sensing models, relies on the presence of the \emph{affine}
frequency as materialized by a diagonal operator $\bs D$; whose impact
is analyzed in detail in Sec.~\ref{sec:noise_tomo_vs_deflect} and
Sec.~\ref{sec:tomo_vs_deflect} for perfect and noisy sensing.

\section{REFRACTIVE INDEX MAP RECONSTRUCTION} 
\label{sec:reconstruction}

Three methods can be considered for recovering the discrete RIM
$\imapv\in\Rbb^N$: the common Filtered Back
  Projection (FBP); a penalized least square 
solution, which is related to a \emph{minimal energy solution} (ME) \cite{Candes2006}; and a regularized
approach called TV-$\ell_2$ minimization. Since FBP and ME are widely used in tomographic
problems, we will use them as a standard to compare with the quality of
TV-$\ell_2$. 

\subsection*{Filtered Back Projection}
\label{sec:ls}

The filtered back projection (FBP) can be briefly defined as an
analytical method that consists of first filtering the
tomographic projections with an appropriate function, \ie a ``ramp filter'' for
absorption tomography (AT) \cite[Section~3.3,
p.~60]{Kak1988_Imaging} or a simple Hilbert transform for
deflection tomography \cite{beghuin2010optical, Pfeiffer2007}. The result is then \emph{back projected} in the
spatial domain by angular integration. Despite its simplicity, this
technique suffers of severe artefacts when the number of angular
observations decreases. Moreover, FBP does not integrate the
noise distortion in its processing.

\subsection*{Minimum Energy Reconstruction}
\label{sec:ls}

Given a linear sensing model 
\begin{equation}
\label{eq:gen-inv-prob}
\bs y = \bs \Phi \bs x \in \Rbb^M
\end{equation}
of some image (vector) $\bs x\in \Rbb^N$ by a sensing operator $\bs
\Phi \in \Rbb^{M \times N}$ with $M \leq N$, a common procedure is
to estimate $\bs x$ by applying the (right) pseudo-inverse of $\bs
\Phi$ to
the observations, \ie by computing $\bs \Phi^\dagger \bs y = \bs \Phi^* (\bs \Phi \bs
\Phi^*)^{-1} \bs y$ (for non-singular $\bs \Phi \bs
\Phi^*$).

This solution is actually equivalent to compute a \emph{minimum energy}
solution (or \emph{penalized least
  square}) $\bs x_{\rm ME}$ by solving the convex problem 
\begin{equation}
  \label{eq:ls-method}
  \bs x_{\rm ME} = \argmin_{ \bs u \in \bb
    R^N} \|\bs u\|_{2} \ \st\ \bs y = \bs \Phi \bs u. 
\end{equation}
In words, the general inverse problem \eqref{eq:gen-inv-prob} is
solved by a \emph{regularization} promoting a small $\ell_2$-norm (or \emph{energy}). 
For large scale problems, solving this last convex formulation is preferred 
to the estimation of the pseudo-inverse that requires the costly inversion of
a $M\times M$ matrix.

Common reconstruction approaches (\eg in Medical
Imaging) follow this minimum energy principle \cite{Candes2006}. They
generally proceed by zeroing the
Fourier coefficients of all unobserved frequencies, \ie
a process that minimizes the energy of the solution
for frequency sampling defined on regular grids. The ME solution is
also close to a discretization of the FBP procedure
for a densely covered frequency space. 

As shown later, the ME method presents some strong limitations when the
model (\ref{eq:gen-inv-prob}) is severely ill-conditioned or when
noise corrupts the observation of $\bs y$. 
For instance, in the case where the
  tomographic problem subsamples a regular
  sampling of the Fourier domain, artifacts and bluring appear from the
  convolution of the pure image with the filter associated to the
  subsampling pattern (or \emph{dirty map} \cite{wiaux2009compressed}).

In ODT, the sensing operator and the sensing model read $\bs
\Phi = \bs \Theta \bs D \bs F$ and $\bs y = \bs \Theta\bs D \bs F
\imapv$, respectively, and we will denote by $\imapv_{\rm ME}$ the corresponding ME
solution. 

\subsection*{TV-$\ell_2$ minimization} 
\label{sec:minimization}

\begin{figure}
  \centering
  {\raisebox{0.5cm}{\includegraphics[height=1.4in]{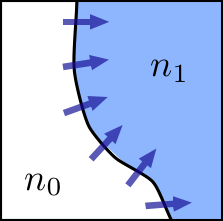}}}
	\hspace{0.5cm}
  {\includegraphics[height=1.75in]{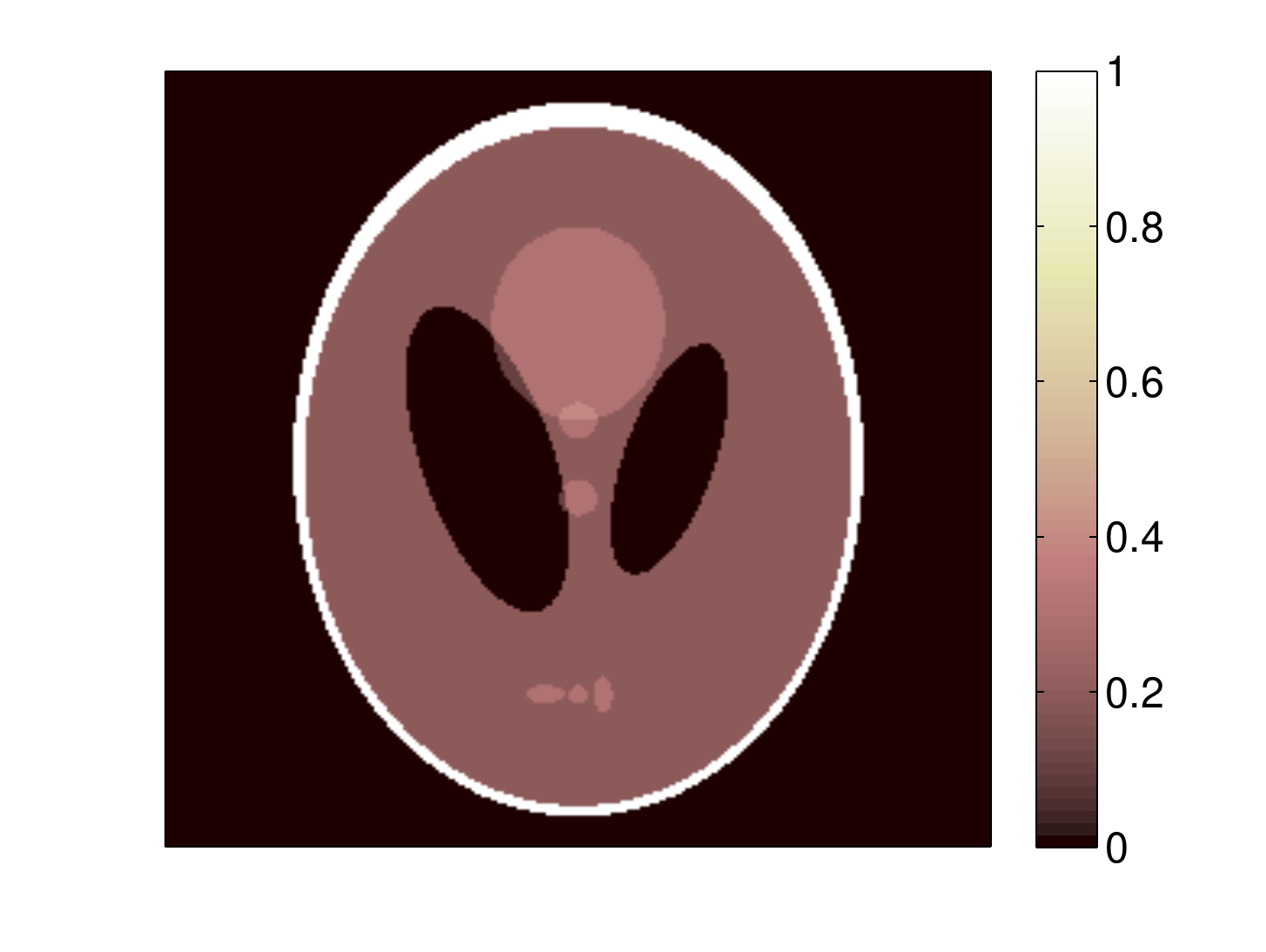}}
  \caption{(left) RIM TV Model: the gradient of the RIM (represented
    by arrows) is non-zero only on the interface between the two RIM
    homogeneous areas $\mathfrak n_0$ and $\mathfrak n_1$. This
    induces a small TV norm. (right) The Shepp-Logan phantom, an
    example of the ``Cartoon shape'' model.}
 \label{fig:RIM_SL}
\end{figure}

In order to overcome the limitations of both FBP and ME methods, we introduce a
new reconstruction method which is both less sensitive to unwanted
oscillations due to a low density frequency sampling~$\widehat{\cl P}_M$
and to additional observational noise $\bs \eta$ in
\eqref{eq:disc-measures-noise}. 

In particular, since the spatial dimensions in $\cl P_M$ and in $\cl
C_N$ are expected to be equal, \ie $N_0 \approx N_\tau$
(Sec.~\ref{sec:discrete-domains}) we are interested in lowering the
density of $\widehat{\cl P}_M$ in the Fourier plane by decreasing the
number of angular observations $N_\theta$. In other words, with
respect to this reduction, we aim at developing a numerical
reconstruction which makes Optical Deflectometric Tomography
``\emph{compressive}'' in a similar way other compressive imaging
devices which, inspired by the Compressed Sensing
paradigm~\cite{Candes2004,donoho2006cs}, reconstruct high resolution
images from few (indirect) observations of its content
\cite{duarte2008single,wiaux2009compressed,Sidky2011,lustig2007sparse}.
This ability would lead of course to a significant reduction of the
ODT observation time with potential impact, for instance, in fast
industrial object quality control relying on this technology.

This objective of compressiveness can only be reached by regularizing
the ODT inverse problem by an appropriate prior model on the
configuration of the expected RIM $\imapv$.  Interestingly, the actual
RIM of most human made transparent materials (\eg optical fiber
bundles, lenses, optics, ..) is composed by slowly varying areas
separated by sharp boundaries (material interfaces) (see
Fig.~\ref{fig:RIM_SL}-(left)). This can be interpreted with a Bounded
Variation (BV) or ``Cartoon Shape'' model \cite{rudin1992nonlinear} as
the typical Shepp-Logan phantom in
Fig.~\ref{fig:RIM_SL}-(right). Therefore, the inverse problem in
\eqref{eq:disc-measures-noise} can be regularized by promoting a small
Total-Variation norm, which in its discrete formulation is defined as
\cite{chambolle2004algorithm}
	$$ \|\imapv\|_{TV} := \| \bs \nabla \imapv \|_{2,1}, $$
where $\|\!\cdot\!\|_{2,1}$ is the mixed $\ell_2/\ell_1$ norm defined on
any $\bs u=(\bs u_1,\bs u_2) \in \Rbb^{N\times 2}$ as $\| \bs u \|_{2,1} = \sum_{k=1}^N
((\bs u_1)_k^2 + (\bs u_2)_k^2)^{1/2}$, and $\bs \nabla : \Rbb^N \rightarrow
\Rbb^{N\times 2}$ is the (finite difference) gradient operator. Reusing the 2-D
coordinates of $\imapv$, this operator is defined along each direction
as $\bs \nabla \imapv = (\bs \nabla_1 \imapv, \bs \nabla_2 \imapv)$
with $(\bs \nabla_1 \imapv)_{kl} = \imap_{k+1,l} - \imap_{k,l}$ and
$(\bs \nabla_2 \imapv)_{kl} = \imap_{k,l+1} - \imap_{k,l}$.

In order to obtain a reconstruction method which is also robust to
additive observation noise, we must lighten the strict fidelity
constraint implicitly used by the LS method in \eqref{eq:ls-method}.
Therefore, assuming the data corrupted by an additive white Gaussian
noise (AWGN), we impose a data fidelity requirement using the
$\ell_2$-norm, \ie if $\bs \eta=\bs y - \bs \Phi\imapv$ is known to
have a bounded norm (or \emph{energy}) $\|\bs \eta\|\leq \varepsilon$,
we force any reconstruction candidate $\bs u$ to satisfy $\|\bs y -
\bs \Phi \bs u\| \leq \varepsilon$. The way $\varepsilon$ can be
estimated will be explained in Sec.~\ref{sec:noise_estimation}.

Additionally to a fidelity criterion with the observed data, other
requirements can be imposed on the reconstruction. First, we can often
assume that the reference refractive index $\imap_{\rm r}$ (\ie the
one of some optical fluid surrounding the object) is lower than the object
RIM. Second, if the object is completely contained in the
field-of-view $\Omega$ of the ODT experiment, we can force any
candidate RIM $\bs u$ to match $\imap_{\rm r}$ on the boundary of
$\Omega$, \ie imposing $u_k = \imap_{\rm r}$ for all indices $k$
belonging to the border of $\cl C_N$. 
These indices are associated to pixels
$\bs r_{m,n}\in\cl C_N$ for which at least one of the two 2-D
coordinates is equal to either $-N_0/2$ or $N_0/2-1$. The
corresponding index set is denoted $\partial\Omega$ for simplicity.

Gathering all these aspects, we could propose the following
reconstruction program
$$ 
\imapv_{{\rm TV-}\ell_2} = \argmin_{\cmapv \in \bb R^N}
\|\cmapv\|_{TV} \ \st\ \|\bs y - \bs \Phi \cmapv \|_{2} \leq
\varepsilon, \ \cmapv \succeq \imapv_{\rm r}, \ \cmapv_{\partial
  \Omega} = \imap_{\rm r} \bs 1_{\partial \Omega}, 
$$
denoting by $\bs 1\in\Rbb^N$ the vector of ones, \ie the unit RIM in
$\cl C_N$, and recalling that $\bs v_{\cl A} = \bs R_{\cl A} \bs v$
stands for the restriction of the components of $\bs v\in\Rbb^N$ to
$\cl A\subset [N]$.

However, the reconstruction can be slightly simplified by observing
that the kernel of the sensing operator $\bs\Phi = \bs \Theta \bs D \bs F$ in ODT
contains the set of constant vectors in $\Rbb^N$. This is a
consequence of the vanishing affine frequency $\omega$ (which mainly
defines the action of $\bs D$) on the frequency origin, or more
simply, this relies on the occurrence of the RIM gradient in the
deflection model~(\ref{eq:DeflectionAngle}).

Therefore, a change of variable $\bs u \to \bs u - \imap_{\rm r}\bs
1$ does not disturb the previous reconstruction which can be recast as
\begin{equation}
\label{eq:inv-problem-deflecto-solving}
\imapv_{{\rm TV-}\ell_2} = \argmin_{\bs u \in \bb R^N} \|\bs u\|_{TV} \ \st\
\|\bs y - \bs \Phi \bs u \|_{2} \leq \varepsilon, \ \bs u \succeq 0, \
\bs u_{\partial \Omega} = 0, 
\end{equation}
remembering that the true RIM estimation is actually $\imapv_{{\rm TV}-\ell_2} +
\imap_{\rm r} \bs 1$.  

Without the \emph{frontier constraint} ($\bs
  u_{\partial \Omega} = 0$), the unicity of the solution is not
  guaranteed. In that case, for one minimizer $\bs u^*$, all the
  vectors of $\{\bs u^* + \lambda \bs 1:
\lambda\in\Rbb, \bs u^* + \lambda \bs 1 \succeq 0\}$ also minimize the problem since the kernels of both $\bs
\Phi$ and $\bs \nabla$ contain constant vectors. Considering the frontier
constraint is thus essential for enforcing the unicity of the solution
as expressed in the following lemma. 
\begin{lemma}
If there is at least one feasible point for the constraints of
(\ref{eq:inv-problem-deflecto-solving}), then the solution of this problem
is unique.
\end{lemma}
\begin{proof}
Using the TV norm definition (and squaring it), the TV-$\ell_2$
optimization \eqref{eq:inv-problem-deflecto-solving} is equivalent to solve
$$
\textstyle\argmin_{\bs u \in \bb R^N} \|\bs \nabla \bs u\|^2_{2,1} \ \st\
\|\bs y - \bs \Phi \bs u \|_{2} \leq \varepsilon, \ \bs u \succeq 0, \
\bs u_{\partial \Omega} = 0.
$$
The kernel of $\bs \nabla$ is the set of constant
vectors, while the one of $\bs
R_{\partial \Omega}$ (defining the frontier constraint) is the set of
vectors equal to 0 on $\partial \Omega$. Therefore, $\Ker \bs \nabla
\cap \Ker \bs R_{\partial\Omega} = \{\bs 0\}$. Moreover, since the domain of $\|\bs \nabla \cdot\|^2_{2,1}$ is $\Rbb^N$,
and since we assume at least one feasible point, \eqref{eq:inv-problem-deflecto-solving} has
at least one solution. 

Let $\bs x_1$ and $\bs x_2$ be two distinct minimizers.
Then, $\bs R_{\partial \Omega} \bs x_1 = \bs R_{\partial \Omega}
\bs x_2 = 0$ and $\bs x_1 - \bs x_2 \in \Ker \bs R_{\partial
  \Omega}$ while $\bs x_1 - \bs x_2 \neq \bs 0$. Therefore, $\bs x_1
- \bs x_2 \notin \Ker\bs\nabla$ and $\bs \nabla \bs x_1 \neq \bs
\nabla \bs x_2$. By the strict convexity of the function $\varphi(\cdot) =
\|\cdot\|^2_{2,1}$, writing $\bs x_\lambda = \lambda \bs x_1 +
(1-\lambda) \bs x_2$ for $\lambda\in(0,1)$, $\varphi(\bs \nabla \bs x_1) = \varphi(\bs \nabla
\bs x_2)$ involves
$$
\textstyle\varphi(\bs \nabla \bs x_1) = \lambda \varphi(\bs \nabla \bs x_1) +
(1-\lambda) \varphi(\bs \nabla \bs x_2) > \varphi(\bs \nabla \bs x_\lambda),
$$ 
showing that $\bs x_{\lambda}$, which also satisfies the convex constraints, is a better minimizer, which is a contradiction. 
\end{proof}

Therefore, we see that the uniqueness is actually reached by the
stabilizing condition $\bs u_{\partial \Omega} = \bs R_{\partial
  \Omega} \bs u = 0$ making the optimization running outside of $\ker
\bs\nabla \setminus \{\bs 0\}$.

As explained in Section~\ref{sec:gcp}, the program
(\ref{eq:inv-problem-deflecto-solving}) can be efficiently solved
using proximal methods \cite{combettes2011proximal} and operator
splitting techniques, like the recently proposed primal-dual algorithm
by Chambolle and Pock in~\cite{ChambollePock_2011}.

\section{NOISE IDENTIFICATION, ESTIMATION AND ANALYSIS}
\label{sec:noise_estimation}

In this section we first discuss about the different sources of noise
and how to estimate the noise energy present in the experimental
data. Then, we analyze the noise impact in both AT and ODT
measurements.

\subsection{Noise sources}
\label{sec:noise-sources}

When a real sensing scenario is being studied, such as the ODT,
different sources of noise are present and they have to be considered
when determining the global noise energy bound $\varepsilon$ in
\eqref{eq:inv-problem-deflecto-solving}.

First, we have the \emph{observation noise}. Under high light source
intensity, the images collected by a Schlieren deflectometer (Fig.~\ref{fig:deflectometry}-(bottom)), 
and used for computing the ODT deflections $\bs z$, are mainly
affected by electronic noise such as the CCD thermal noise. 
This induces a noise in the measured deflection angles that 
can reasonably be assumed Gaussian and homoscedastic, \ie 
with an homogeneous variance through all the measurements.
By computing the 1-D Fourier transform of the ODT measurements using
$\bs \Theta \bs F_{\rm rad}$ in (\ref{eq:F-rad-def-real}) the
corresponding noise $\bs \eta_{\rm obs}$ remains
Gaussian~\cite{morita1995fourier} in the FDM $\bs y$, \ie $\bs
\eta_{\rm obs} \sim \cl N(0,\sigma_{\rm
    obs}^2)$. Actually, from the orthonormality of the Fourier
  basis, \eqref{eq:F-rad-def-real} provides
  $\sigma_{\rm obs}^2 = (\delta\tau)^2 N_\tau
  \sigma_{z}^2$, where $\sigma_{z}^2$ is the variance of the noise
  present in the ODT measurements ($\bs z$). This one can be estimated
  from the noisy observations using the Robust Median Estimator~\cite{donoho94ideal,taswell00}.

Finally, this noise, defined as the difference between the noisy FDM
$\bs y$ and the
noiseless FDM $\bs y_{\rm true}$, has an energy that can be bounded
using the Chernoff-Hoeffding bound~\cite{hoeffding1963}:
$$
\textstyle\| \bs y - \bs y_{\rm true} \|^2 = \| \bs \eta_{\rm obs} \|^2 <
\varepsilon_{\rm obs}^2 := \sigma_{\rm obs}^2(M + c
\sqrt{M}),
$$ 
with high probability for $c = \cl O(1)$.

Second, there is the \emph{modeling error} that comes with every mathematical
discrete representation of a physical continuous system. In the ODT
system, this error is due to \emph{(i)} the first order approximation used to
formulate \eqref{eq:DeflectionAngle}, \emph{(ii)} the sampling of the
continuous RIM and \emph{(iii)} the discrete model itself. The modeling noise
is related to the difference between the noiseless FDM and the sensing
model $\bs \Phi_{\rm true} \imapv = \bs \Theta \bs D \bs F_{\rm true} \imapv$,
where $\bs F_{\rm true}$ performs the exact Polar Fourier
Transform. As explained in Sec.~\ref{sec:principles}, the
  modeling noise can be estimated by means of a \emph{ray tracing}
  method based on the Snell law (not detailed here). This shows that for simple objects,
  an absolute error of $0.7^{\circ}$ is expected on light deflections smaller than $7^{\circ}$. A~Gaussian noise model provides a
  rough estimation of $10$ dB for the
  corresponding measurement SNR. This is equivalent to:
\begin{equation}
\label{eq:epsilon_model}
\| \bs y_{\rm true} - \bs \Phi_{\rm true} \imapv \| = \| \bs
\eta_{\rm model} \| < \varepsilon_{\rm model} \ = {\| \bs
    y_{\rm true} \|}/{\sqrt{10} \simeq \| \bs
    y \|}/{\sqrt{10}}.
\end{equation}

Third, we must consider the \emph{interpolation noise} given by the mathematical
error committed by estimating the polar Fourier Transform with the
NFFT algorithm, \ie the noise $\bs \Phi_{\rm true} \imapv - \bs \Phi
\imapv$.  To determine a bound $\varepsilon_{\rm nfft}$ on the energy
of this error we first estimate the NFFT distortion (\ie without the
action of $\bs D$), defined as the difference between the NFFT polar
Fourier Transform $\bs F_{\rm app}$ and the true NDFT $\bs
F$. Theoretically, for any vector $\bs f \in \Rbb^N$, the
$\ell_\infty$-norm of this distortion is bounded as $\|\bs F_{\rm app}
\bs f - \bs F \bs f\|_\infty \leq \epsilon\,C(\bs f) = \cl O(\epsilon
\| \bs f \|_1)$, where $\epsilon$ controls both the accuracy and the
complexity $\cl O(N \log N/\epsilon)$ of the NFFT~\cite{KeinerNFFT} (see
Appendix~\ref{app:NFFT}). Assuming that each component of $\bs F_{\rm
  app} \bs f - \bs F \bs f$ is iid with a uniform distribution $\cl
U([-C(\bs f),C(\bs f)])$, and using the Chernoff-Hoeffding
bound~\cite{hoeffding1963,Dequant}, we can estimate $\|\bs F_{\rm app}
\bs f - \bs F \bs f\|^2 < \frac{C(\bs f)^2}{3} (M + c\,\sqrt{M}),$
with high probability for $c = \cl O(1)$.  Finally, $\varepsilon_{\rm
  nfft}$ can be crudely computed as
\begin{align*}
\| \bs \Phi_{\rm true} \imapv - \bs \Phi \imapv \|&= \|\bs\Theta \bs
D(\bs F_{\rm app}
\imapv - \bs F \imapv) \| \leq \bbb\bs \Theta \bs D \bbb \|\bs F_{\rm app}
\imapv - \bs F \imapv\|\\
&= \tfrac{2\pi (\delta
r)^2\omega_{\rm max}}{\imap_{\rm r}} \tfrac{\epsilon C(\imapv)}{\sqrt 3}
(M + c\,\sqrt{M})^{1/2}\\
&\approx  \tfrac{\pi \delta r}{\sqrt 3\,\imap_{\rm r}} \,\epsilon\,C(\imapv) (M +
c\,\sqrt{M})^{1/2} =: \varepsilon_{\rm nfft}, 
\end{align*}
with $\omega_{\rm max} \approx \tinv{2} N_\tau \delta\omega =
\tinv{2\delta \tau} \approx \tinv{2\delta r}$ representing the maximum
frequency amplitude in $\widehat{\cl P}_M$. In practice, because of the
RIM shift $\imapv \to \imapv - \bs 1 \imap_{\rm r}$ explained in
Sec.~\ref{sec:minimization}, we can bound $\|\imapv\|_1$ and hence
$C(\imapv)$ with the expected RIM dynamics $\delta \imap$, \ie
$\|\imapv\|_1\leq N \delta\imap$, and we adjust $\epsilon$ in order to
have $\varepsilon_{\rm nfft}$ much lower than the other sources of
noises.
 
Finally, we may also have an error introduced by the \emph{instrument
calibration}, when determining the exact $\tau$ and $\theta$ associated
to the projections. We are going to neglect this error by assuming a
pre-calibration process that provides an exact knowledge of these
values (see Sec.~\ref{sec:expdata}).

In conclusion, gathering all the previous noise identifications and assuming the different noise are of zero mean and independent of each other, we
can bound the difference between the actual ODT measurement and the
sensing model as follows:
\begin{align*}
  \| \bs y - \bs \Phi \imapv \|^2 & = \| \bs y - \bs y_{\rm true} + \bs y_{\rm true} - \bs \Phi_{\rm true} \imapv + \bs \Phi_{\rm true} \imapv - \bs \Phi \imapv \|^2 \\
  & = \| \eta_{\rm obs} + \eta_{\rm model} + \eta_{\rm nfft} \|^2 \ <\ \varepsilon_{\rm obs}^2 + \varepsilon_{\rm model}^2 +
  \varepsilon_{\rm nfft}^2.
\end{align*}
Therefore, we have $\varepsilon \approx \sqrt{\varepsilon_{\rm obs}^2 + \varepsilon_{\rm model}^2 +  \varepsilon_{\rm nfft}^2}$.

\subsection{Comparative study of the noise impact on AT and ODT}
\label{sec:noise_tomo_vs_deflect}

As we have seen in Sec.~\ref{sec:forward-model}, the main difference
between the AT and ODT problems is the appearance of the diagonal
operator $\bs D$ in the last one. Therefore, the AT sensing operator is described as $\bs \Phi_{\rm AT} = \bs \Theta \bs F$ and the ODT sensing operator as $\bs \Phi_{\rm ODT} = \bs \Theta \bs D \bs F$. We will now analyze the impact of
an additive white Gaussian noise on the measurements regarding the presence of this operator.

For this, we apply the sensing operators $\bs \Phi_{\rm AT}$ and $\bs \Phi_{\rm ODT}$ to a section of the fibers bundle (see Fig.~\ref{fig:RIM_fibers_sphere}-(left)), in order to obtain the AT and ODT acquisitions, respectively. In Fig.~\ref{fig:Tomo_Measure} and Fig.~\ref{fig:Deflecto_Measure} we show the real part of the Fourier measurements in AT and ODT.

For the class of images we are interested in, we can notice than in AT the
magnitude presents a peak around $\omega = 0$ and then decreases
significantly when the distance to the center increases, tending to
zero in the borders (see Fig.~\ref{fig:Tomo_Measure}). Whereas in ODT,
the presence of operator $\bs D$ makes the image intensity to be quite
spread through all the pixels (see
Fig.~\ref{fig:Deflecto_Measure}). This has a direct impact on the
reconstruction when the measurement is affected by additive Gaussian
noise. As the noise spreads evenly through the image, the pixels that
are not around $\omega = 0$ will be more affected in the AT model
because their intensity is significantly lower.

\begin{figure}
  \centering
	\includegraphics[height=4.55cm]{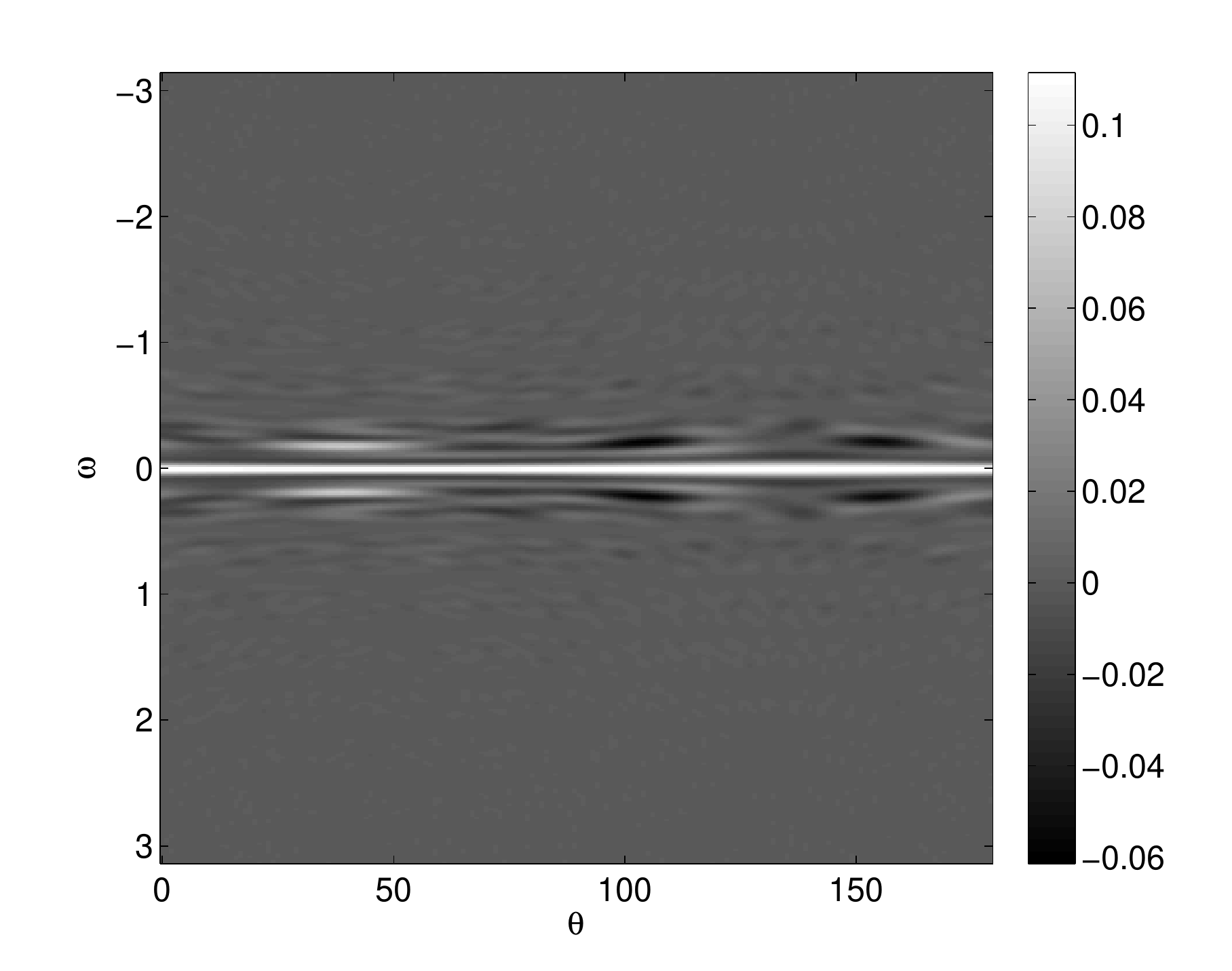}
	\hspace{0.5cm}
	\includegraphics[height=4.55cm]{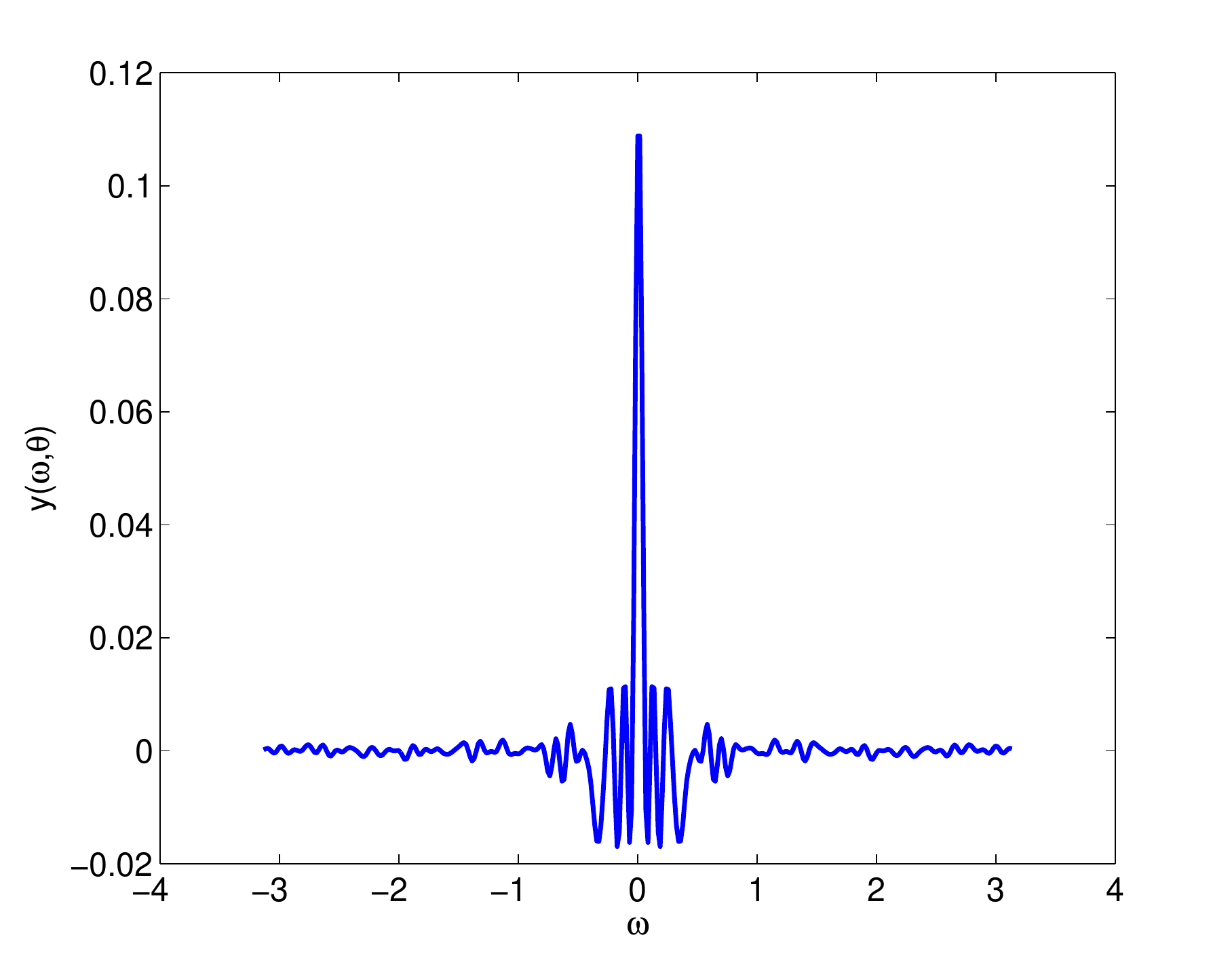}
  \caption{AT Measurement (in Fourier) (left) On the whole grid
  $\widehat{\cl P}_M$ for $N_\theta = 180$. (right) The slice $\theta = 80^\circ$.}
 \label{fig:Tomo_Measure}
\end{figure}

\begin{figure}
  \centering
	\includegraphics[height=4.55cm]{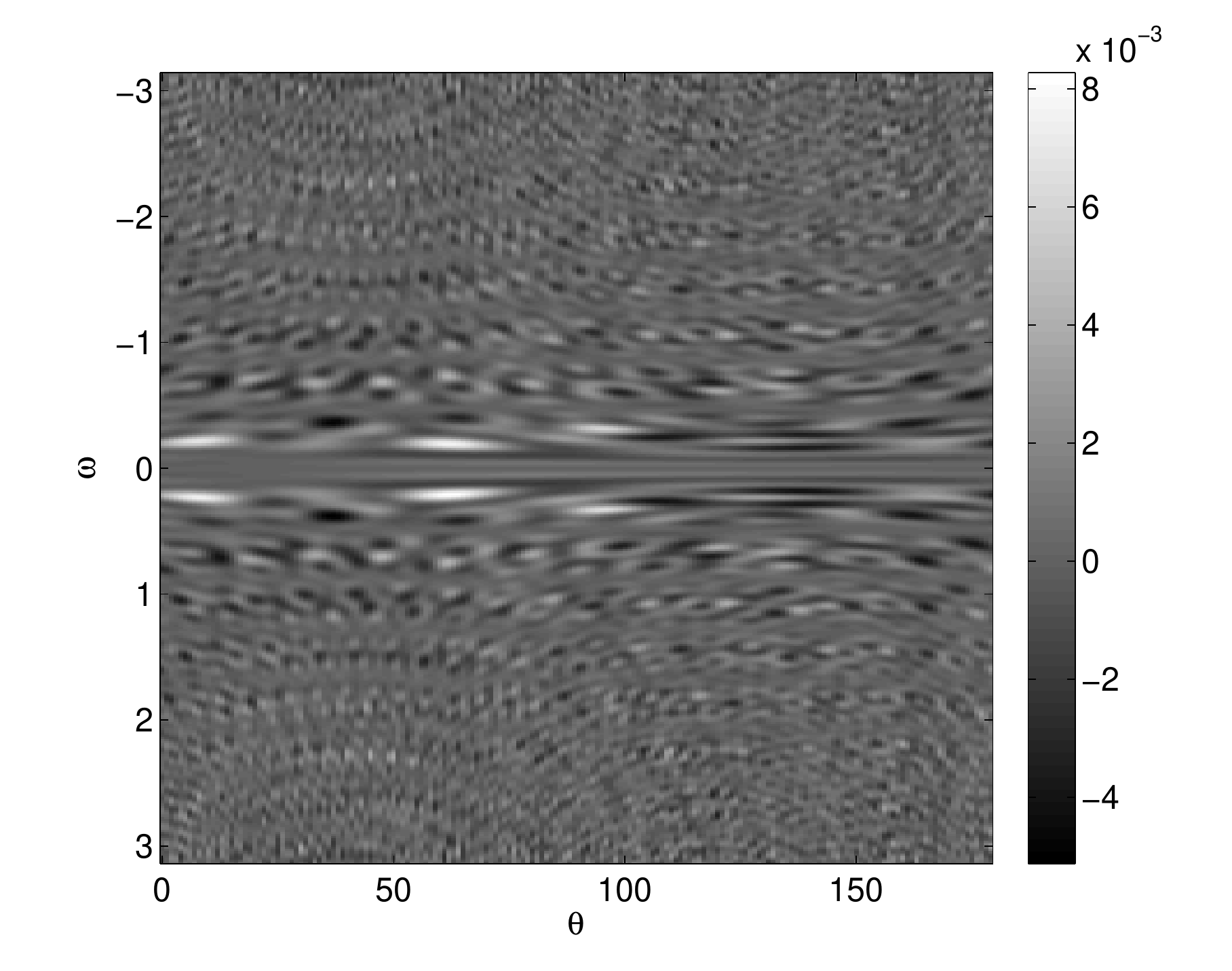}
	\hspace{0.5cm}
	\includegraphics[height=4.55cm]{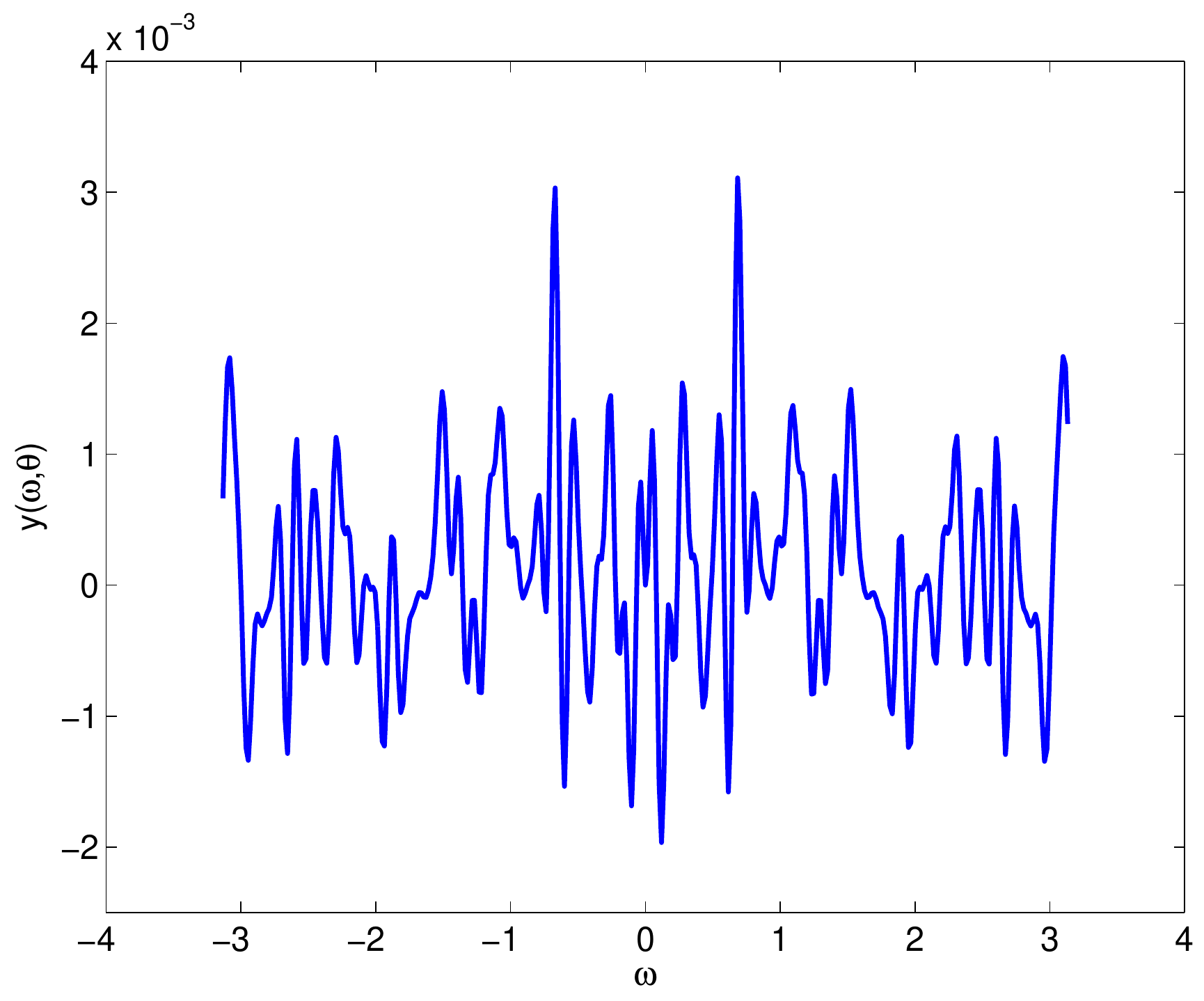}
  \caption{ODT Measurement (in Fourier) (left) On the whole grid
  $\widehat{\cl P}_M$ for $N_\theta = 180$. (right) The slice $\theta = 80^\circ$.}
 \label{fig:Deflecto_Measure}
\end{figure}

\section{NUMERICAL METHODS} 
\label{sec:gcp}

This section provides first an overview of the Chambolle-Pock
primal-dual algorithm \cite{ChambollePock_2011}
used for solving \eqref{eq:inv-problem-deflecto-solving}
and hence
recovering the RIM from the noisy FDM. The algorithm is then
generalized into a product space optimization that allows the
minimization of more than two convex functions. We show next how to use this generalization to solve our ODT
problem under multiple constraints, before to briefly present an adaptive selection of the
  optimization parameters due to \cite{goldstein2013} offering faster convergence speeds.

\subsection{Chambolle-Pock Primal-Dual Algorithm}

The Chambolle-Pock (CP) algorithm (Algorithm 1
in~\cite{ChambollePock_2011}) is an efficient, robust and flexible
algorithm that allows to solve minimization problems of the form:
\begin{equation}
\label{eq:primal_CP}
	\min_{\bs x \in \Rbb^N} \ F(\bs K \bs x) + G(\bs x),
\end{equation}
for a linear operator $\bs K : \Rbb^N \rightarrow \Rbb^W$ and any
variable $\bs x \in \Rbb^N$. The functions $F$ and $G$ belong to the
functional sets $\Gamma^0(\Rbb^W)$ and $\Gamma^0(\Rbb^N)$,
respectively.

In short, CP solves the primal problem described above simultaneously
with its dual problem, until the difference between their objective
functions -- the primal-dual gap -- is zero.

For any variable $\bs v \in \Rbb^W$, the primal-dual optimization can
be formulated as the following saddle-point problem:
\begin{equation}
\label{eq:saddle-point-CP}
	\min_{\bs x \in \Rbb^N} \max_{\bs v \in \Rbb^W} \langle \bs K \bs x, \bs v \rangle + G(\bs x) - F^\star(\bs v),
\end{equation}
where $F^\star$ is the convex conjugate of function $F$ provided by the Legendre transform $F^\star(\bs v) = \max_{\bar{\bs v} \in \Rbb^W} \langle \bs v, \bar{\bs v} \rangle - F(\bar{\bs v})$.

Using the Legendre transform we obtain the primal version described in
\eqref{eq:primal_CP} and also the dual version as follows:
\begin{equation}
\label{eq:dual_CP}
\max_{\bs v \in \Rbb^W} -F^\star(\bs v) - G^\star(-\bs K^* \bs v),
\end{equation}
where $\bs K^*$ is the exact adjoint of the operator $\bs K$, such that $ \langle \bs K \bs x, \bs v \rangle \ = \ \langle \bs x, \bs K^* \bs v \rangle$.

The CP algorithm is defined by the following iterations:
\begin{equation}
\label{eq:CP_Algorithm}
\begin{cases}
  \bs v^{(k+1)}&= \prox_{\nu F^\star} (\bs v^{(k)} + \nu \bs K \bar{\bs x}^{(k)}),\\
  \bs x^{(k+1)}&= \prox_{\mu G} (\bs x^{(k)} - \mu \bs K^\star \bs v^{(k+1)}),\\
  \bar{\bs x}^{(k+1)}&= \bs x^{(k+1)} + \vartheta (\bs x^{(k+1)} - \bs x^{(k)}).
\end{cases}
\end{equation}

The quantity $\prox_f$ denotes the \emph{proximal operator} of a
convex function $f \in \Gamma^0(\cl V)$ for a certain finite
dimensional vector space $\cl V$
\cite{combettes2011proximal,moreau1962fonctions}. This operator is
defined as:
$$ \prox_{f} \bs \zeta\ :=\ \argmin_{\bs \zeta'\in \cl V} 
f(\bs \zeta') + \tfrac{1}{2} \| \bs \zeta - \bs \zeta'\|^2,\qquad \bs
\zeta \in \cl V.$$ The proximal operator admits the use of non-smooth
convex functions as the TV norm, making the algorithm suitable to
solve the TV-$\ell_2$ problem described in
Sec.~\ref{sec:reconstruction}.

Most numerical methods require operator $\bs K$ being in a tight
frame, which is not the case for our sensing operator $\bs \Phi$. The
CP algorithm reduces the convergence requirements, since we only need
to tune the step sizes $\mu$ and $\nu$ such that the condition $ \mu
\nu \bbb \bs K \bbb^2 < 1 $ is true for any operator $\bs
K$. Moreover, as presented in Sec.~\ref{sec:adaptive-odt}, these
two parameters can be adaptively tuned during the iterations in order
to reach faster convergence \cite{goldstein2013}.

There is an easy way to estimate the convergence of the
  primal-dual algorithm~\eqref{eq:CP_Algorithm}. This relies on
  evaluating the \emph{primal and dual residuals} defined by the
  subgradient of the saddle-point problem \eqref{eq:saddle-point-CP}
  with respect to the primal and dual variables, namely, $P(\bs x, \bs v) := \partial G(\bs x) +
  \bs K^* \bs v$ and $D(\bs x, \bs v) := \partial F^\star(\bs v) - \bs K \bs x$, respectively
  ~\cite{goldstein2013}.

By definition, for an optimal point ($\tilde{\bs x},\tilde{\bs
    v}$) of \eqref{eq:saddle-point-CP}, zero must belong to these
  residuals and, therefore, by tracking the size of these residuals we
  can perform an analysis of the algorithm convergence. Explicit
  formulas can be obtained for the primal and dual residuals by
  observing the optimality conditions of \eqref{eq:CP_Algorithm} at
  each iteration. This provides
\begin{align*}
	0 & \in \mu \partial G(\bs x^{(k+1)}) + \bs x^{(k+1)} - \bs x^{{(k)}} + \mu \bs K^* \bs v^{{(k+1)}},\\ 
	0 & \in \nu \partial F^\star(\bs v^{(k+1)}) + \bs v^{(k+1)} - \bs v^{{(k)}} - \nu \bs K (\bs x^{(k)}(1 + \vartheta) - \vartheta \bs x^{(k-1)}). 
\end{align*}
or, equivalently, $\bs P^{(k+1)} \in P(\bs x^{(k+1)},\bs
v^{{(k+1)}})$ and $\bs D^{(k+1)} \in D(\bs x^{(k+1)},\bs v^{{(k+1)}})$
with the primal and dual residual vectors 
\begin{equation}
  \label{eq:PD_residuals_CP}
  \begin{array}{ll}
    \bs P^{(k+1)} = \tinv{\mu}(\bs x^{(k)} - \bs x^{{(k+1)}}),\\
    \bs D^{(k+1)} = \tinv{\nu}(\bs v^{(k)} - \bs v^{{(k+1)}}) +
    \bs K\big((1 + \vartheta)\bs x^{(k)} -\!\vartheta \bs x^{(k-1)} -
    \bs x^{(k+1)}\big).
  \end{array} 
\end{equation}
Goldstein \emph{et al} \cite{goldstein2013} show experimentally that a converging algorithm respects
\begin{equation}
\label{eq:convergence_condition}
	\lim_{k \rightarrow \infty} \| \bs P^{(k)} \|_*^2 + \| \bs D^{(k)} \|_*^2 = 0,
\end{equation}
for some norm $\|\cdot\|_*$ (\eg $\ell_1$). We will analyze the same convergence measure during our experiments in
Sec.~\ref{sec:experiments}.

\subsection{Product Space Optimization}

In this paper, we are interested in minimization problems containing
more than two convex functions. In particular, we aim at solving the
general optimization
$$ \min_{\bs x \in \bb R^{N}} \sum_{j=1}^{p} F_j(\bs K_j \bs x) + H(\bs x), $$
with $\bs K_j : \Rbb^N \rightarrow \Rbb^{W_j}$ and $p+1$ the number of
convex functions. Such a problem does not allow the direct use of the
CP algorithm as described before. However, it is easy to adapt it by
considering a $p$-times expanded optimization space $\Rbb^{pN}$. This
space is composed of $\bs x' = (\bs x_1^T, \cdots, \bs x_p^T)^T \in
\Rbb^{pN}$, with $\bs x_j \in \Rbb^N$. In this context, we define
$p-1$ bisector planes $\Pi_{1,j} = \{ \bs x' \in\Rbb^{pN}: \bs x_1 =
\bs x_j,\ 2 \leq j \leq p\}$ in order to work with the following
equivalent primal problem:
\begin{equation}
\label{eq:primal_expanded}
 \min_{\bs x' \in \bb R^{pN}} \sum_{j=1}^{p} F_j(\bs K_j \bs x_j) + \sum_{j=2}^{p} \imath_{\Pi_{1,j}} (\bs x') + H(\bs x_1). 
\end{equation}

The CP-shape \eqref{eq:primal_CP} is thus recovered by working in this
bigger space $\Rbb^{pN}$ and by setting $F(\bs s) = \sum_{j=1}^{p}
F_j(\bs s_j)$, with $\bs s = (\bs s_1^T, \cdots, \bs s_p^T)^T \in
\Rbb^{W=\sum_j W_j}$ and $\bs s_j \in \Rbb^{W_j}$, $\bs K = \diag(\bs
K_1, \cdots, \bs K_p) \in \Rbb^{W \times pN}$ and $G(\bs x') =
\sum_{j=2}^{p} \imath_{\Pi_{1,j}} (\bs x') + H(\bs x_1)$.

In this expanded optimization space, the equivalent dual problem is
written (see Appendix~\ref{sec:conv-conj-funct}):
$$ \max_{\bs s \in \bb R^{W}} \ - \sum_{j=1}^{p} F^\star_j(\bs s_j) - H^\star\bigg(- \sum_{j=1}^{p} \bs K_j^* \bs s_j \bigg). $$

For the functions described above, it is easy to see that for any $\nu
> 0$ and $\bs \zeta = (\bs \zeta_1^T, \cdots, \bs \zeta_p^T)^T \in \Rbb^W$ we have:
$$ \prox_{\nu F^\star} \bs \zeta = \begin{pmatrix}\prox_{\nu F_1^\star} \bs \zeta_1\\\vdots\\\prox_{\nu F_p^\star} \bs \zeta_p\end{pmatrix}, $$
and, for any $\mu > 0$ we have (more details in Appendix~\ref{app:proxg}):
\begin{displaymath}
  \prox_{\mu G} \bs \zeta = 
  (\bs \Id^N, \,\cdots, \bs \Id^N)^T
  \prox_{\frac{\mu}{p} H} \big( \tfrac{1}{p}{\textstyle \sum_{j}} \bs \zeta_j\big).
\end{displaymath}

\subsection{ODT numerical reconstruction}

Now we need to transform our TV-$\ell_2$ problem into an expanded form
in order to use the CP algorithm. Having two constraints, the
optimization space needs to be expanded by $p = 2$. Using
\eqref{eq:primal_expanded}, we can reformulate the primal problem
from \eqref{eq:inv-problem-deflecto-solving} as
\begin{equation}
  \label{eq:primal_OD}
        \min_{\bs x'=(\bs x_1,\bs x_2) \in \bb R^{2N}}\ \| \bs \nabla \bs x_1 \|_{2,1} + \imath_{\cl C} (\bs \Phi \bs x_2) + \imath_{\cl P_0} (\bs x_1) + \imath_{\Pi_{1,2}} (\bs x'),
\end{equation}
where $\cl C=\{\bs v\in\Rbb^M:\|\bs y - \bs v\|\leq \varepsilon\}$ and
$\cl P_0=\{\bs u\in\Rbb^N:u_{i} \geq 0 \ \textrm{if} \ i \in
\textrm{int} \ \Omega; \ u_{i} = 0 \ \textrm{if} \ i \in \partial
\Omega\}$. We show easily that \eqref{eq:primal_OD} has the shape of
\eqref{eq:primal_expanded} with $F_1 (\bs s_1) = \| \bs s_1 \|_{2,1}$
for $\bs s_1 \in \Rbb^{N \times 2} \simeq \Rbb^{2N}$, $F_2(\bs s_2) = \imath_{\cl C} (\bs s_2)$ for $\bs s_2 \in
\Rbb^M$, $H(\bs x_1) = \imath_{\cl P_0} (\bs x_1)$, $\bs K_1 = \bs
\nabla$ and $\bs K_2 = \bs \Phi = \bs \Theta \bs D \bs
F\in\Rbb^{M\times N}$.

For building the dual problem, we need the conjugate functions of
$F_1$, $F_2$ and $H$, which are easily computed using the Legendre
transform. As a matter of fact, $F_1^\star$ is the indicator function onto
the convex set $\cl Q = \{ \bs q = (\bs q_1, \bs q_2)\in \Rbb^{N
  \times 2} : \| \bs q \|_{2,\infty} \leq 1\}$ with the mixed
$\ell_\infty/\ell_2$-norm defined as $\| \bs q \|_{2,\infty} = \max_k
\sqrt{(\bs q_1)_k^2 + (\bs q_2)_k^2}$~\cite{ChambollePock_2011}. The
conjugate function $F_2^\star$ is computed as (see
Appendix.~\ref{sec:conv-conj-indic})
$$ F_2^\star(\bs s_2) = \imath_{\cl C}^\star(\bs s_2) = (\bs s_2)^T \bs y + \varepsilon \| \bs s_2 \|,$$ 
while the convex conjugate of $H$ is simply $H^\star(\bs u) = \imath_{\cl D} (- \bs u)$, where $\cl D=\{\bs u\in\Rbb^N:u_{i} \geq 0 \ \textrm{if} \ i \in \textrm{int} \ \Omega\}$.

The dual optimization problem is thus defined as:
$$ 
\max_{\bs s \in \Rbb^{2N + M}} -\imath_{\cl Q} (\bs s_1) - \langle \bs
s_2, \bs y \rangle - \varepsilon \| \bs s_2 \| - \imath_{\cl D} (\bs
\nabla^* \bs s_1 + \bs \Phi^* \bs s_2). 
$$

In order to apply \eqref{eq:CP_Algorithm}, we must compute the
proximal operators of $F_1^\star$, $F_2^\star$ and $H$. The one of $F_1^\star$ is
given by~\cite{ChambollePock_2011}
$$ (\prox_{\nu F_1^\star} \bs \zeta)_{k} = \tfrac{\big((\bs
\zeta_1)_{k},(\bs \zeta_2)_{k}\big)}{\max(1, \sqrt{(\bs \zeta_1)_{k}^2
+ (\bs \zeta_2)_k^2})},\quad \bs \zeta =(\bs \zeta_1,\bs \zeta_2)\in\Rbb^{N\times 2}.$$ 

The proximal operator of $F_2^\star$ is determined via the proximal
operator of $F_2$ by means of the conjugation property defined
in~\cite{combettes2011proximal}:
$$\textstyle\prox_{\nu F_2^\star} \bs s_2 = \bs s_2 - \nu
\prox_{\frac{1}{\nu}F_2} \frac{1}{\nu} \bs s_2.$$
 
The proximal operator of $F_2$ is given by the projection onto the
convex set $\cl C$:
$$ 
\textstyle\prox_{\frac{1}{\nu} F_2} \bs s_2 = \bs y + (\bs s_2 - \bs y) \min
(1, \frac{\varepsilon}{ \| \bs s_2 - \bs y \| }).
$$
The proximal operator of the function $H$ represents a projection onto
the positive orthant with zero borders:
\begin{displaymath}
 \prox_{\frac{\mu}{2} H} \bs \xi = \textrm{proj}_{\cl P_0} \bs \xi = 
 \begin{cases}
   \big( \xi_i \big)_+ & \text{if } i \in \intset \Omega,\\
   0 & \text{if } i \in \partial\Omega, 
 \end{cases},\qquad \bs \xi \in \Rbb^N.
\end{displaymath}

Finally, making use of the above computations and taking $\vartheta =
1$, the CP algorithm applied to our TV-$\ell_2$ problem in ODT can be
reduced to (see Appendix~\ref{app:FinalFormulation})
\begin{equation}
\label{eq:CP_Algorithm2}
\begin{cases}
  \bs s_1^{(k+1)}&= \prox_{\nu F_1^\star} \big(\bs s_1^{{(k)}} + \nu \bs \nabla \bar{\bs x}^{(k)} \big),\\
  \bs s_2^{(k+1)}&= \prox_{\nu F_2^\star} \big(\bs s_2^{{(k)}} + \nu \bs \Phi \bar{\bs x}^{(k)} \big),\\
  \bs x^{(k+1)}&= \textrm{proj}_{\cl P_0} \big(\bs x^{(k)} - \frac{\mu}{2} (\bs \nabla^* \bs s_1^{{(k+1)}} + \bs \Phi^* \bs s_2^{{(k+1)}}) \big),\\
  \bar{\bs x}^{(k+1)}&= 2 \bs x^{(k+1)} - \bs x^{(k)}.  
\end{cases}
\end{equation}
In our experiments, the variables $\bar{\bs x}^{(0)}$, $\bs s_1^{{(0)}}$ and $\bs s_2^{{(0)}}$ are initialized to zero vectors and the variable $\bs x^{(0)}$ is initialized with the FBP solution as $\bs x^{(0)} = \imapv_{\rm FBP}$.

The algorithm presented in \eqref{eq:CP_Algorithm2} stops when it
achieves a stable behavior, \ie when $\| \bs x^{(k+1)} - \bs x^{(k)}
\| / \| \bs x^{(k)} \| \leq \mathtt{Th}$. The threshold $\mathtt{Th}$
is defined for ODT in the next section. In parallel, we analyze the
convergence of the algorithm by means of the primal and dual residuals as
described in \eqref{eq:PD_residuals_CP}. For the iterates in \eqref{eq:CP_Algorithm2}, the primal and dual residuals are defined as:
\begin{align}
\label{eq:PD_residuals_CP_ODT}
	\bs P^{(k+1)} & = \tfrac{2}{\mu} (\bs x^{(k)} - \bs x^{{(k+1)}}), \notag \\
	\bs D^{(k+1)} & = \begin{pmatrix} \bs D_1^{(k+1)} \\ \bs D_2^{(k+1)} \end{pmatrix} = \begin{pmatrix} \tinv{\nu} (\bs s_1^{(k)} - \bs s_1^{{(k+1)}}) + \bs \nabla (2 \bs x^{(k)} - \bs x^{(k-1)} - \bs x^{(k+1)}) \\ \tinv{\nu} (\bs s_2^{(k)} - \bs s_2^{{(k+1)}}) + \bs \Phi (2 \bs x^{(k)} - \bs x^{(k-1)} - \bs x^{(k+1)}) \end{pmatrix}. 
\end{align}

In order to guarantee the convergence of the algorithm, \ie to ensure that $\bs x^{(k)}$ converges to the solution of (\ref{eq:inv-problem-deflecto-solving}) when $k$ increases, we need to set $\mu$ and $\nu$ such that $\mu \nu \bbb \bs K \bbb^2 < 1$. The induced norm of the operator ($\bbb \bs K \bbb$) was computed as explained in~\cite{Sidky2011} using the standard power iteration algorithm to calculate the largest singular value of the associated matrix $\bs K$.

\subsection{Adaptive optimization procedure}
\label{sec:adaptive-odt}

Until now, a non-adaptive optimization method, \ie with
  constant step-sizes $\mu$ and $\nu$, has been considered. However,
  such a procedure often presents slow convergence as presented in
  Section~\ref{sec:experiments}. Therefore, motivated by the recent
  work by Goldstein et al.~\cite{goldstein2013}, an adaptive
  version of the CP algorithm is used for the ODT reconstructions (See
  Algorithm \ref{alg:adaptive_CP_ODT}). While this variant also depends on a
  couple of parameters that adjust the algorithm adaptivity,
  these are more naturally related to the characteristics of the
  inverse problem solved by this optimization.

In this adaptive approach, the stepsize update are tuned to the
  size of the primal and dual residuals at each iteration. If the
  primal residual is large with respect to the dual residual, then the
  primal stepsize $\mu^{(k)}$ is increased by a factor $\tinv{1 -
    \rho^{(k)}}$, and the dual stepsize $\nu^{(k)}$ is decreased by
  the same factor. If the dual residual is large with respect to the
  primal residual, then the primal stepsize is decreased and the dual
  stepsize is increased. The parameter $\rho^{(k)}$ is a constant that
  controls the adaptivity level of the method, and it is updated as
  $\rho^{(k+1)} = \beta \rho^{(k)}$, for some $\beta < 1$. In
  Algorithm~\ref{alg:adaptive_CP_ODT}, the parameter $\Gamma > 1$ is
  used to compare the sizes of the primal and dual residuals and the
  scaling parameter $c > 0$, which depends on the image expected
  dynamics, is used to balance the residuals. The specific values
  selected for the parameters of Algorithm~\ref{alg:adaptive_CP_ODT} are those recommended by \cite{goldstein2013}. 
\begin{algorithm}[t]
\footnotesize
\newcommand{\sComment}[1]{\vspace{1mm}\Statex\quad\ \emph{\underline{#1}}:\vspace{1mm}}
  \caption{Adaptive ODT reconstruction
    \label{alg:adaptive_CP_ODT}}
  \begin{algorithmic}[1]
    \Require{\parbox[t]{10cm}{$\bar{\bs x}^{(0)} = \bs s_1^{(0)} = \bs s_2^{(0)} = \bs
      0$, $\bs x^{(0)} = \imapv_{\rm FBP}$, $\nu^{(0)} = \mu^{(0)} =
      0.9/\bbb \bs K \bbb$, $\Gamma = 1.1$, $\rho^{(0)} = 0.5$,
      $\beta = 0.95$, $c = 1000$, MaxIter = $50\times10^4$.}\vspace{2mm}}
			\For{$k = 1$ to MaxIter}
                        \sComment{Compute the CP iterations \eqref{eq:CP_Algorithm2}}
                        	\State $\bs s_1^{(k+1)} = \prox_{\nu^{(k)} F_1^\star} \big(\bs s_1^{(k)} + \nu^{(k)} \bs \nabla \bar{\bs x}^{(k)} \big)$
\State $\bs s_2^{(k+1)} = \prox_{\nu^{(k)} F_2^\star} \big(\bs s_2^{(k)} + \nu^{(k)} \bs \Phi \bar{\bs x}^{(k)} \big)$
	\State $\bs x^{(k+1)} = \textrm{proj}_{\cl P_0} \big(\bs x^{(k)} - \frac{\mu^{(k)}}{2} (\bs \nabla^* \bs s_1^{{(k+1)}} + \bs \Phi^* \bs s_2^{{(k+1)}}) \big)$
				\State $\bar{\bs x}^{(k+1)} = 2 \bs
                                x^{(k+1)} - \bs x^{(k)}$
                                \vspace{2mm}\sComment{Compute the primal and dual residual norms \eqref{eq:PD_residuals_CP_ODT}}
                                \State $p^{(k+1)} = \left\| \frac{2}{\mu^{(k)}} (\bs x^{(k)} - \bs x^{{(k+1)}}) \right\|_1$ 
				\State $d^{(k+1)} = 
                                \left\| \begin{pmatrix}
                                      \tinv{\nu^{(k)}} (\bs s_1^{(k)}
                                      - \bs s_1^{{(k+1)}}) + \bs
                                      \nabla (\bar{\bs x}^{(k)} - \bs x^{(k+1)}) \\
                                      \tinv{\nu^{(k)}} (\bs s_2^{(k)}
                                      - \bs s_2^{{(k+1)}}) + \bs \Phi
                                      (\bar{\bs x}^{(k)} -
                                      \bs x^{(k+1)}) \end{pmatrix}
                                  \right\|_1$ 
                                \vspace{2mm}\sComment{Parameters update}
				\If{$p^{(k+1)} > c d^{(k+1)} \Gamma$}
                                {\quad $\rhd$ \emph{Primal residual is larger than dual}}\vspace{2mm}	 
					\State \vspace{2mm} $\mu^{(k+1)} =
                                        \mu^{(k)}(1-\rho^{(k)})$;\quad
                                        $\nu^{(k+1)} =
                                        \nu^{(k)}/(1-\rho^{(k)})$;\quad
                                        $\rho^{(k+1)} =
                                        \rho^{(k)}\beta$ 
				\ElsIf{$p^{(k+1)} < c d^{(k+1)}/\Gamma$}{\quad $\rhd$ \emph{Dual residual is larger than primal}}\vspace{2mm}
					\State \vspace{2mm} $\mu^{(k+1)} =
                                        \mu^{(k)}/(1-\rho^{(k)})$;\quad
                                        $\nu^{(k+1)} =
                                        \nu^{(k)}(1-\rho^{(k)})$;\quad $\rho^{(k+1)} =
                                        \rho^{(k)}\beta$ 
				\Else {\quad $\rhd$ \emph{Similar primal and dual
                                  residuals, \ie $p^{(k+1)}\in
                                  \big[c d^{(k+1)}/\Gamma,\Gamma c d^{(k+1)}\big]$}}\vspace{2mm}
					\State \vspace{2mm}
                                        $\mu^{(k+1)} = \mu^{(k)}$;
                                        \quad $\nu^{(k+1)} =
                                        \nu^{(k)}$; \quad $\rho^{(k+1)} = \rho^{(k)}$
				\EndIf
                                \sComment{Stop if stable behavior}
				\If{$\| \bs x^{(k+1)} - \bs x^{(k)} \|
                                  / \| \bs x^{(k)} \| \leq
                                  \mathtt{Th}$\ }\quad  \rm break.
				\EndIf
			\EndFor
  \end{algorithmic}
\end{algorithm}

\section{EXPERIMENTS} 
\label{sec:experiments}

In this section, the ODT reconstruction is first compared to the
common tomographic (AT) reconstruction using the FBP and ME methods. Then, the
proposed regularized reconstruction (TV-$\ell_2$) is compared with the
FBP and ME procedures on synthetic and experimental ODT data.

\subsection{Synthetic Data} 
\label{sec:synth_data}

Three kinds of discrete synthetic 2-D RIM are selected to test the
reconstruction. They are defined on a $256\!\times\!256$ pixel grid
($N=256^2$). In the first object, the RIM ($\imap$) simulates a 2-D
section of a bundle of 10~fibers of radius 8 pixels each, immersed in
an optical fluid (the background). The two media have a refractive
index difference of $\delta \imap = 12.1 \times 10^{-3}$ (see
Fig.~\ref{fig:RIM_fibers_sphere}-(left)). The second object consists in
a homogeneous ball centered in the pixel $(154,154)$ with a radius of
60 pixels, immersed in a liquid with $\delta \imap = 2.8 \times
10^{-3}$ (see Fig.~\ref{fig:RIM_fibers_sphere}-(right)). These two
objects were selected in correspondence to the available experimental
data we use for reconstruction later in this section. The third object
is the well-known Shepp-Logan phantom (see
Fig.~\ref{fig:RIM_SL}-(right)), which is a more complex image in a
``Cartoon-shape'' model.

\begin{figure}
  \centering
	\includegraphics[height=1.4in]{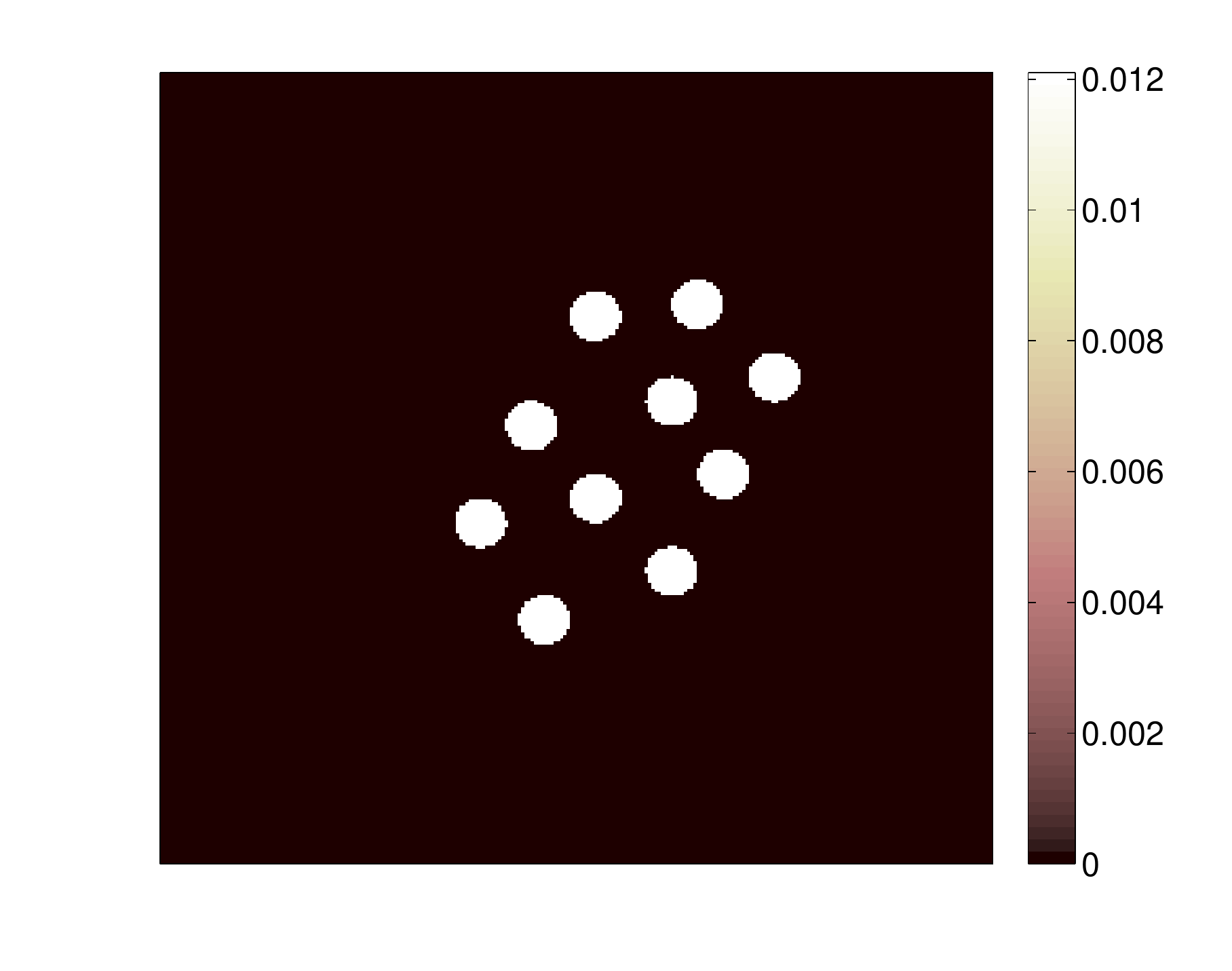}
	\hspace{0.5cm}
	\includegraphics[height=1.4in]{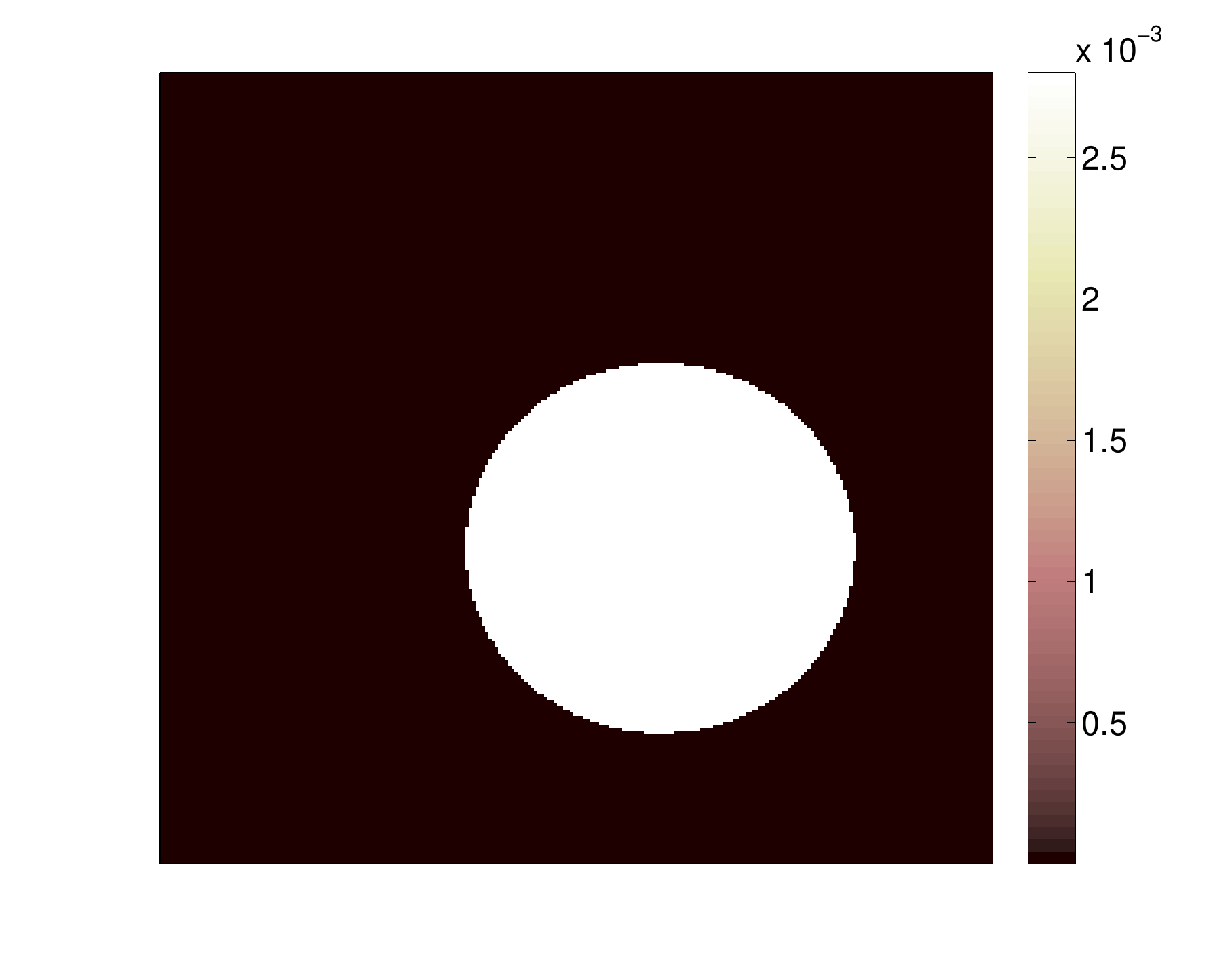}
  \caption{ Realistic refractive index map: (left) Synthetic fibers bundle and (right) Synthetic ball.}
 \label{fig:RIM_fibers_sphere}
\end{figure}

The measurements were simulated according to
\eqref{eq:disc-measures-noise} by means of the operator $\bs \Phi$,
and then, additive white Gaussian noise $\bs \eta_{\rm obs} \sim_{\rm
  iid} \cl N(0,\sigma_{\rm obs}^2)$ is added in order to
simulate a realistic ODT scenario.

The operator $\bs \Phi$ is defined as $\bs \Phi_{\epsilon} = \bs \Theta \bs D \bs
F_{\epsilon}$, with $\epsilon$ representing the distorsion of
$\bs F_{\epsilon}$ regarding a true operator $\bs F_{\rm true}$ that would
provide the actual NDFT. As discussed in Sec.~\ref{sec:forward-model},
the NDFT computational time is inversely proportional to this
parameter $\epsilon$ in $\cl O(N \log (N/\epsilon))$. Therefore, we
need to do a compromise between an accurate and an efficient
computation of the NDFT. For this reason we use two different
operators: (i) an accurate and high dimensional operator $\bs
\Phi_{\epsilon_0} = \bs \Theta \bs D \bs F_{\epsilon_0}$ for the acquisition,
with a small $\epsilon_0 = 10^{-14}$; and (ii) a less accurate but
lower dimensional operator $\bs \Phi_{\epsilon_1} = \bs \Theta \bs D \bs
F_{\epsilon_1}$ for the reconstruction, with $\epsilon_1 >
\epsilon_0$. The error caused for using a higher $\epsilon$ for the
reconstruction is taken into account in $\varepsilon_{\rm nfft}$ (see
Sec.~\ref{sec:noise_estimation}).

For each object, the ODT measurements are obtained with $N_{\tau} =
367$ according to a varying number of orientations $N_{\theta}$, which
allows to analyze the compressiveness of the reconstruction method. In
this synthetic experiment, the orientations $\theta$ are taken in $[0,
\pi)$ so that $N_\theta = 360$ corresponds to two orientations per
degree. Hereafter, we consider this last situation, \ie
$\delta\theta=\pi/360$ as a ``full observation'' scenario since, given
the considered RIM resolutions, the discrete frequency plane is almost
fully covered in this case. More generally, we say that a given
orientation number $N_\theta$ is associated to
$(100\,\tfrac{N_\theta}{N_{\rm full}}) \%$ of the full coverage,
$N_{\rm full}$ being the number of orientations for having
$\delta\theta = \pi/360$.

The reconstruction robustness with respect to the noise level has been
considered for a Measurement SNR 
${\rm MSNR} = 20 \log_{10} \| \bs \Delta \| /
\| \bs \eta \|$ taken in $\{10\,{\rm dB}, 20\,{\rm
  dB},\,\infty\}$. This last case with MSNR close to $+\infty$
corresponds to the noiseless scenario, where no Gaussian noise is
added, only the NFFT interpolation error ($\bs \eta_{\textrm{nfft}}$)
is taken into account. This actually provides a high MSNR value around
$270\,{\rm dB}$.

The reconstruction quality of $\tilde \imapv \in\{\imapv_{\rm FBP}, \imapv_{\rm ME},  \imapv_{{\rm TV-}\ell_2}\}$
 is measured using the Reconstruction SNR
(RSNR) measured by $ {\rm RSNR}\ =\ 20\log_{10} \|\bs \imap\|/\|\bs
\imap - \tilde{\imapv}\|.$

\subsubsection{Robustness comparison for AT and ODT}
\label{sec:tomo_vs_deflect}

In order to assess numerically the impact of operator $\bs D$ in ODT,
we compare the RSNR between AT and ODT in similar noisy acquisition
scenarios. The comparison is made using the FBP and ME procedures, commonly
applied in tomographic reconstructions. In ODT the mean of the image cannot be estimated but this is not taken into account by FBP and ME procedures, thus the mean of the reconstructed image was removed for the computation of the RSNR in order to correctly compare the two reconstructions. We analyze the impact of the
affine frequency $\omega$, present in ODT, via the compressiveness and
noise robustness. For this, we focus on the reconstruction of the
bundle of fibers for different number of orientations $N_{\theta} \in
\{ 4, 18, 90, 180, 360 \}$. The results are depicted in
Fig.~\ref{fig:TomoDeflect} and Fig.~\ref{fig:TomoDeflect2}.

\begin{figure}
  \centering
	\includegraphics[height=1.8in]{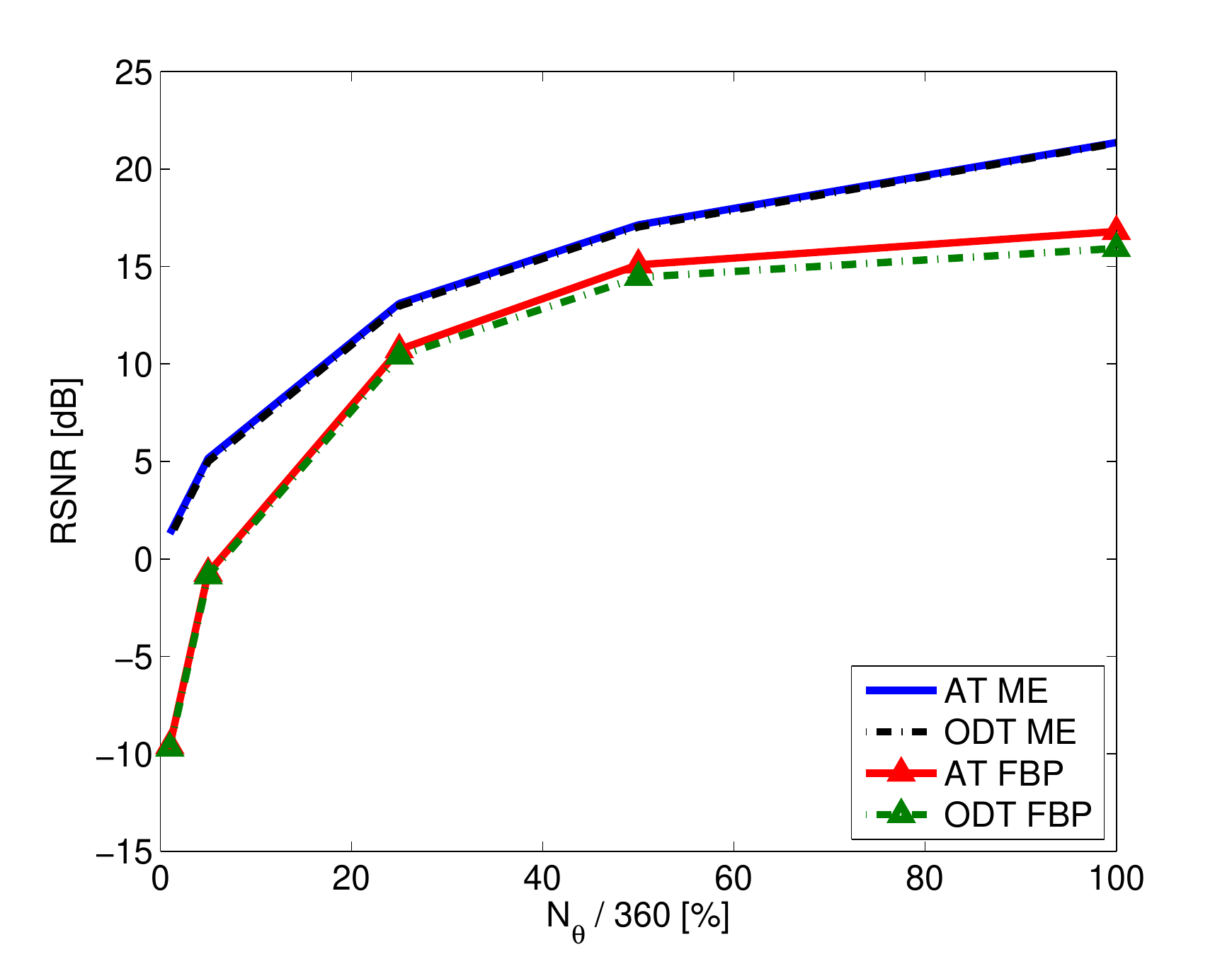}
  \caption{ Absorption Tomography vs Deflectometric Tomography for different number of orientations $N_{\theta}$ with MSNR = $\infty$. Using FBP and ME procedures.} 
 \label{fig:TomoDeflect}
\end{figure}

\begin{figure}
  \centering
	\includegraphics[height=1.8in]{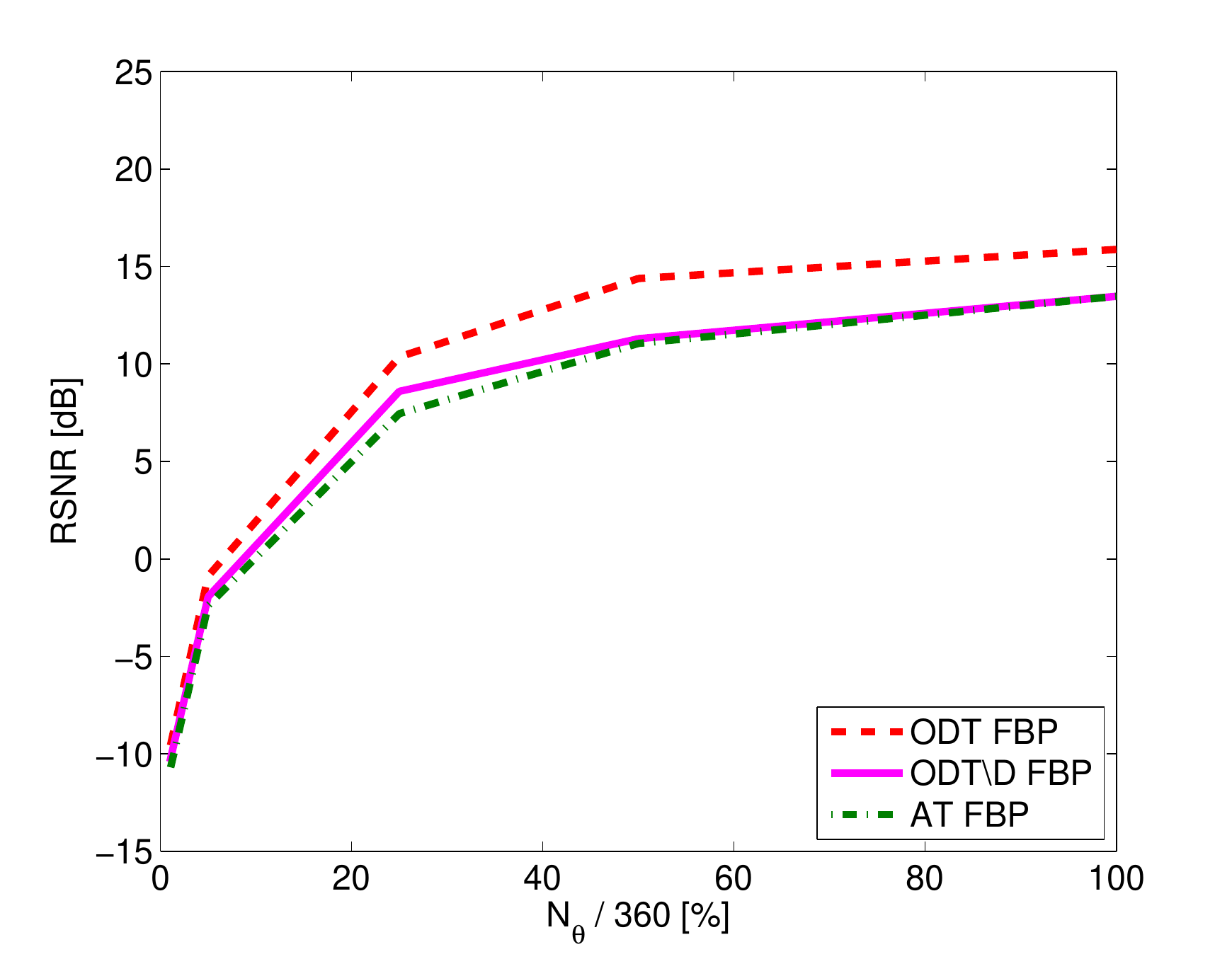}
	\includegraphics[height=1.8in]{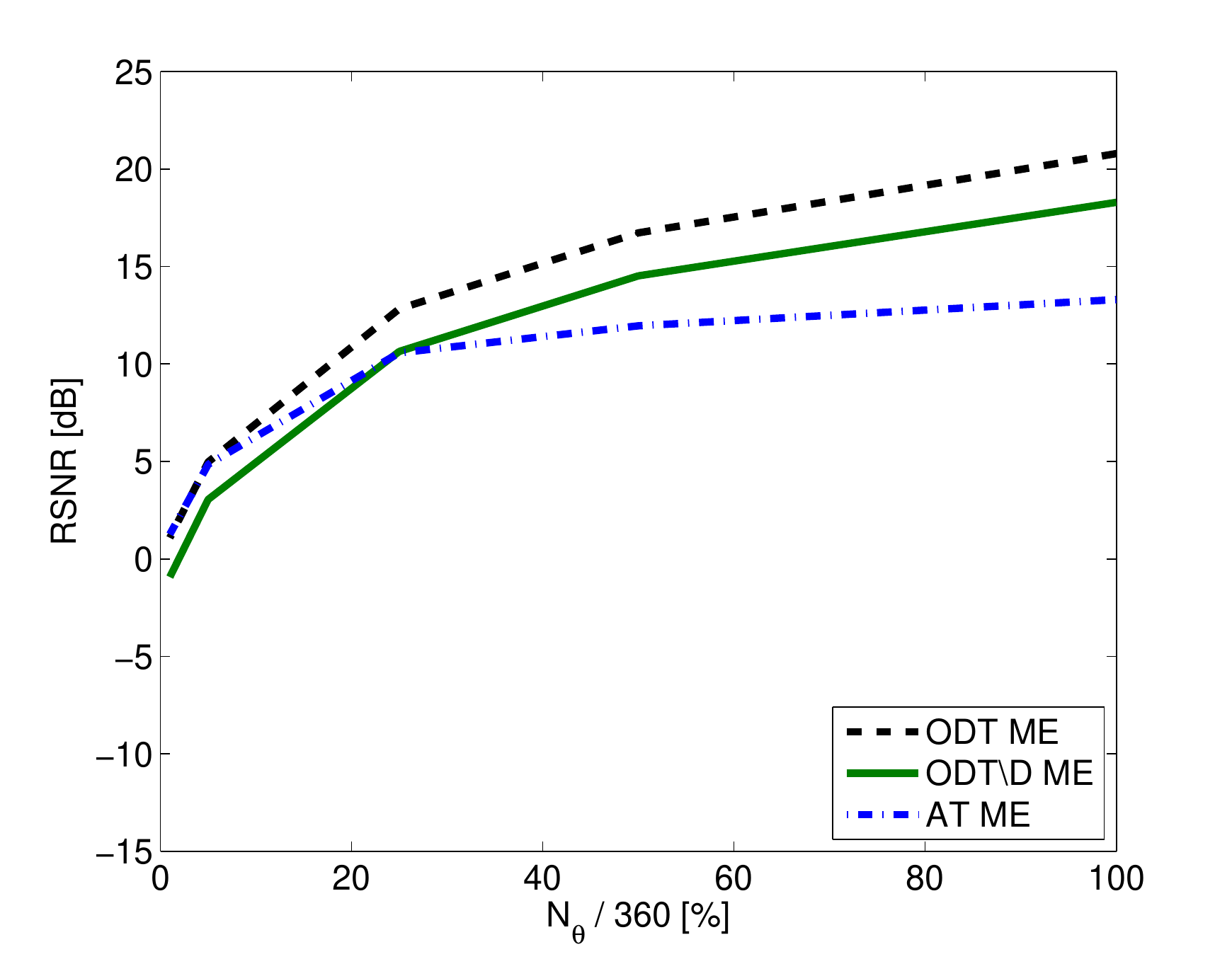}
  \caption{ Absorption Tomography vs Deflectometric Tomography for different number of orientations $N_{\theta}$ with MSNR = 20\,{\rm dB} using (left) the FBP algorithm  and (right) the ME procedure.}
 \label{fig:TomoDeflect2}
\end{figure}

In Fig.~\ref{fig:TomoDeflect} we can see that, when no Gaussian
noise is added, we obtain similar RSNR for both AT and ODT. 
The same behavior is observed for both FBP and ME reconstructions, with FBP always providing a lower RSNR. The impact
of the parameter $\omega$ is evident only in the convergence time of ME,
causing the ODT reconstruction to be $4$ times slower than the AT
one. However, when we add Gaussian noise in such a way that both data
have a MSNR $= 20\,{\rm dB}$, the AT reconstruction presents a fast
degradation while the ODT reconstruction remains almost unaffected by
the noise (see Fig.~\ref{fig:TomoDeflect2}). These results, observed for both FBP and ME,
corroborate the discussion in Sec.~\ref{sec:noise_tomo_vs_deflect}.

Following the discussion from Sec.~\ref{sec:forward-model}, we analyze
now the reconstruction of the RIM using a simplified ODT sensing model
that is close to a classical tomographic model
\eqref{eq:tomo_transform}. In Fig.~\ref{fig:TomoDeflect2} we
show a third curve that corresponds to the RIM reconstruction from a
noisy ODT sensing where the Fourier measurements are divided by the diagonal operator $\bs D$ 
(ODT$\backslash \bs D$) as in \eqref{eq:tomo_transform}. The results
were obtained using the FBP and ME procedures and for a MSNR $= 20\,{\rm
  dB}$. As it was expected, when dividing the measurements by the operator $\bs D$, the reconstruction quality decreases
significantly compared to the results obtained with the complete ODT
sensing model \eqref{eq:disc-measures-noise}. Moreover, the
regularized formulation TV-$\ell_2$ cannot be used for this ODT
reconstruction because the noise is then heteroscedastic.

\subsubsection{TV-$\ell_2$ Reconstruction method}
\label{sec:tv-l2}

The TV-$\ell_2$ reconstruction is compared with the common FBP and ME
methods. The reconstruction quality is investigated with respect to
compressiveness and noise
robustness. On the contrary of FBP and ME, the TV-$\ell_2$ method takes into account the zero mean of the image by the \emph{frontier constraint}. Therefore, the mean of the reconstructed image is only removed from the ME and FBP results. Fig.~\ref{fig:FBP-TVL2_n0n1n2_fibers} presents comparison
graphs of FBP, ME and TV-$\ell_2$ showing the RSNR vs the number of
orientations $N_{\theta} \in \{ 4, 18, 90, 180, 360 \}$
for the three noise scenarios. These results correspond to the
reconstruction of the bundle of fibers for $\rm \mathtt{Th} =
10^{-5}$.

\begin{figure}
  \centering
	\includegraphics[height=1.8in]{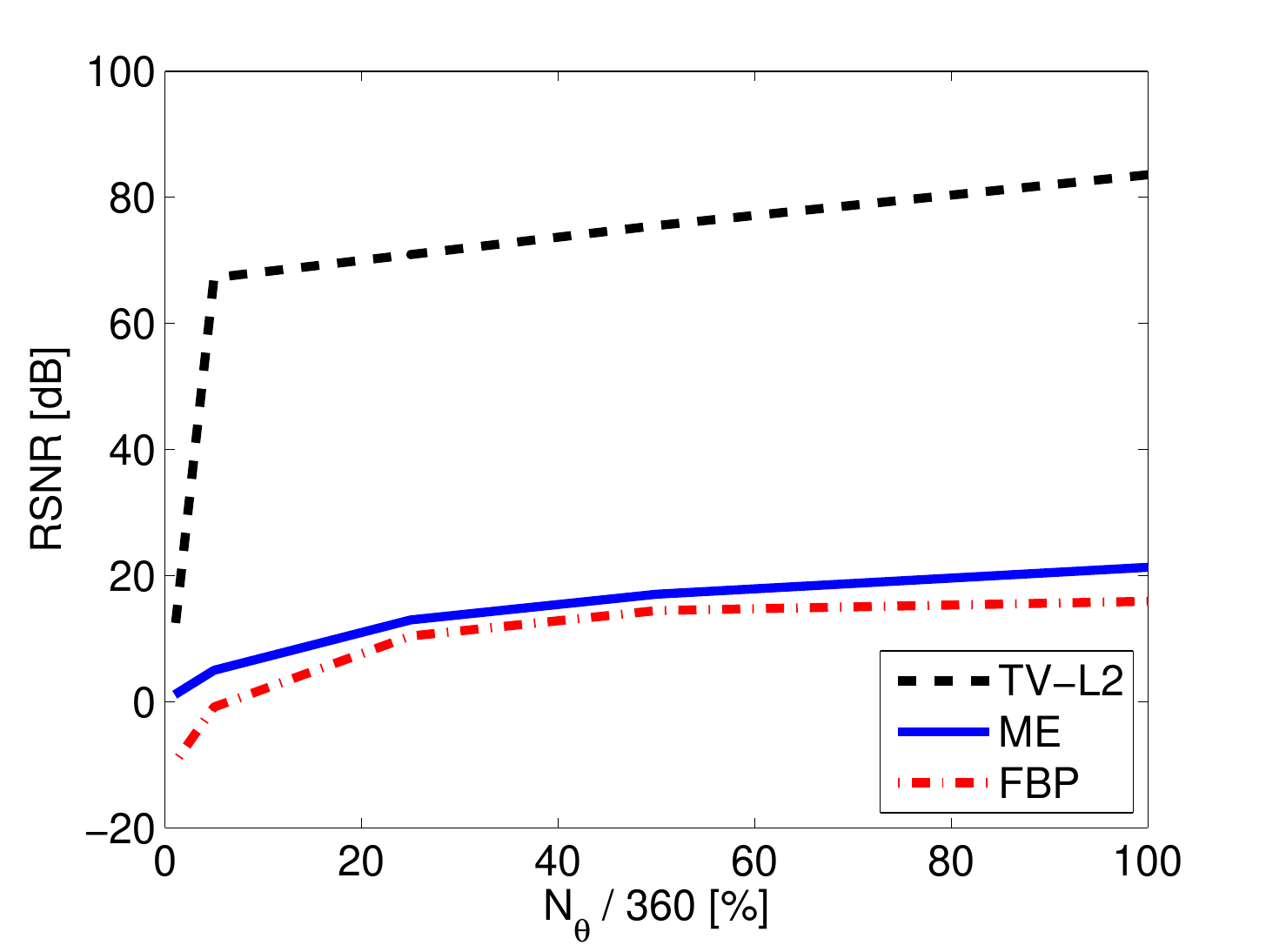}
	\hspace{0.5cm}
	\includegraphics[height=1.8in]{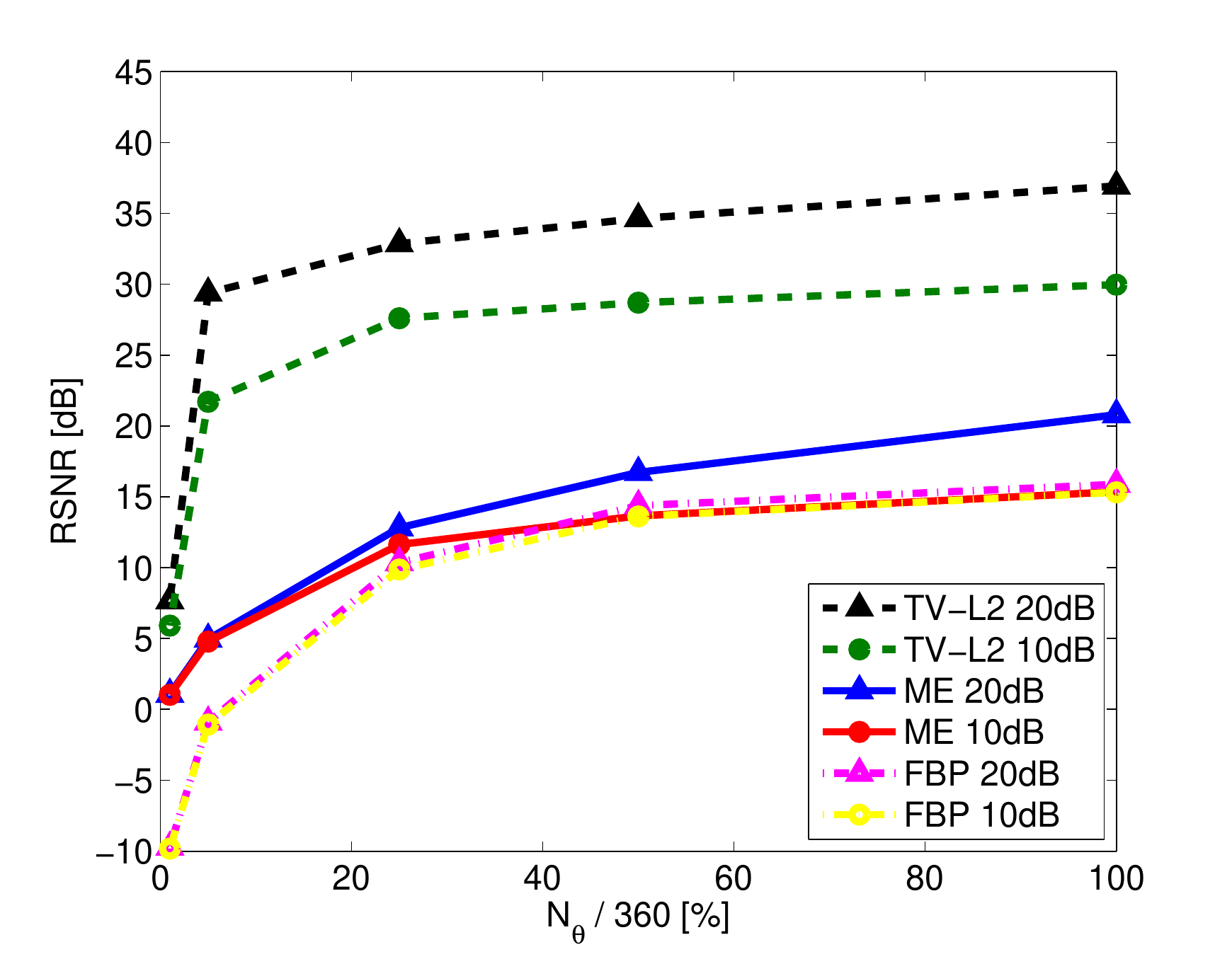}
  \caption{FBP, ME and TV-$\ell_2$ for different number of orientations $N_{\theta}$ with (left) MSNR = $\infty$ and (right) MSNR = 20\,{\rm dB} and MSNR = 10\,{\rm dB}.}
 \label{fig:FBP-TVL2_n0n1n2_fibers}
\end{figure}

In Fig.~\ref{fig:FBP-TVL2_n0n1n2_fibers}-(left) we present the
scenario without added noise, \ie MSNR $= \infty$. We can see that for
a full coverage, \ie $N_\theta = 360$, as the TV-$\ell_2$ method takes
into account the small noise coming from the NFFT interpolation error,
it provides a very good reconstruction that outperforms by 62\,{\rm
  dB} the ME reconstruction quality and by 68\,{\rm
  dB} the FBP reconstruction quality. 

Both FBP and ME methods degrade rapidly when the problem is ill-posed, \ie
when the projections space is not fully covered, whereas the
TV-$\ell_2$ method maintains a high performance. By promoting a small
TV-norm, the regularized method presents high compressiveness, as it
can be observed in the graph where a high reconstruction quality is
still achieved at only 5\% of 360 incident angles, obtaining a gain of
62\,{\rm dB} over ME and of 68\,{\rm dB} over FBP. Although the performance of the algorithm
decreases significantly for a coverage of 1\%, it still provides a
higher reconstruction quality than both ME and FBP.

The high compressiveness properties of the TV-$\ell_2$ method are
preserved when we add Gaussian noise. We are able to obtain good
quality images even for a compressive and highly noisy sensing. With
TV-$\ell_2$, at a MSNR $=10\,{\rm dB}$, we get a RSNR of 22\,{\rm dB}
for a 5\% radial coverage compared to 5\,{\rm dB} for ME and -1\,{\rm dB} for FBP. However, we
can notice how the reconstruction quality of TV-$\ell_2$ diminishes
with respect to the noiseless scenario, whereas FBP and ME are less affected
by the noise.

Fig.~\ref{fig:Reconstructed_fibers_FBPvsTVL2} presents the
resulting images when reconstructing the bundle of fibers in a
noiseless scenario and for $N_{\theta} = \{ 18, 90 \}$, which
represents, respectively, a coverage of $5\%$ and $25\%$ of the
frequency plane. The algorithm is set to stop when $\mathtt{Th} =
10^{-5}$ is reached. 

\begin{figure}
 \centering
  \subfigure[\label{fig:Reconstructed_fibers_FBP_18_n0}][$\imapv_{\rm
    FBP}$.\,$5\%$.\,RSNR\,=\,-\,0.9\,{\rm 
    dB}.]{\includegraphics[height=1.23in]{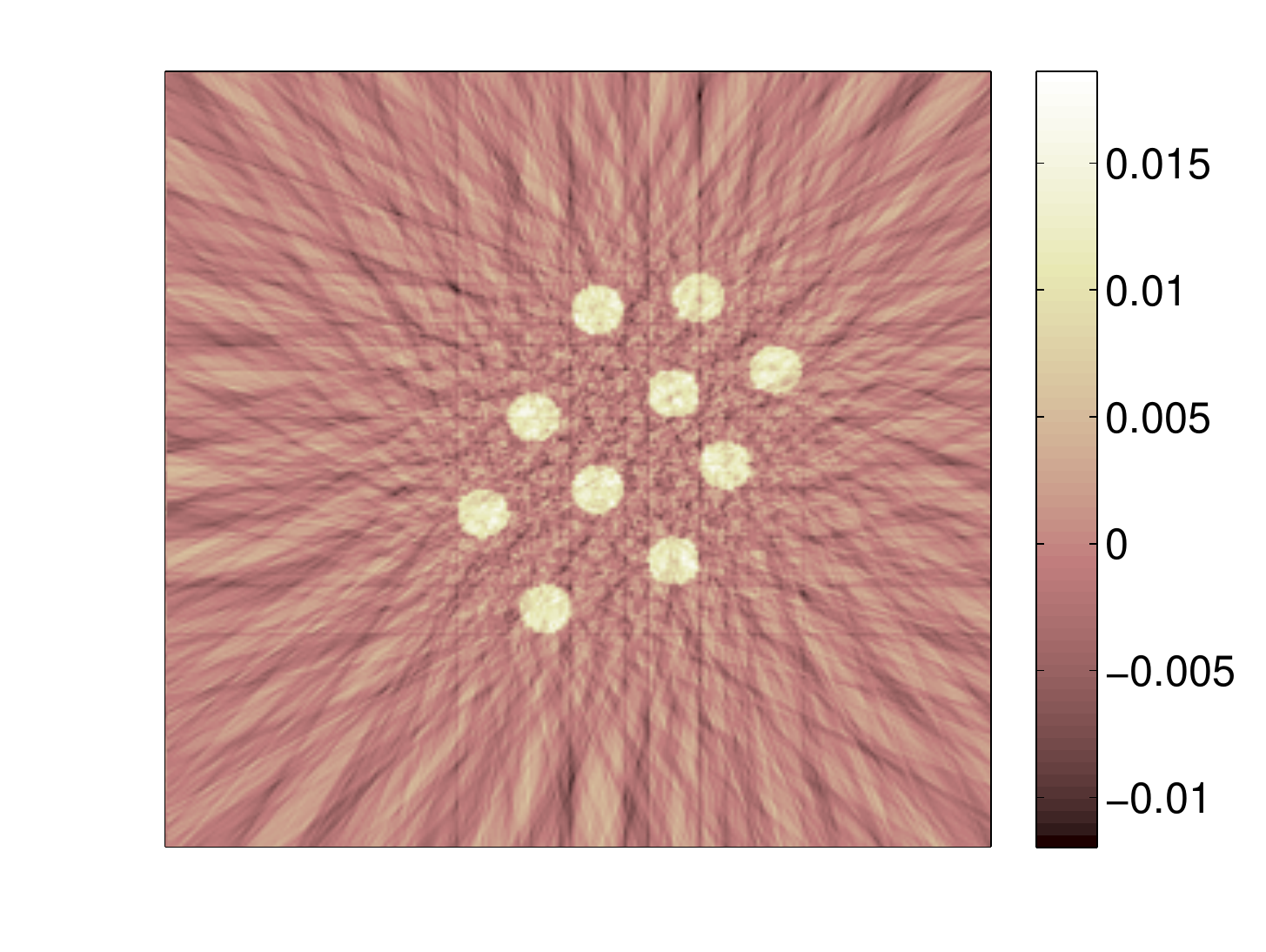}}
  \subfigure[\label{fig:Reconstructed_fibers_ME_18_n0}][$\imapv_{\rm
    ME}$.\,$5\%$.\,RSNR\,=\,5.1\,{\rm
    dB}.]{\includegraphics[height=1.23in]{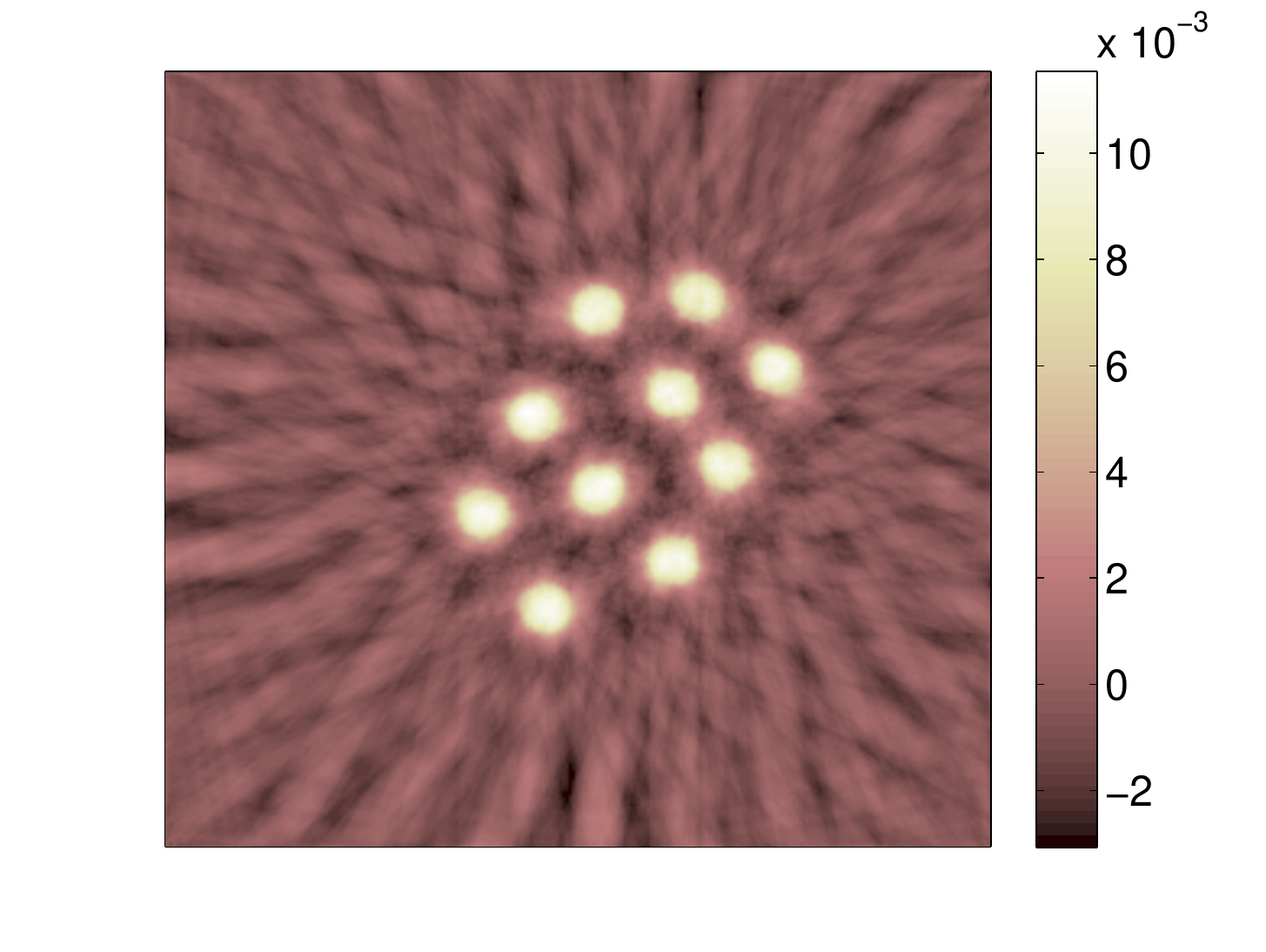}}
  \subfigure[\label{fig:Reconstructed_fibers_TVL2_18_n0}][$\imapv_{{\operatorname{
      TV-}}\ell_2}$.\,$5\%$.\,RSNR\,=\,67\,{\rm
    dB}.]{\includegraphics[height=1.23in]{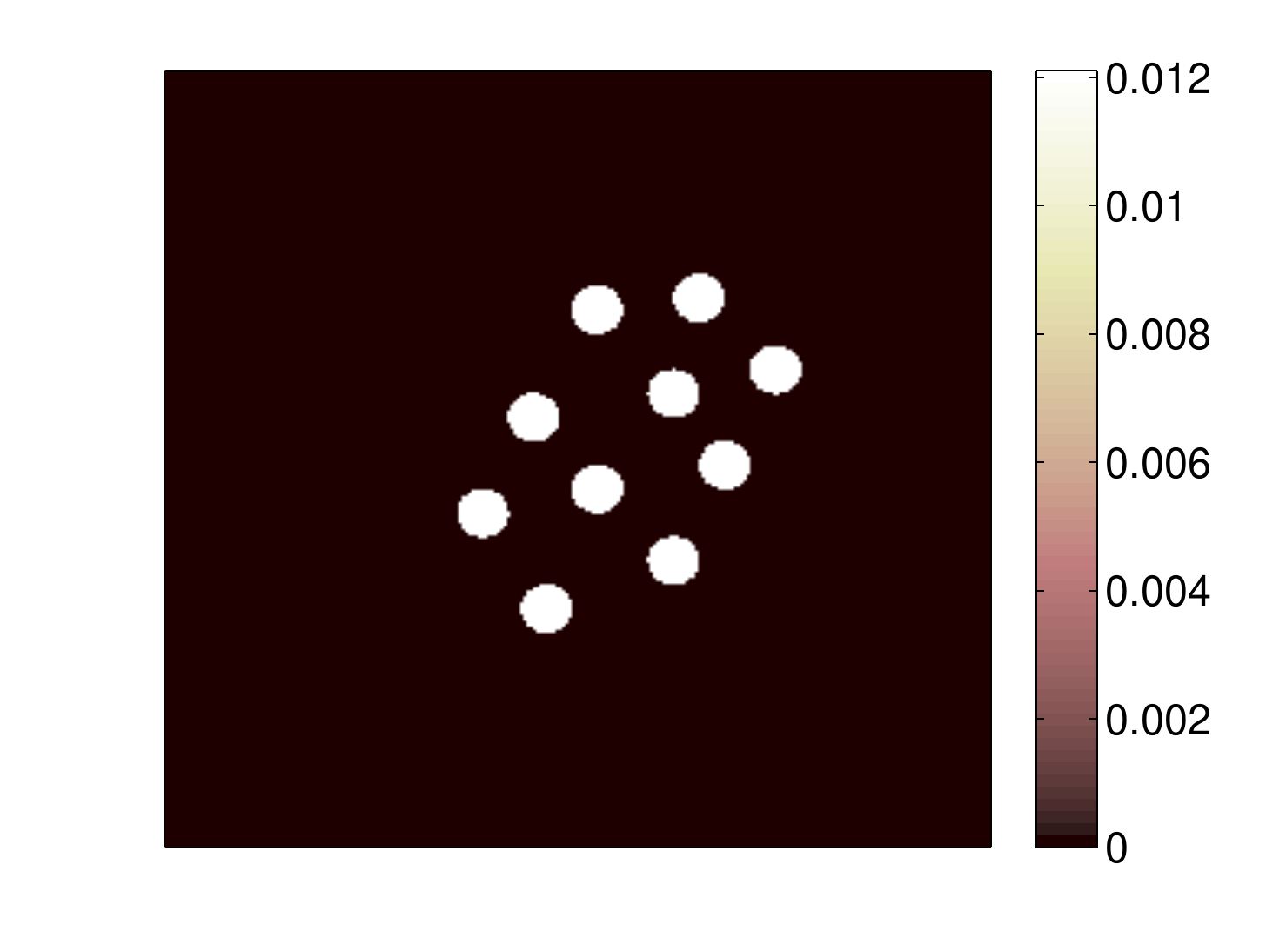}}
  \subfigure[\label{fig:Reconstructed_fibers_FBP_90_n0}][$\imapv_{\rm
    FBP}$.\,$25\%$.\,RSNR\,=\,10\,{\rm
    dB}.]{\includegraphics[height=1.23in]{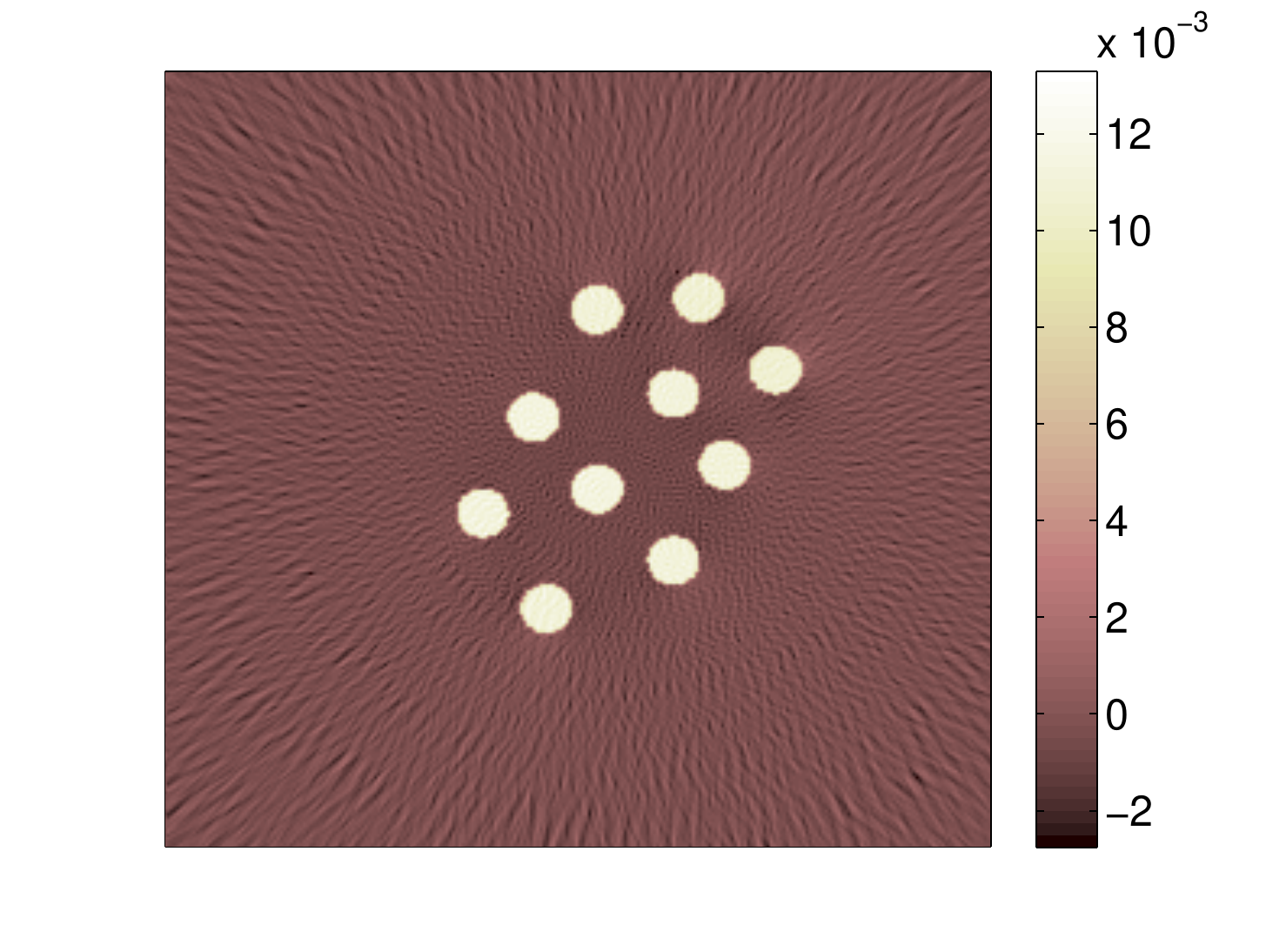}}
  \subfigure[\label{fig:Reconstructed_fibers_ME_90_n0}][$\imapv_{\rm
    ME}$.\,$25\%$.\,RSNR\,=\,13\,{\rm
    dB}.]{\includegraphics[height=1.23in]{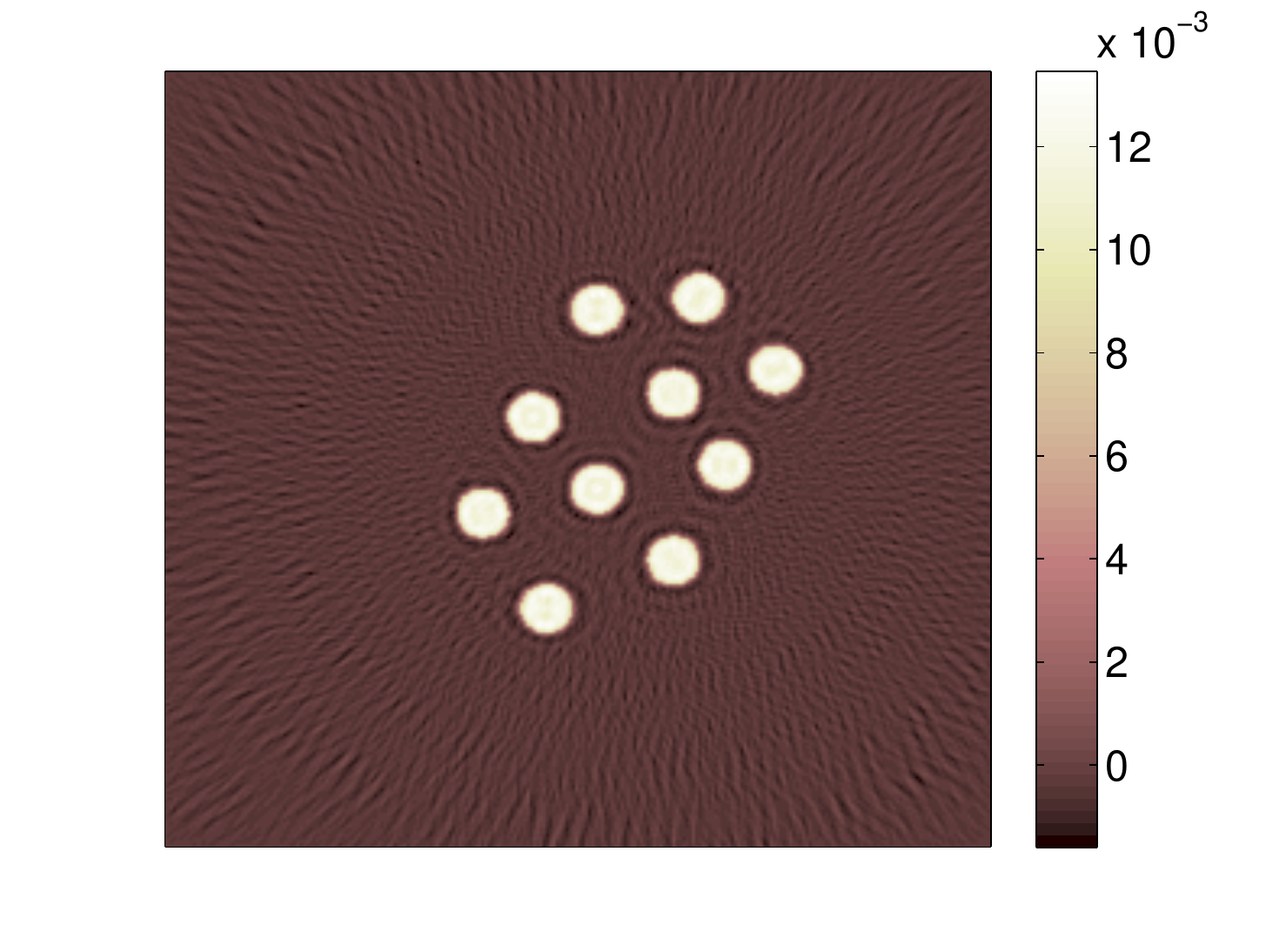}}     
  \subfigure[\label{fig:Reconstructed_fibers_TVL2_90_n0}][\footnotesize $\imapv_{{\operatorname{
      TV-}}\ell_2}$.\,$25\%$.\,RSNR\,=\,71\,{\rm
    dB}.]{\includegraphics[height=1.23in]{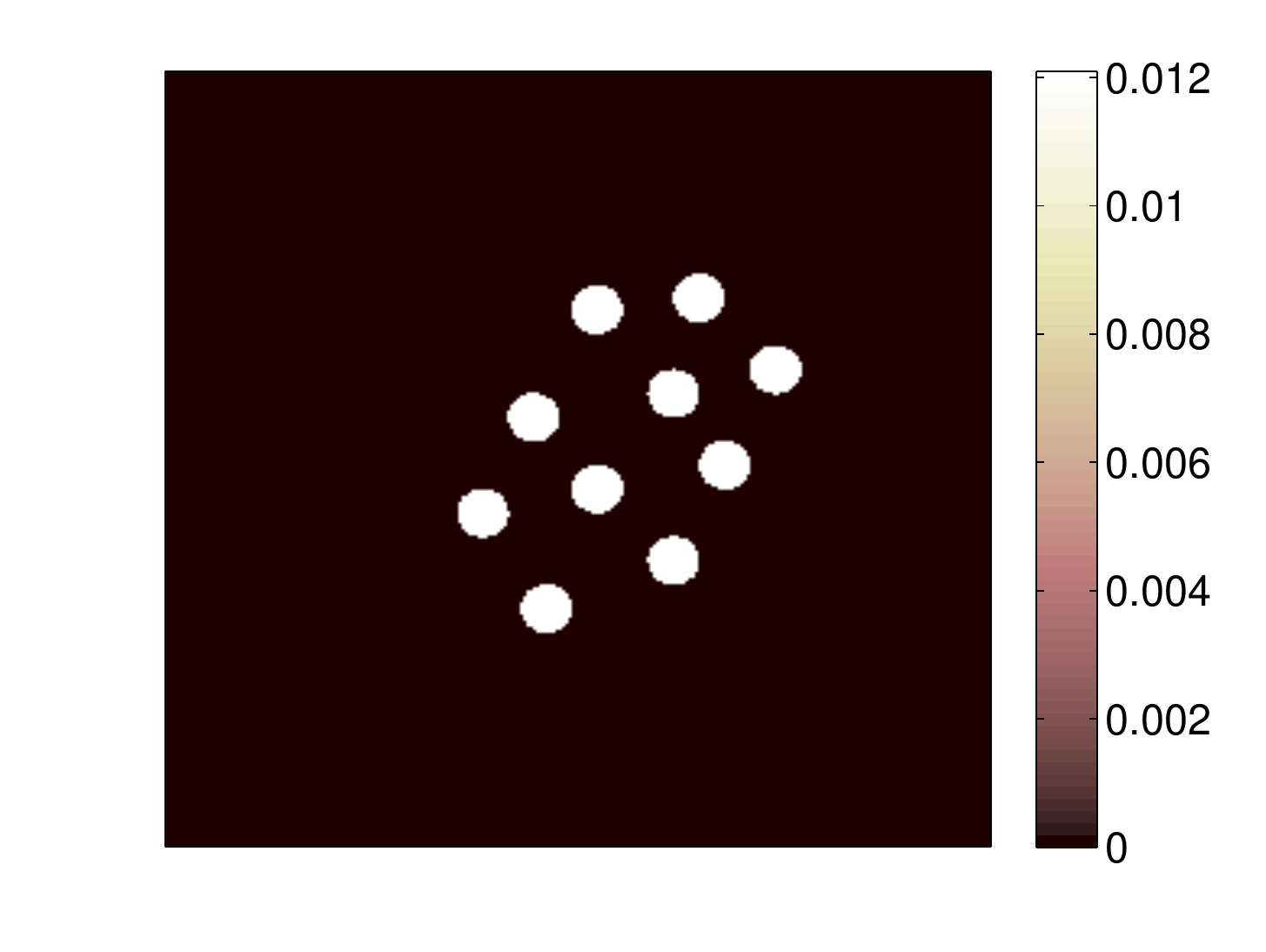}}
  \caption{ Reconstruction images using FBP, ME and TV-$\ell_2$
    reconstruction methods for MSNR $= \infty$ and different number of
    orientations. In the left column we have the FBP reconstruction
    results for (a) $N_{\theta} = 18$ and (d) $N_{\theta} = 90$. In
    the central column we have the ME reconstruction results
    for (b) $N_{\theta} = 18$ and (e) $N_{\theta} = 90$. In
    the right column we have the TV-$\ell_2$ reconstruction results
    for (c) $N_{\theta} = 18$ and (f) $N_{\theta} = 90$.}
 \label{fig:Reconstructed_fibers_FBPvsTVL2}
\end{figure}

We notice how the TV-$\ell_2$ method preserves the image dynamics
even for 5\% of coverage, while FBP and ME provide images with implausible
negative values. Moreover, at low measurement regime, some artifacts appear in both FBP and ME 
results. The image dynamics are less preserved in the FBP reconstruction, where the artifacts also affect the center of the fibers for a $5$\% of coverage. 

About the numerical complexities, the FBP method takes less
  than 1 second for the reconstruction, while both ME and TV-$\ell_2$
  methods require more time. For a coverage of 25\% of the frequency
  plane, to reach the same stopping threshold value of $10^{-5}$, ME requires 6280 iterations (1h05') while TV-$\ell_2$ requires 1900 iterations (28').  For a coverage of 5\% of the frequency plane, the situation reverses and, to reach $\mathtt{Th} = 10^{-5}$, ME requires 3450 iterations (51') while TV-$\ell_2$ requires 6620 iterations (1h49'). However, the reconstruction quality is clearly higher when using our regularized
method. In the case where the quality of the image reconstruction is
sufficiently high, the threshold can be decreased to a less
restrictive value. At the end of this section we perform a
convergence analysis for different threshold values, which allows us
to choose the suitable threshold for a required quality or convergence
time.

Although the FBP method requires less time than ME, we have noticed that the ME method outperforms FBP at low to medium reduction
of the total number of angles. Henceforth, only the ME reconstruction method will be used for comparison with the TV-$\ell_2$ method. In Fig.~\ref{fig:Fibers_Comparison} we present the reconstructed images
for the bundle of fibers for a moderately noisy sensing (MSNR $=
20\,{\rm dB}$). We also show the error images in order to provide a
better appreciation of the difference between both methods.

\begin{figure} 
  \centering
	\subfigure[\label{fig:Fibers_FBP_90_n1}][$\imapv_{\rm ME}$. RSNR = 12.83\,{\rm dB}.]{\includegraphics[height=1.4in]{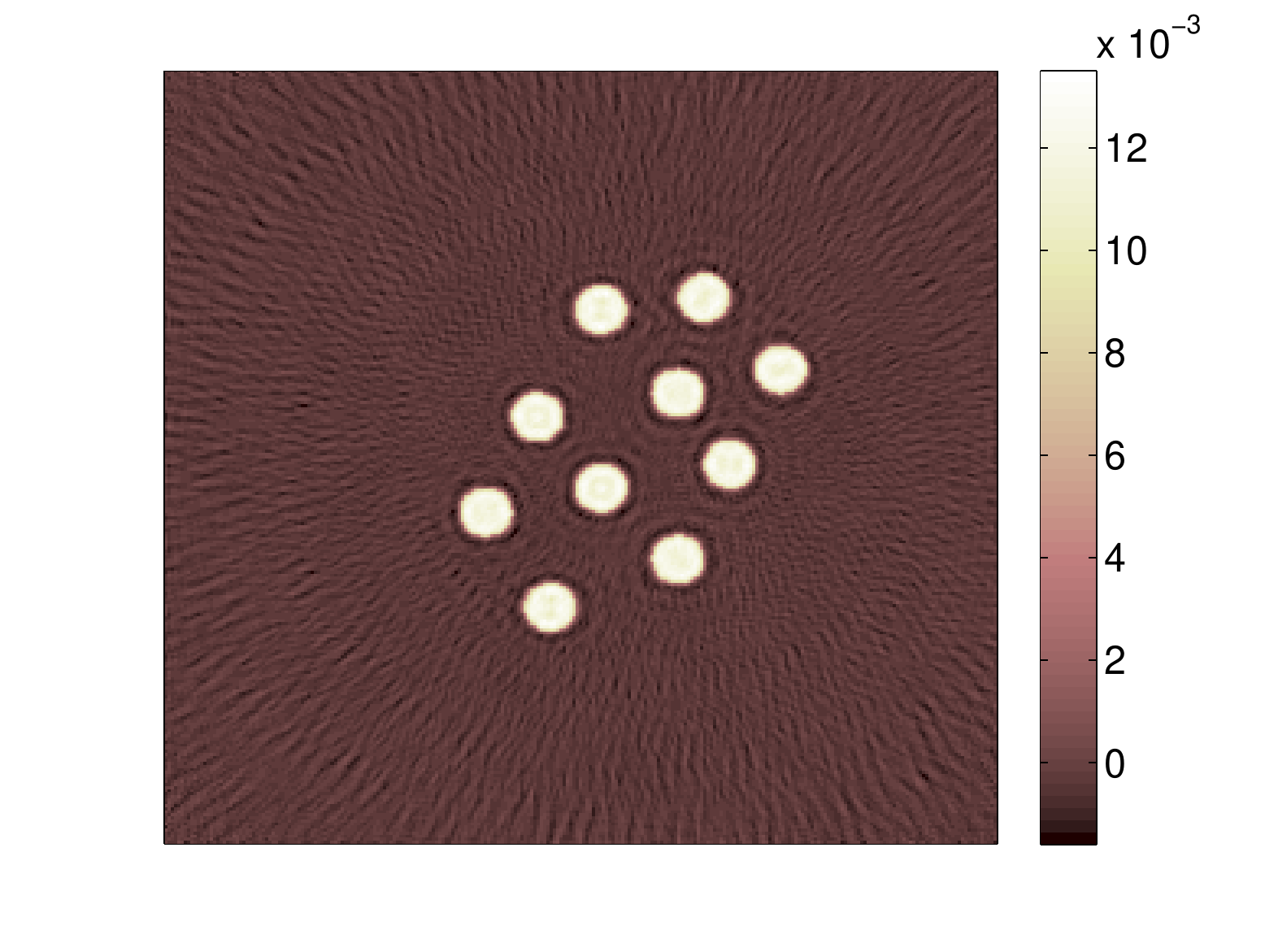}}
	\subfigure[\label{fig:Fibers_TVL2_90_n1}][$\imapv_{{\operatorname{TV-}}\ell_2}$. RSNR = 39.02\,{\rm dB}.]{\includegraphics[height=1.4in]{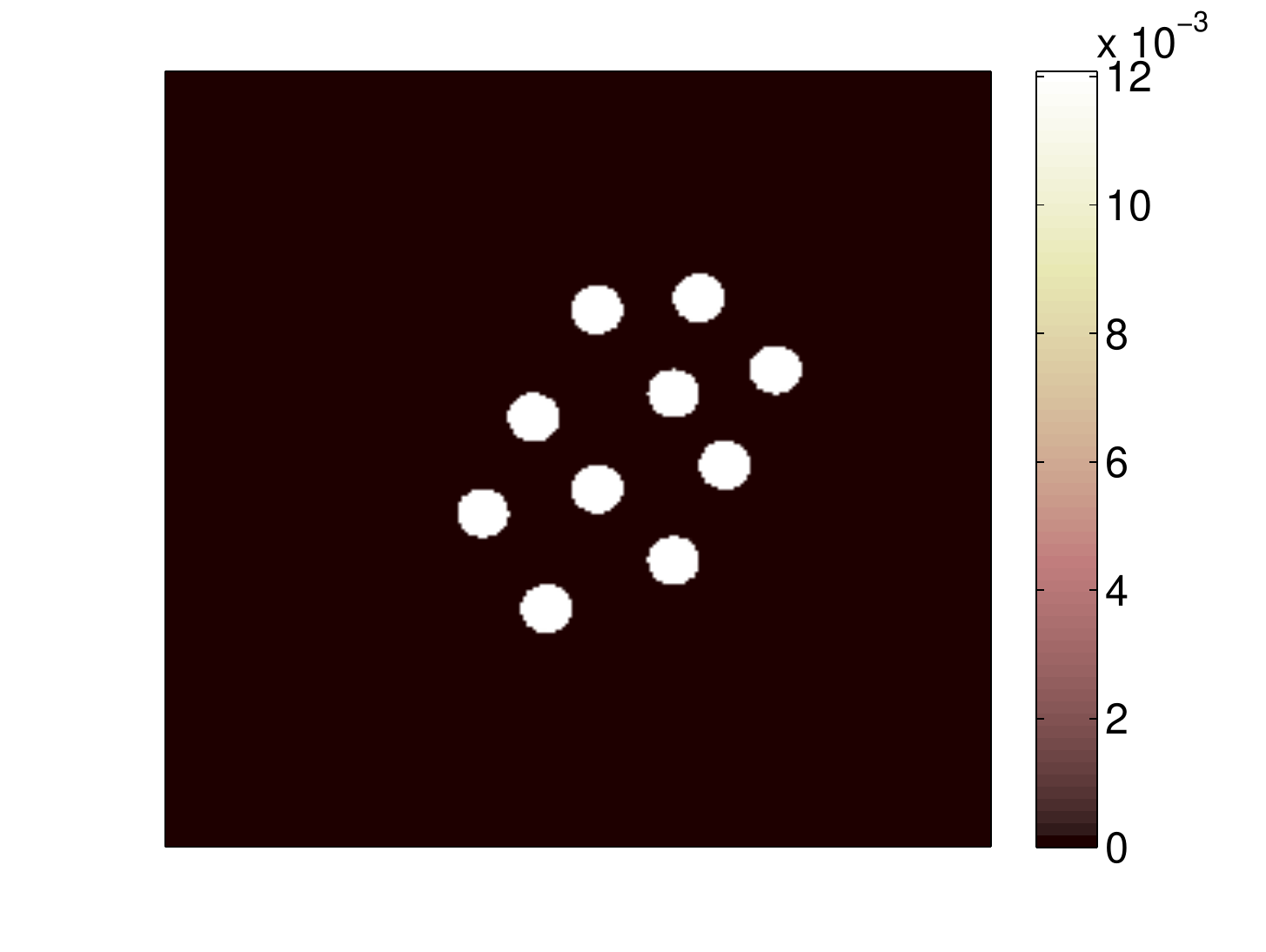}}\\
	\subfigure[\label{fig:Fibers_FBP_90_n1_diff}][$| \imapv - \imapv_{\rm ME} |$.]{\includegraphics[height=1.4in]{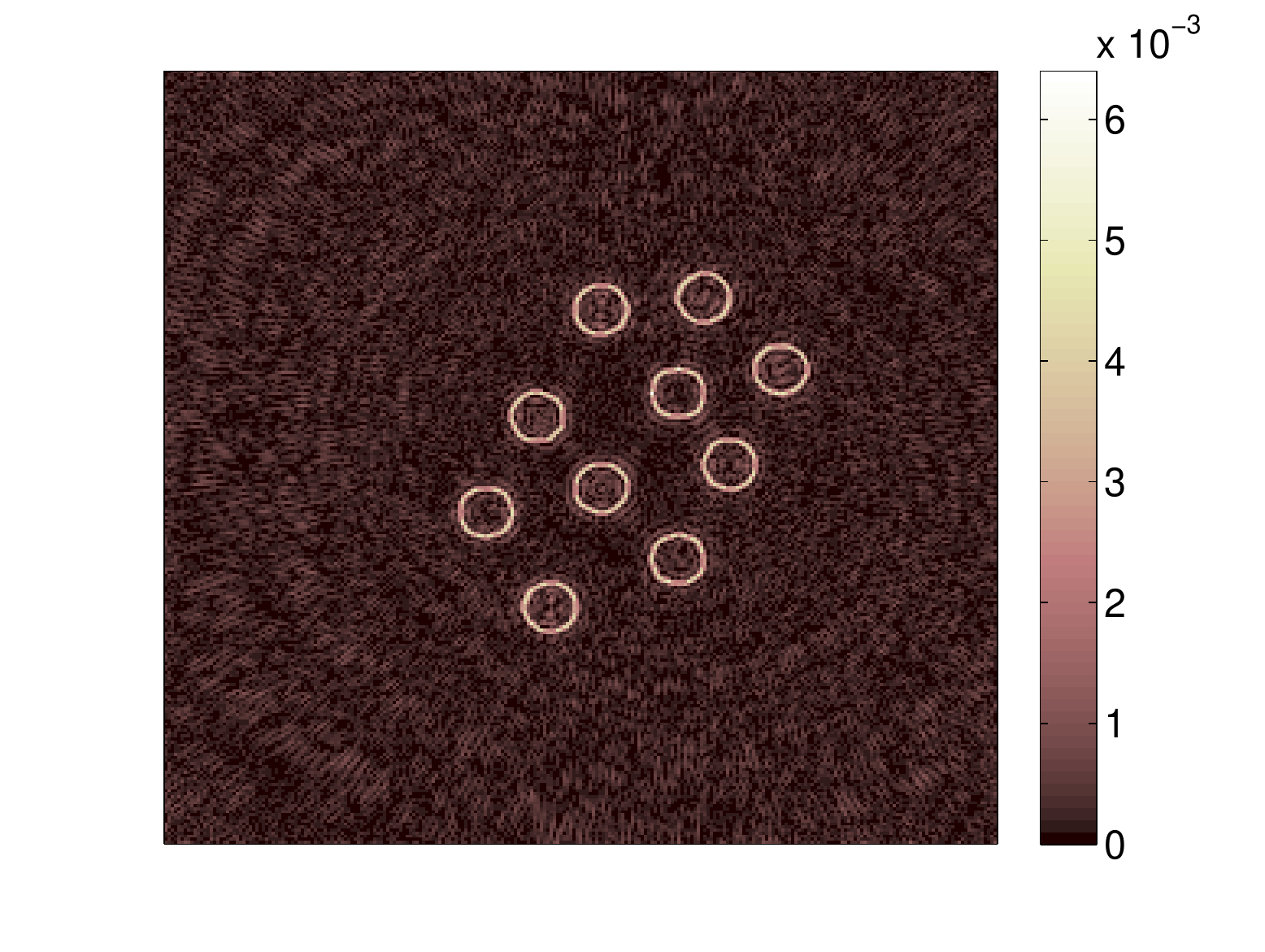}}
	\subfigure[\label{fig:Fibers_TVL2_90_n1_diff}][$| \imapv - \imapv_{{\operatorname{TV-}}\ell_2}|$.]{\includegraphics[height=1.4in]{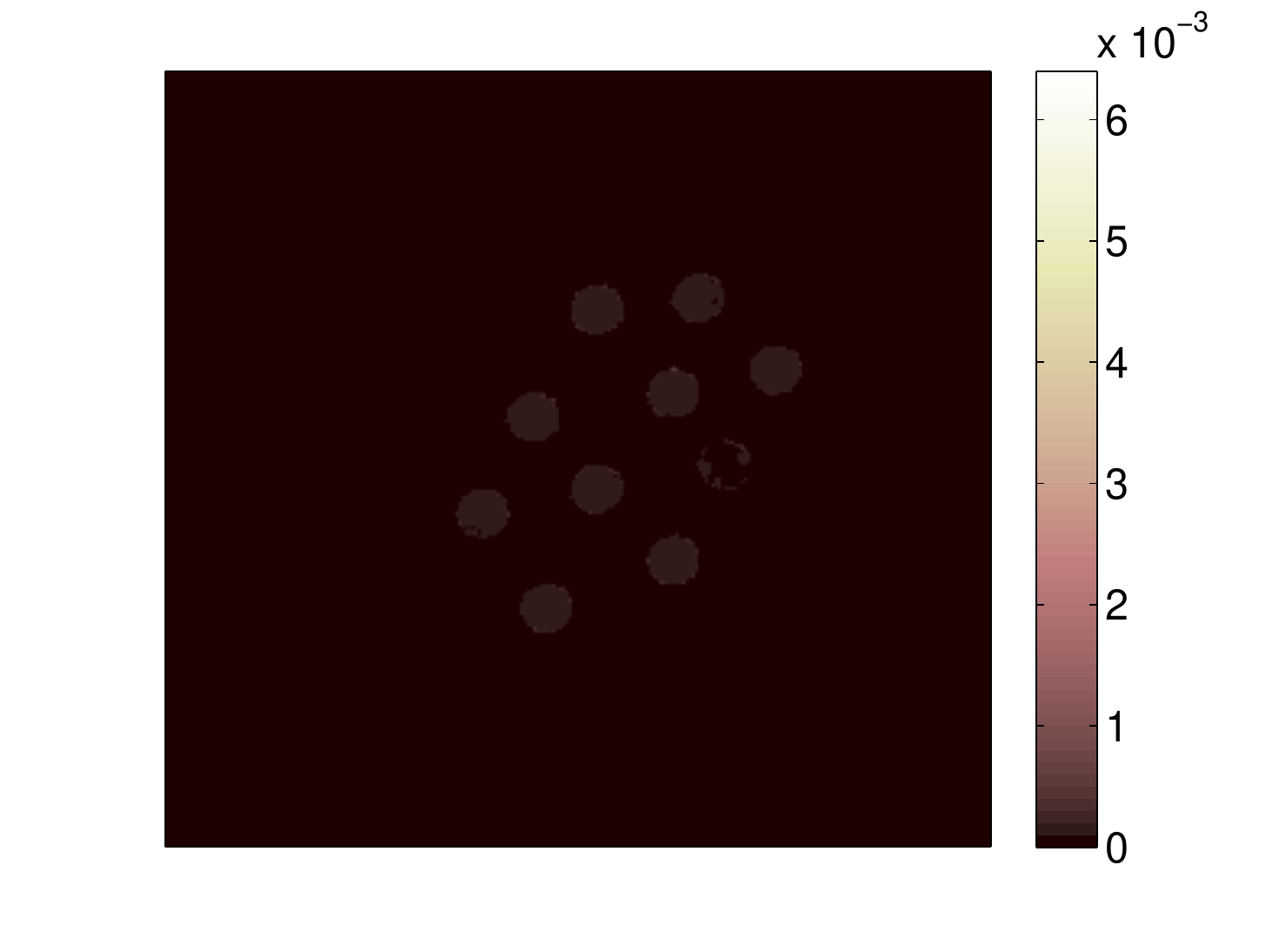}}
  \caption{Fiber reconstruction for MSNR $= 20\,{\rm dB}$ and $N_{\theta} = 90$. Reconstructed image $\tilde{\imapv}$ using (a) ME and (b) TV-$\ell_2$ reconstruction methods. Difference between ground truth $\imapv$ and Reconstructed image $\tilde{\imapv}$ using (c) ME and (d) TV-$\ell_2$ reconstruction methods.}
 \label{fig:Fibers_Comparison}
\end{figure}

In this noisy scenario, the fiber contours are no longer well
estimated using ME and, as we have coverage of only 25\%, some
oscillating (Gibbs) artifacts appear. On the contrary, the regularized
method provides a good estimation on the borders, with no visible
artifacts. We notice certain loss in the dynamics of the image, which
causes a lower RSNR compared to the noiseless scenario.

Let us show now some results obtained for the other two synthetic
images. Fig.~\ref{fig:Sphere_Comparison} and
Fig.~\ref{fig:Shepp_Comparison} present the results obtained using ME
and TV-$\ell_2$ on the sphere and the Shepp-Logan phantoms,
respectively. This comparison was performed for ${\rm MSNR} = \nobreak
20\,{\rm dB}$ and for $N_{\theta} = 90$, \ie a coverage of $25 \% $ of
the frequency plane. As the image expected dynamics change for
  the Shepp-Logan phantom, the parameter $c$ for the residuals
  balancing was set to 250 instead of 1000.

\begin{figure}
  \centering
	\subfigure[\label{fig:Sphere_FBP_90_n1}][$\imapv_{\rm ME}$. RSNR = 21.23\,{\rm dB}.]{\includegraphics[height=1.4in]{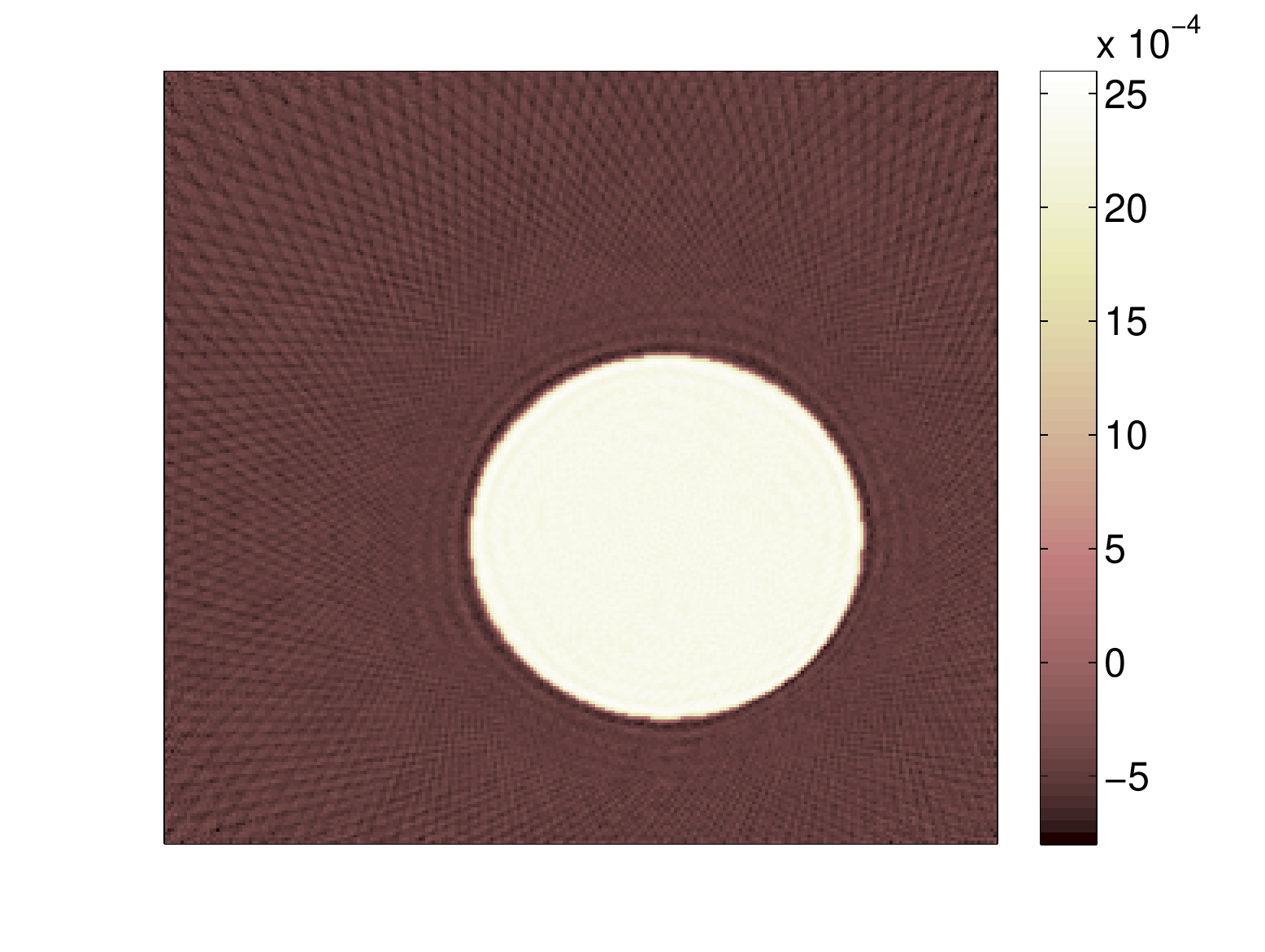}}
	\subfigure[\label{fig:Sphere_TVL2_90_n1}][$\imapv_{{\operatorname{TV-}}\ell_2}$. RSNR = 45.58\,{\rm dB}.]{\includegraphics[height=1.4in]{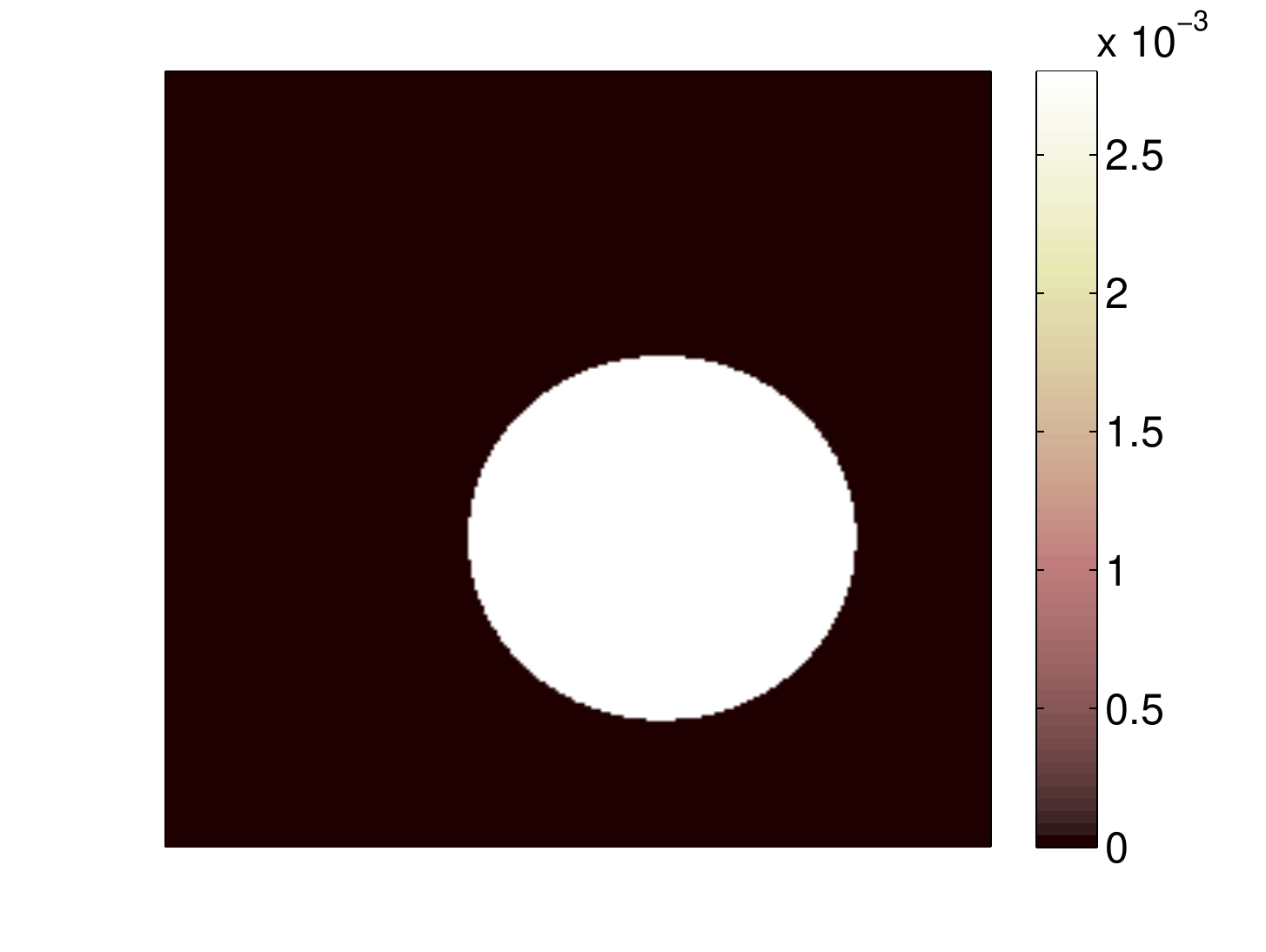}}\\
	\subfigure[\label{fig:Sphere_FBP_90_n1_diff}][$| \imapv - \imapv_{\rm ME} |$.]{\includegraphics[height=1.4in]{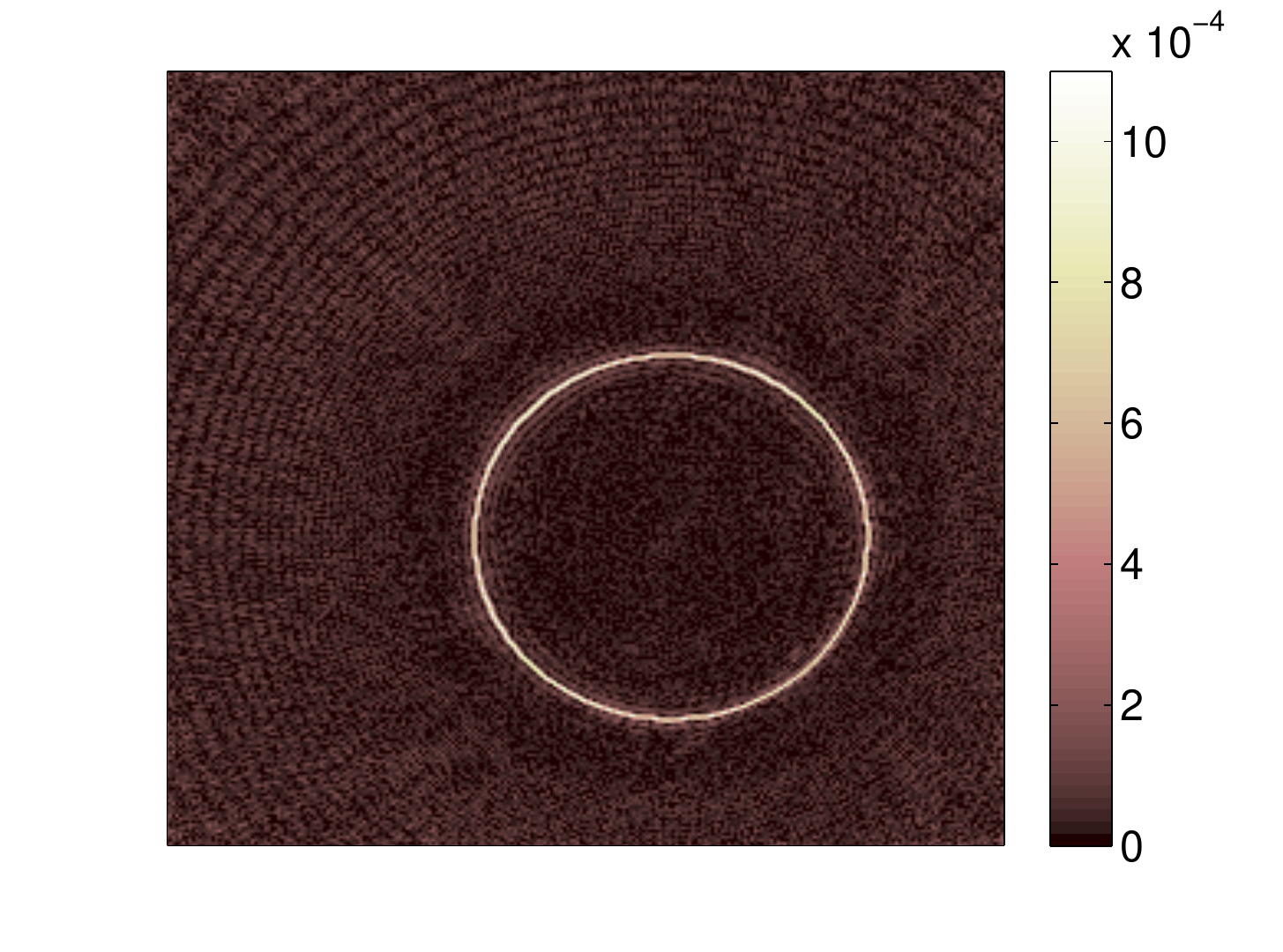}}
	\subfigure[\label{fig:Sphere_TVL2_90_n1_diff}][$| \imapv - \imapv_{{\operatorname{TV-}}\ell_2} |$.]{\includegraphics[height=1.4in]{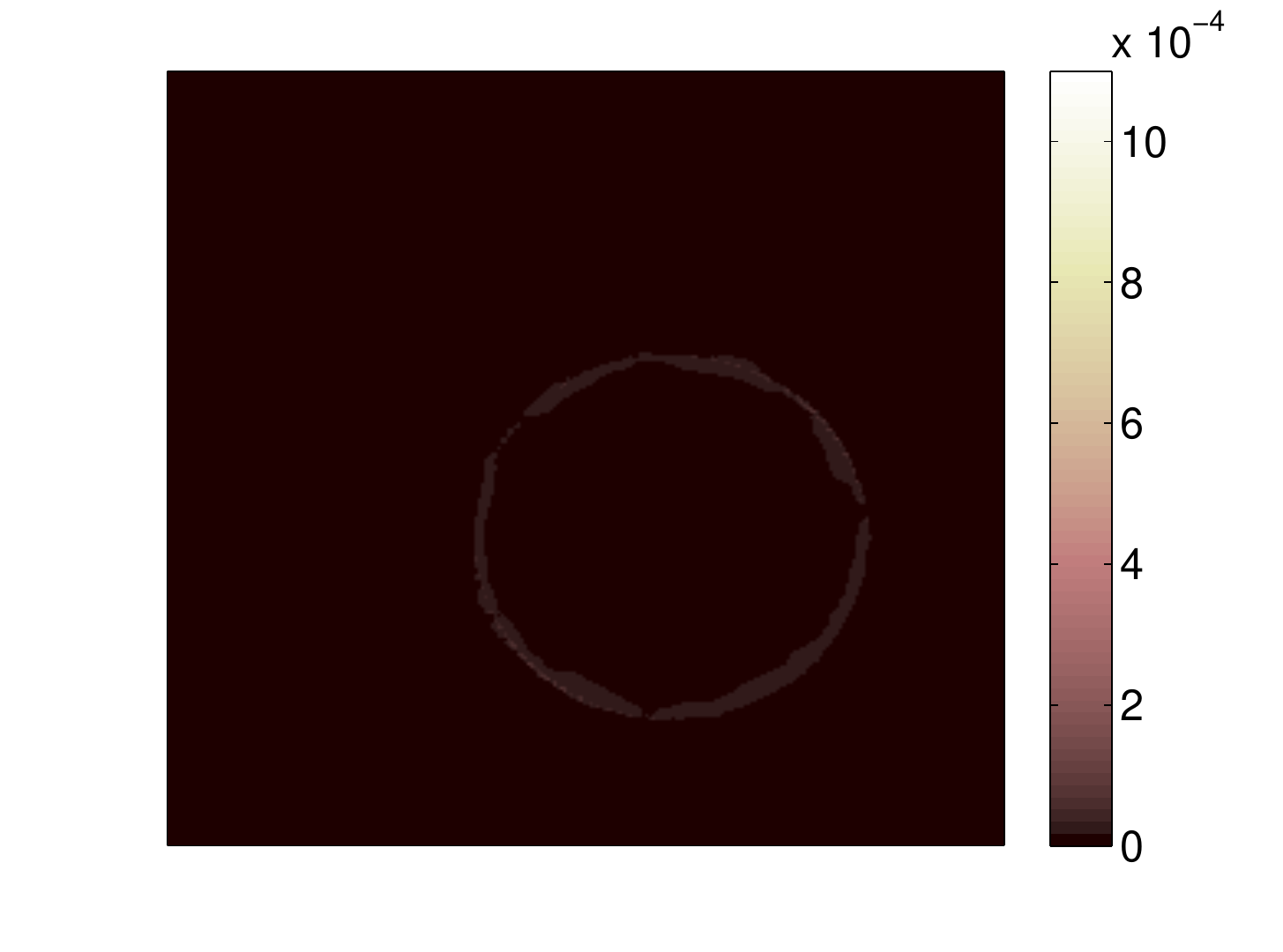}}
  \caption{Sphere reconstruction for MSNR $= 20\,{\rm dB}$ and $N_{\theta} = 90$. Reconstructed image $\tilde{\imapv}$ using (a) ME and (b) TV-$\ell_2$ reconstruction methods. Difference between ground truth $\imapv$ and Reconstructed image $\tilde{\imapv}$ using (c) ME and (d) TV-$\ell_2$ reconstruction methods.}
 \label{fig:Sphere_Comparison}
\end{figure}

\begin{figure} 
  \centering
	\subfigure[\label{fig:Shepp_FBP_90_n1}][$\imapv_{\rm ME}$. RSNR = 13.04\,{\rm dB}.]{\includegraphics[height=1.4in]{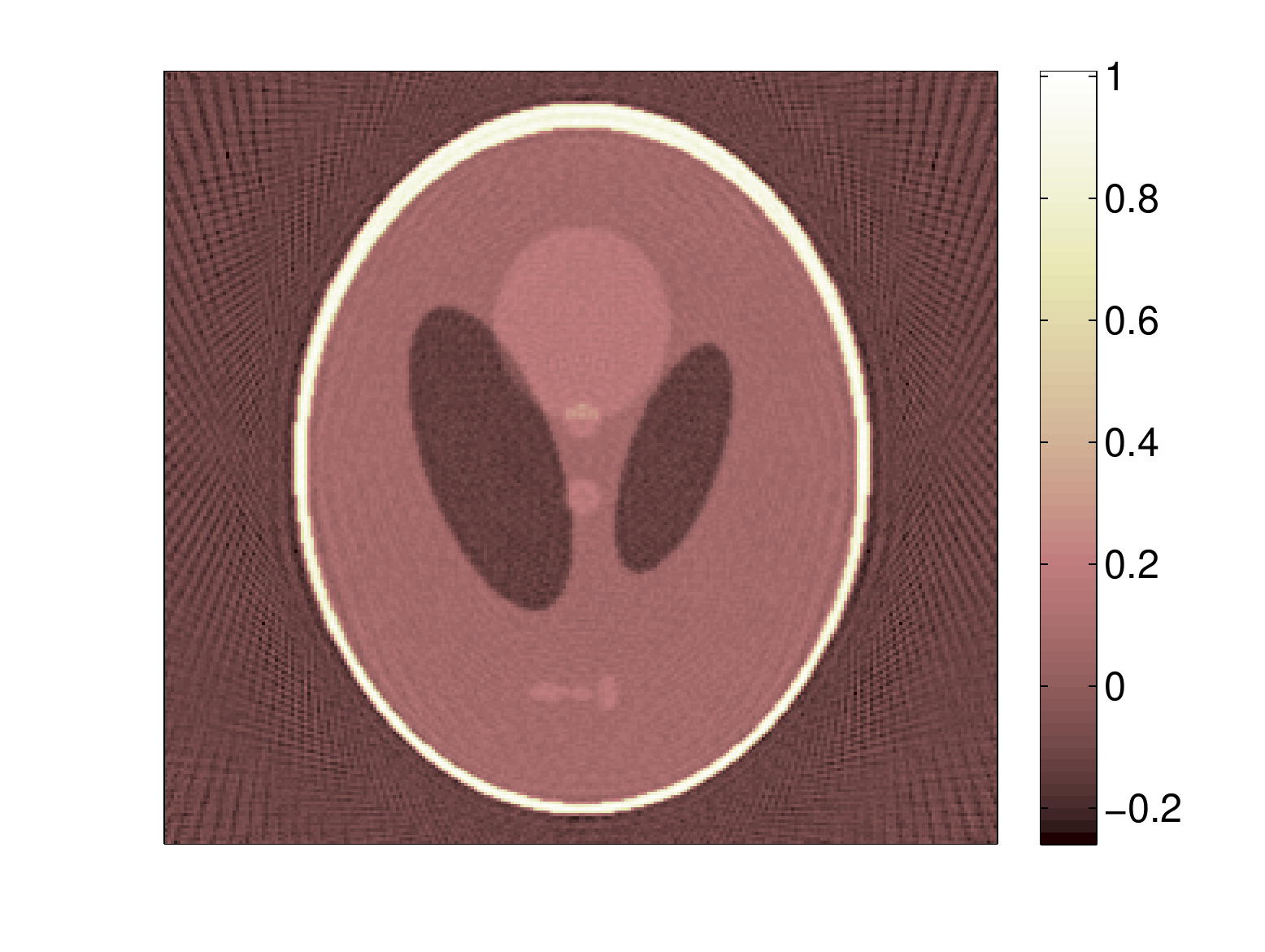}}
	\subfigure[\label{fig:Shepp_TVL2_90_n1}][$\imapv_{{\operatorname{TV-}}\ell_2}$. RSNR = 36.85\,{\rm dB}.]{\includegraphics[height=1.4in]{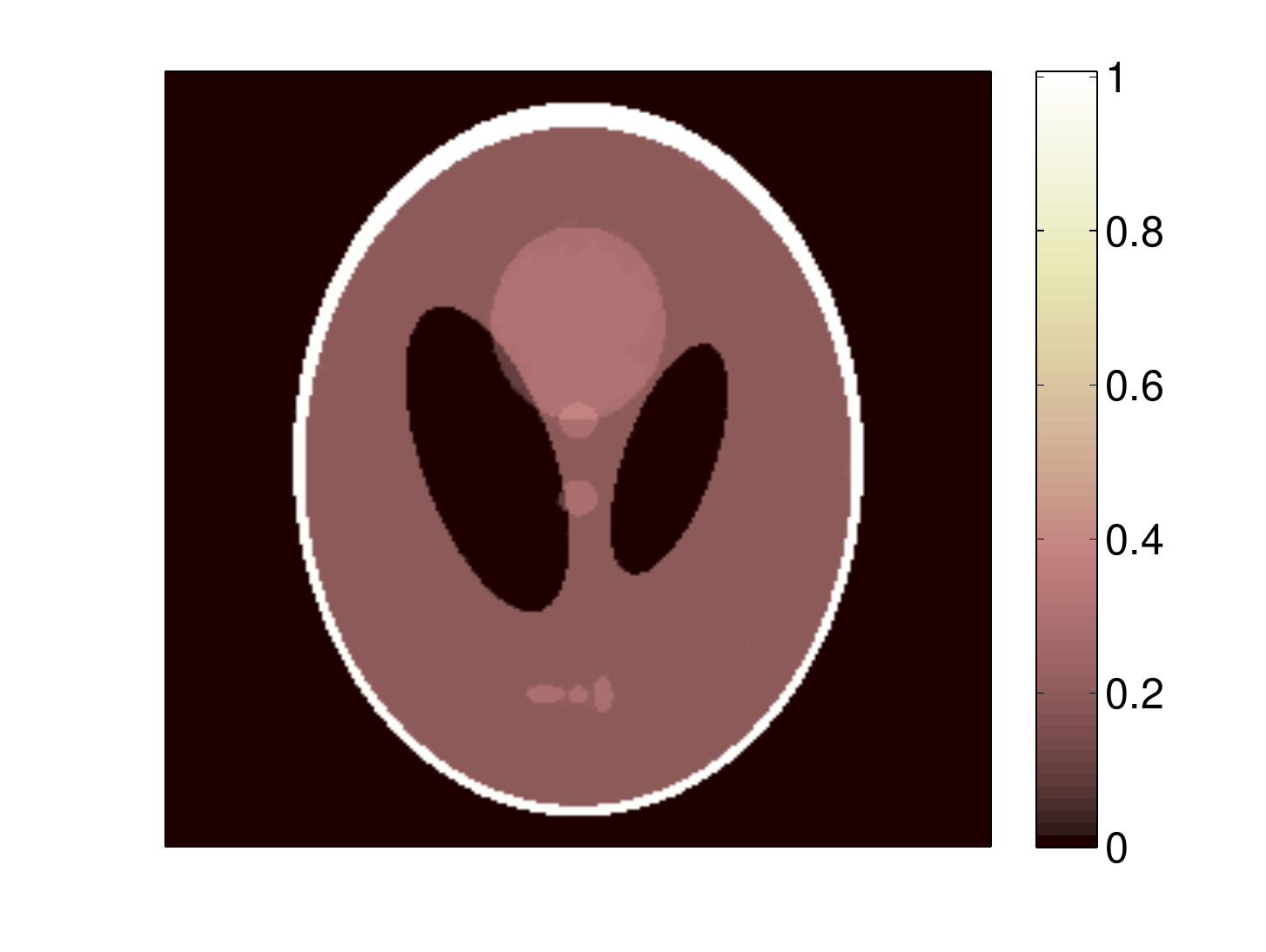}}\\
	\subfigure[\label{fig:Shepp_FBP_90_n1_diff}][$| \imapv - \imapv_{\rm ME} |$.]{\includegraphics[height=1.4in]{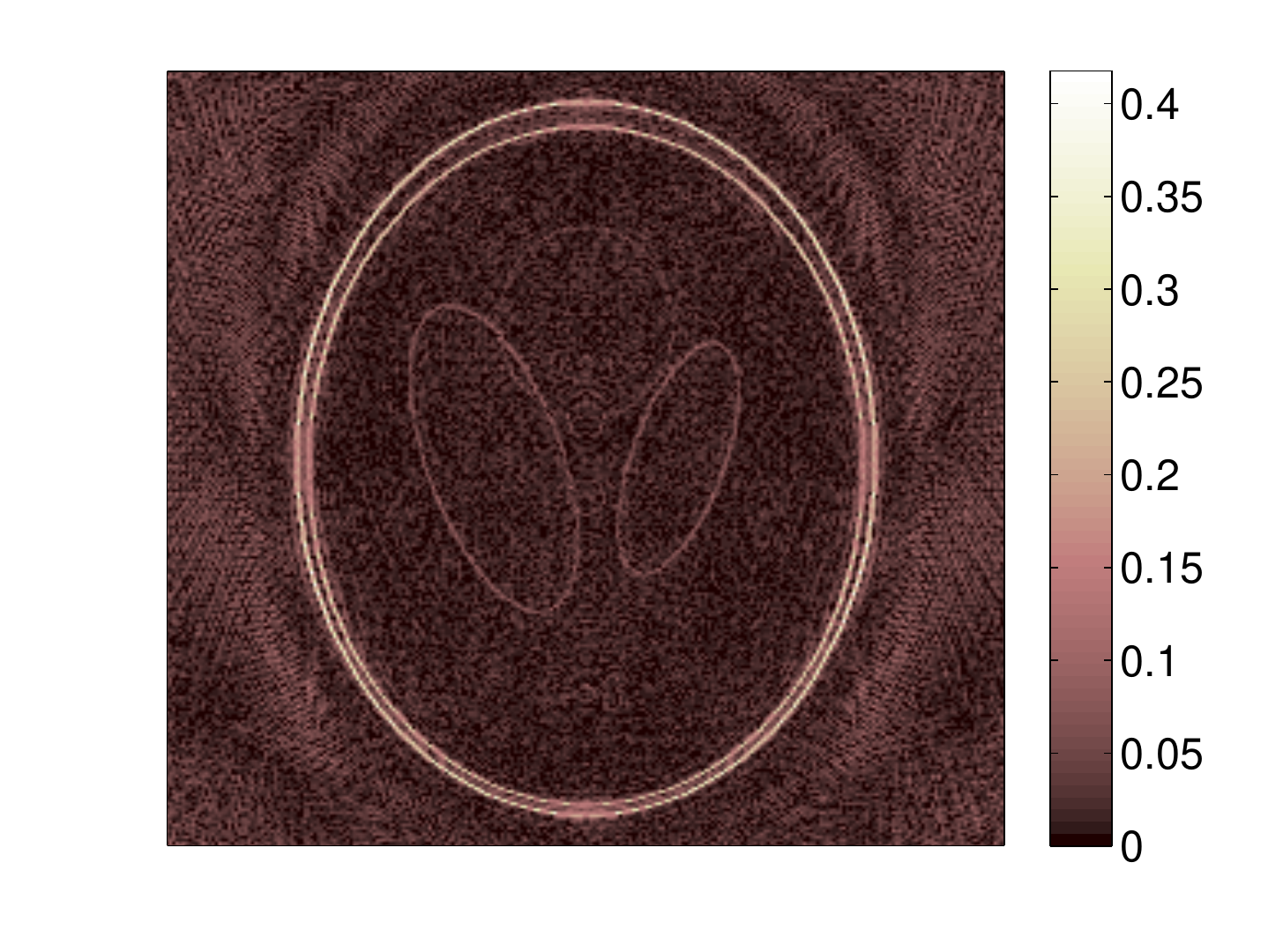}}
	\subfigure[\label{fig:Shepp_TVL2_90_n1_diff}][$| \imapv - \imapv_{{\operatorname{TV-}}\ell_2} |$.]{\includegraphics[height=1.4in]{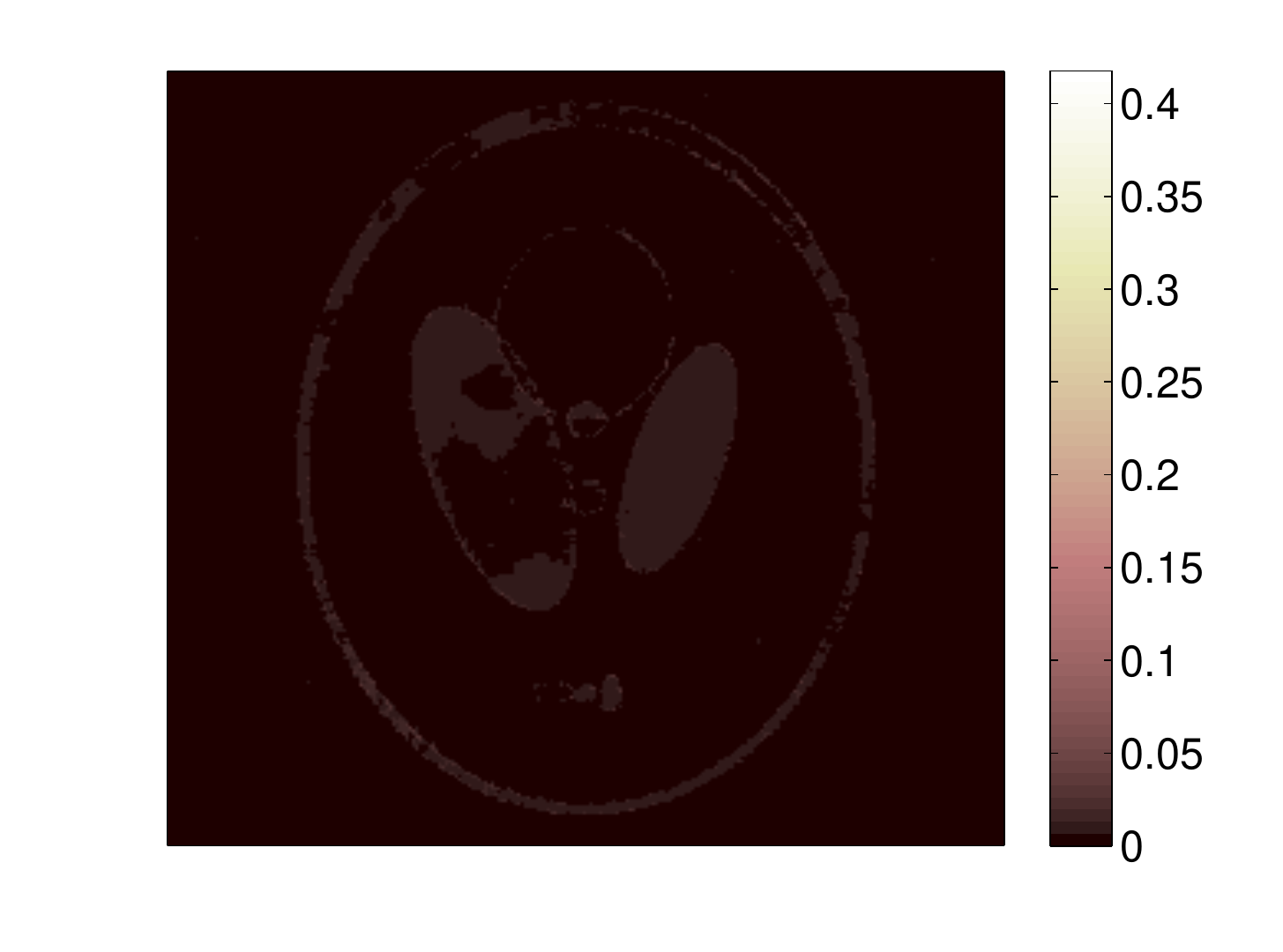}}
  \caption{Shepp-Logan phantom reconstruction for MSNR $= 20\,{\rm dB}$ and $N_{\theta} = 90$. Reconstructed image $\tilde{\imapv}$ using (a) ME and (b) TV-$\ell_2$ reconstruction methods. Difference between ground truth $\imapv$ and Reconstructed image $\tilde{\imapv}$ using (c) ME and (d) TV-$\ell_2$ reconstruction methods.}
 \label{fig:Shepp_Comparison}
\end{figure}

The regularized method provides a better image dynamics for these
phantoms than when reconstructing the fibers for the same noise
level. For these phantoms, it can also be observed that ME
reconstructions present a poor estimation on the borders and
oscillating artifacts. Moreover, the error image shows a higher
discordance with respect to the actual image.

Table~\ref{tab:Images_Comparison} presents a more complete comparison
of the RSNR obtained using ME and TV-$\ell_2$ on the three synthetic
images. The methods are analyzed for three scenarios, one noiseless
with ${\rm MSNR} = \infty$, and the other two with noise such that we
have ${\rm MSNR} = 20\,{\rm dB}$ and ${\rm MSNR} = 10\,{\rm
  dB}$. Results are presented for $\mathtt{Th} = 10^{-5}$ and
$N_{\theta} = 90$, \ie a coverage of $25 \% $ of the frequency plane.

\begin{table}[h]\footnotesize
	\centering
	\begin{tabular}{ c | c | c | c | c | c | c |}
		\cline{2-7}
		 & \multicolumn{6}{|c|}{$\rm RSNR [\rm dB]$}\\
		\cline{2-7}
		 & \multicolumn{2}{|c|}{$\rm MSNR = \infty$} & \multicolumn{2}{|c|}{$\rm MSNR = 20\,{\rm dB}$} & \multicolumn{2}{|c|}{$\rm MSNR = 10\,{\rm dB}$}\\
		\cline{2-7}
		 & TV-$\ell_2$ & ME & TV-$\ell_2$ & ME & TV-$\ell_2$ & ME \\ \hline
		\multicolumn{1}{|c|}{Fibers} & 70.9 & 13.1 & 39.02 & 12.83 & 35.69 & 11.63 \\ \hline
		\multicolumn{1}{|c|}{Sphere} & 53.59 & 21.54 & 45.58 & 21.23 & 37.70 & 18.79 \\ \hline
		\multicolumn{1}{|c|}{Shepp-Logan} & 54.37 & 13.21 & 36.85 & 13.04 & 25.24 & 11.79 \\ \hline
	\end{tabular}
	\caption{Comparison of the different RSNR obtained using ME and TV-$\ell_2$ on the three synthetic images for $N_{\theta} = 90$.}
  \label{tab:Images_Comparison}
\end{table}

When comparing the behavior of the algorithm for the different synthetic images, we can notice that TV-$\ell_2$ method outperforms
the ME method for all cases. 

\subsubsection{Algorithm Convergence}

Finally, we analyze the convergence of the algorithm by studying the
evolution of the primal and dual residual norms and of the RSNR for different
threshold values. We also analyze the convergence difference between the adaptive and the non-adaptive method. 
For this, only the reconstruction of the bundle of fibers for $N_\theta = 360$ and MSNR~=~20\,{\rm dB} is used, however, the other cases present a similar behavior. 

The evolution of the primal and dual residual norms \eqref{eq:PD_residuals_CP_ODT} and the convergence condition in \eqref{eq:convergence_condition}
along the iterations is depicted in Fig.~\ref{fig:Convergence}. We notice that, when the number of iterations increases, the primal and dual residual norms tend to zero along with the energy ($\| P^{(k)} \|^2 + \| D^{(k)} \|^2$).

\begin{figure} 
  \centering
	\includegraphics[height=1.8in]{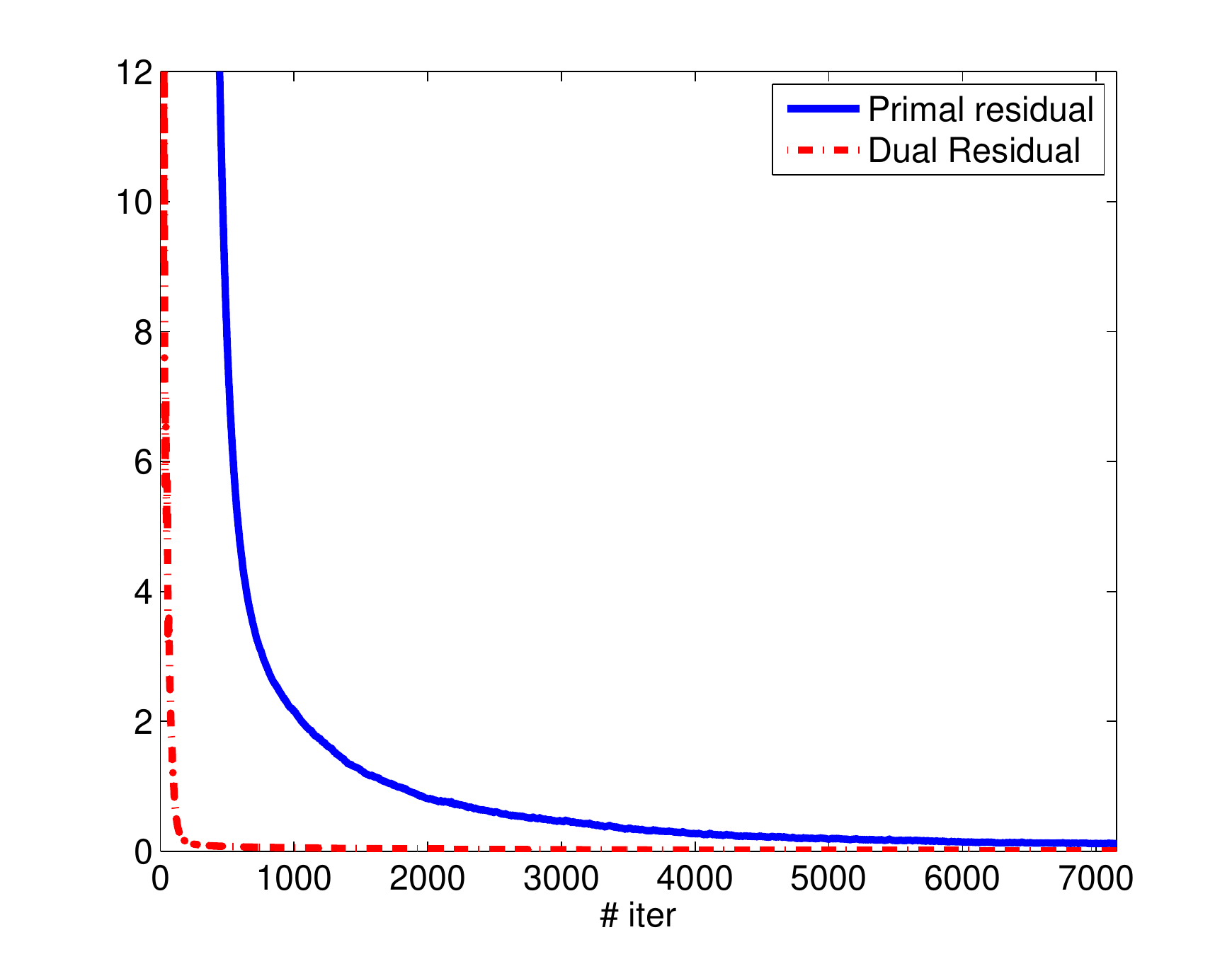}
	\includegraphics[height=1.8in]{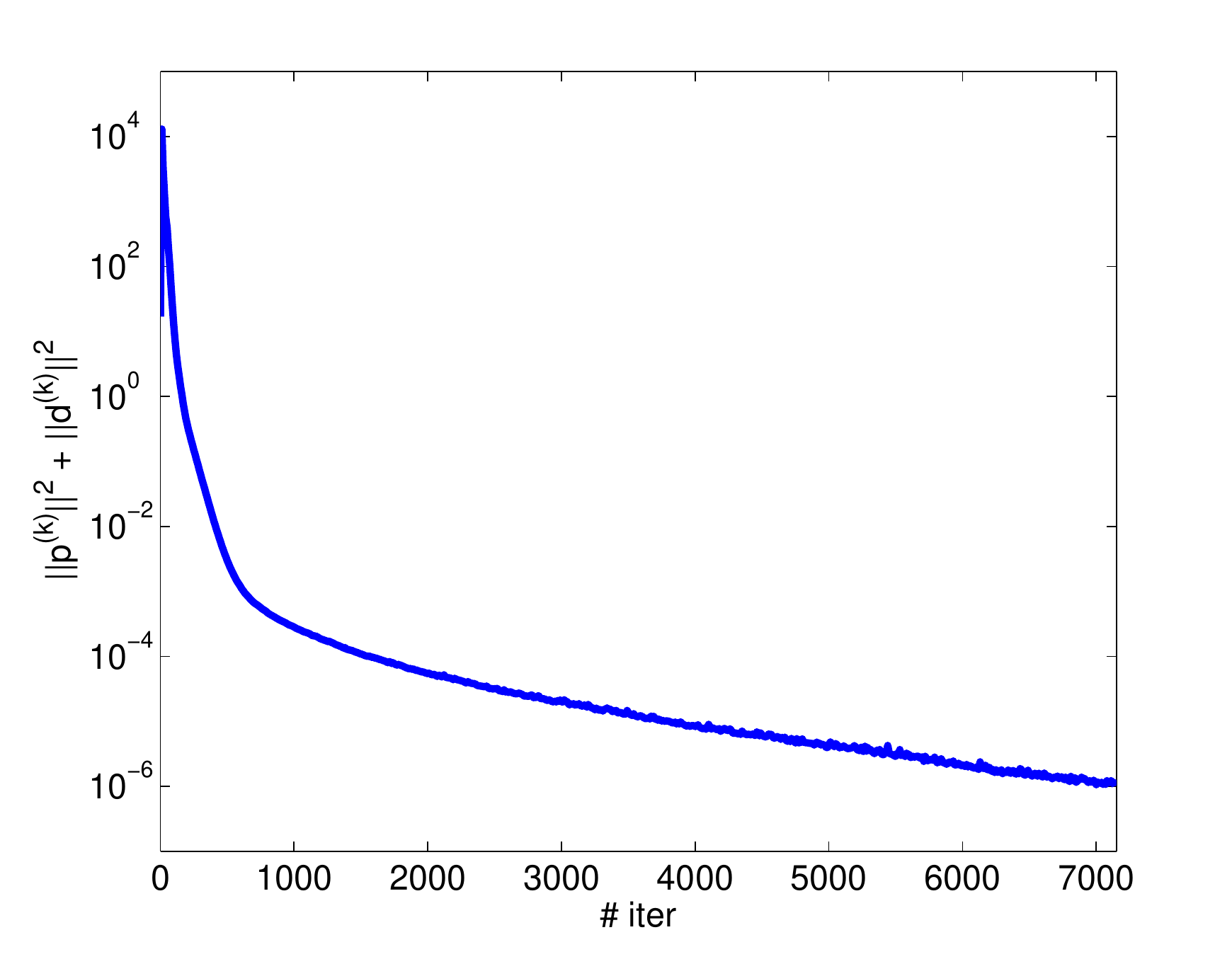} 
  \caption{Convergence results when reconstructing the bundle of fibers with $N_\theta = 360$ and MSNR = 20\,{\rm dB}. (left) The primal and dual residual norms and (right) the convergence condition \eqref{eq:convergence_condition}.}
 \label{fig:Convergence}
\end{figure}

In order to compare the adaptive and non-adaptive methods, we analyze the evolution of  $\| \bs x^{(k+1)} - \bs x^{(k)} \| / \| \bs x^{(k)}
\|$ and the RSNR along the iterations for two cases: \emph{(i)} ``Adapt''  where the stepsizes are adaptively updated using Algorithm \ref{alg:adaptive_CP_ODT} and \emph{(ii)} ``Non-Adapt'' where we use constant stepsizes equal to $\tau = \sigma = 0.9/\bbb\bs K\bbb$. Results are presented in Fig.~\ref{fig:Convergence_Comparison}.

\begin{figure} 
  \centering
	\includegraphics[height=1.8in]{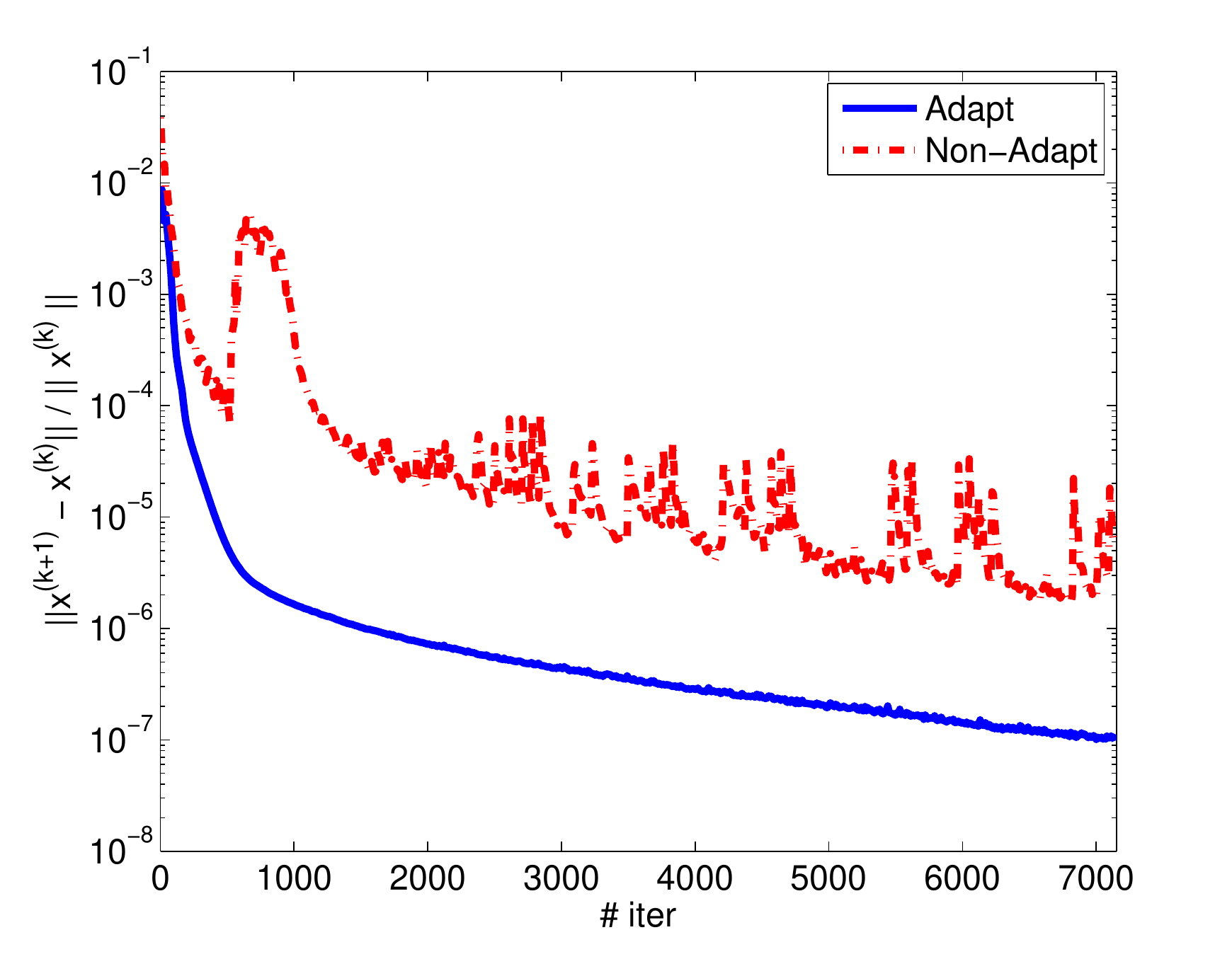}
	\includegraphics[height=1.8in]{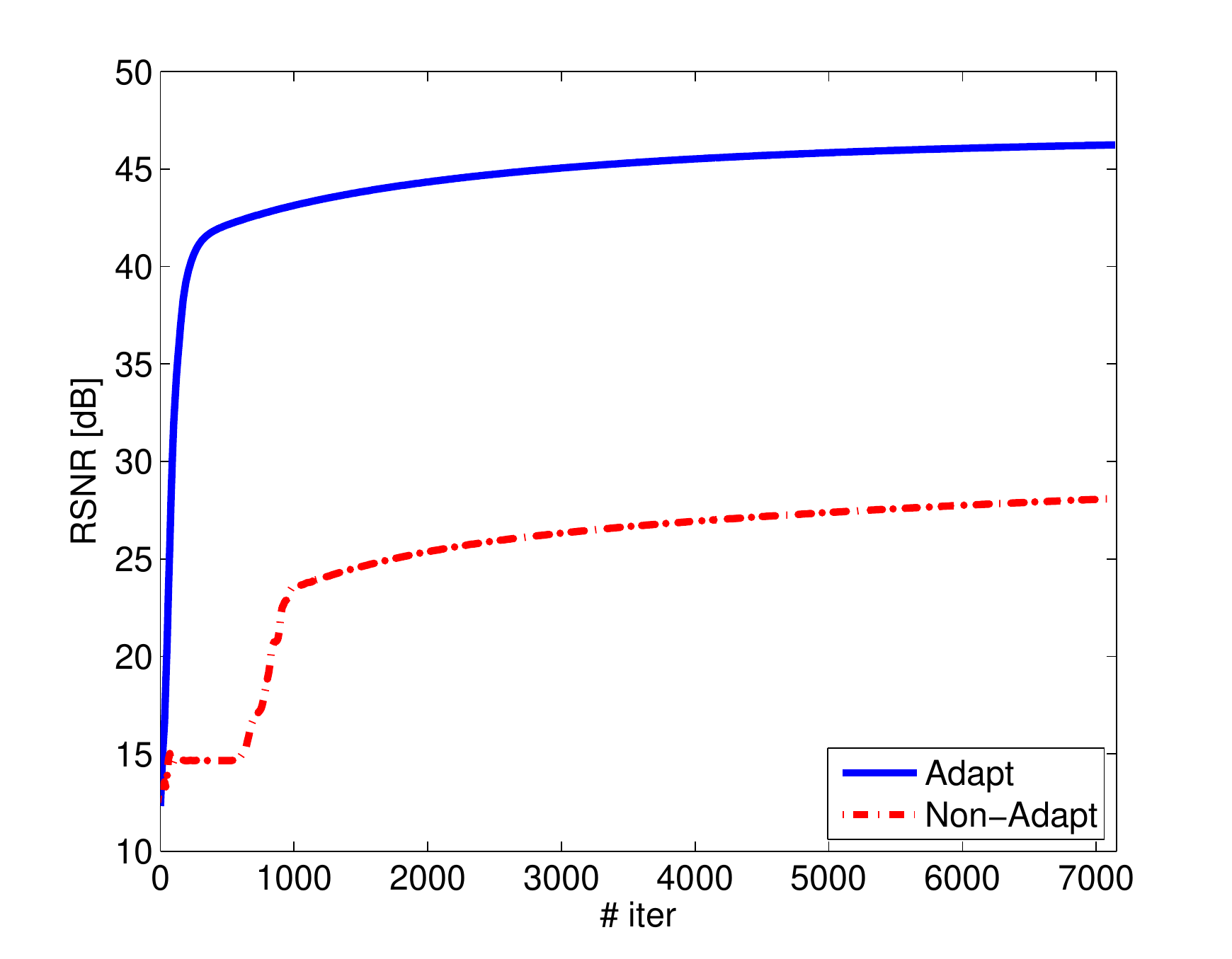}
  \caption{Comparison between the Adaptive and Non-Adaptive method. Convergence results when reconstructing the bundle of fibers with $N_\theta = 360$ and MSNR = 20\,{\rm dB}. The progress of (left) the threshold value Th and (right) the RSNR, along the iterations.}
 \label{fig:Convergence_Comparison}
\end{figure}

The evolution of $\| \bs x^{(k+1)} - \bs x^{(k)} \| / \| \bs x^{(k)}
\|$ helps us to analyze the stability of the algorithm. We can see the curves are not smooth specially for the non-adaptive method, which
indicates a non stable behavior mainly due to a bad conditioning of
the global operator $\bs K$ in the product space optimization. This
could be improved by a preconditioning procedure as described
in~\cite{Becker2012,Preconditioning}. We also notice that for the same number of iterations, the adaptive method converges faster and provides a better result that the non-adaptive. Moreover, the adaptive method does not require to empirically set the parameters $\mu$ and $\nu$ as the algorithm will converge to the optimal parameters independently of the initialization.

Table~\ref{tab:Convergence_Comparison} presents the values of the time and RSNR for different threshold
values. These results show that
a lower threshold provides higher reconstruction quality but
significantly increases the number of CP iterations. 
 In this specific reconstruction, setting the threshold to $10^{-5}$ or running more than $500$ CP iterations guarantees a RSNR higher than $42$ dB.

\begin{table}[h]
	\centering
	\begin{tabular}{ | c | c | c | c | c | c | c |}
		\hline
		$\mathtt{Th}$ & \# iter & Time & RSNR [dB] \\ \hline
		$10^{-4}$ & 190 & 5' & 38.79 \\ \hline
		$10^{-5}$ & 420 & 11' & 41.86 \\ \hline
		$10^{-6}$ & 1540 & 42' & 43.86 \\ \hline
		$10^{-7}$ & 7150 & 3h15' & 46.24 \\ \hline
	\end{tabular}
	\caption{Convergence results when reconstructing the 10 fibers synthetic image with $N_\theta = 360$ and MSNR = 20\,{\rm dB}.}
  \label{tab:Convergence_Comparison}
\end{table}

\subsection{Experimental Data} 
\label{sec:expdata}

The reconstruction algorithm was tested with two particular
transparent objects similar to the synthetic data studied in the
previous section: a homogeneous sphere and a bundle of 10~fibers, both
immersed in an optical fluid. The reconstruction is based on $N_\tau =
696$ parallel and equally spaced light~rays. The experimental setup is
based on the Schlieren Deflectometric Tomography described in
Sec.~\ref{sec:problem}.

A $696 \times 523$ pixels CCD camera was used for the acquisition,
covering a field of view of $3.25 \rm mm \times 2.43 \rm mm$. This
corresponds to $N_\tau = 696$ parallel light rays and 523 2-D slices,
which leads to $\delta \tau = 4.7 \times 10^{-3} \rm mm$, and thus to
$\delta r = 4.7 \times 10^{-3} \rm mm$.

Fig.~\ref{fig:Experimental_Fibers_Obs} presents one measurement
  of the deflection angles on the CCD camera grid for the two analyzed optical phantoms. This observation corresponds to $\theta = 0^\circ$.

\begin{figure}[tbh]
	\centering 
	\includegraphics[height=1.4in]{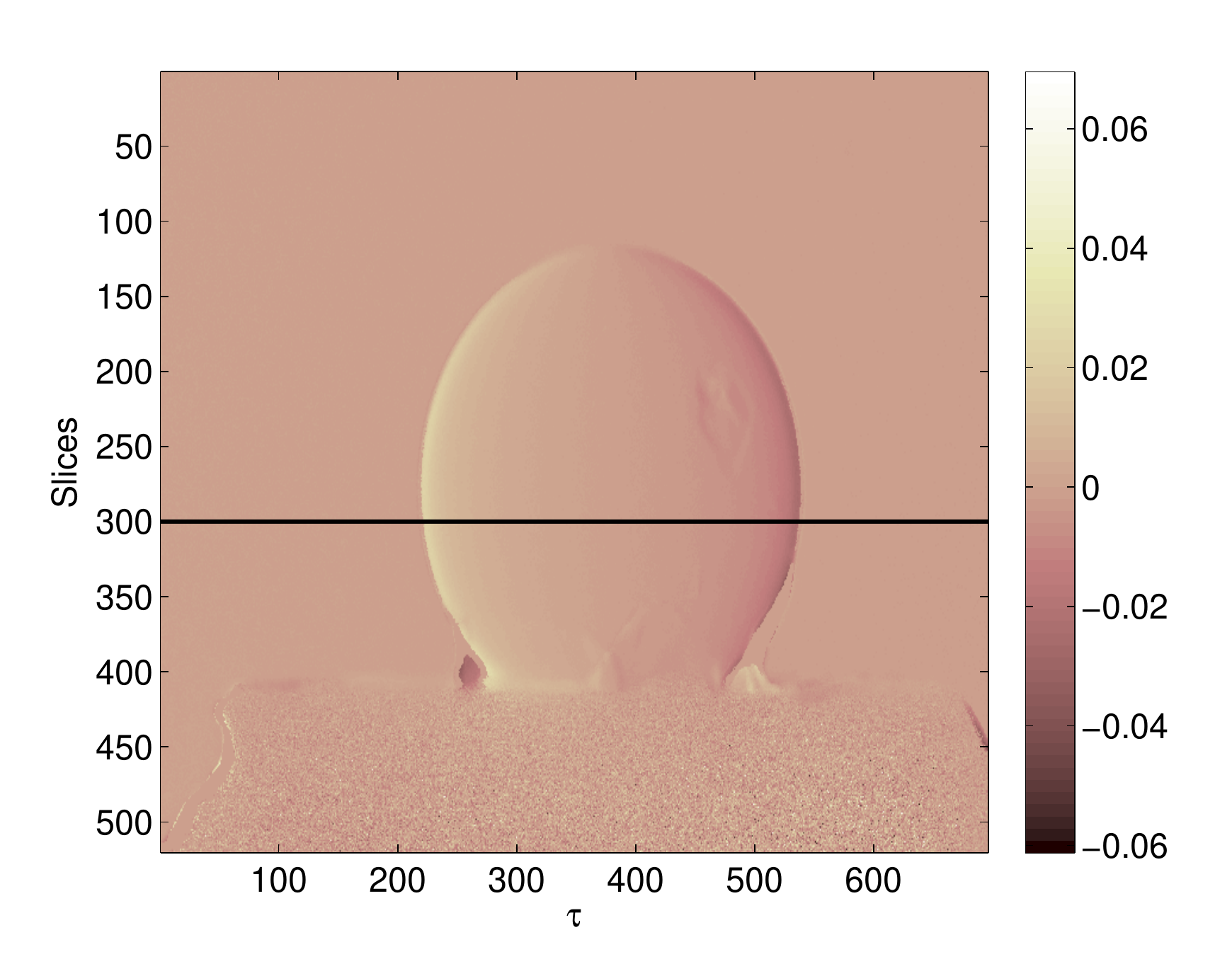}
	\hspace{0.5cm}
	\includegraphics[height=1.4in]{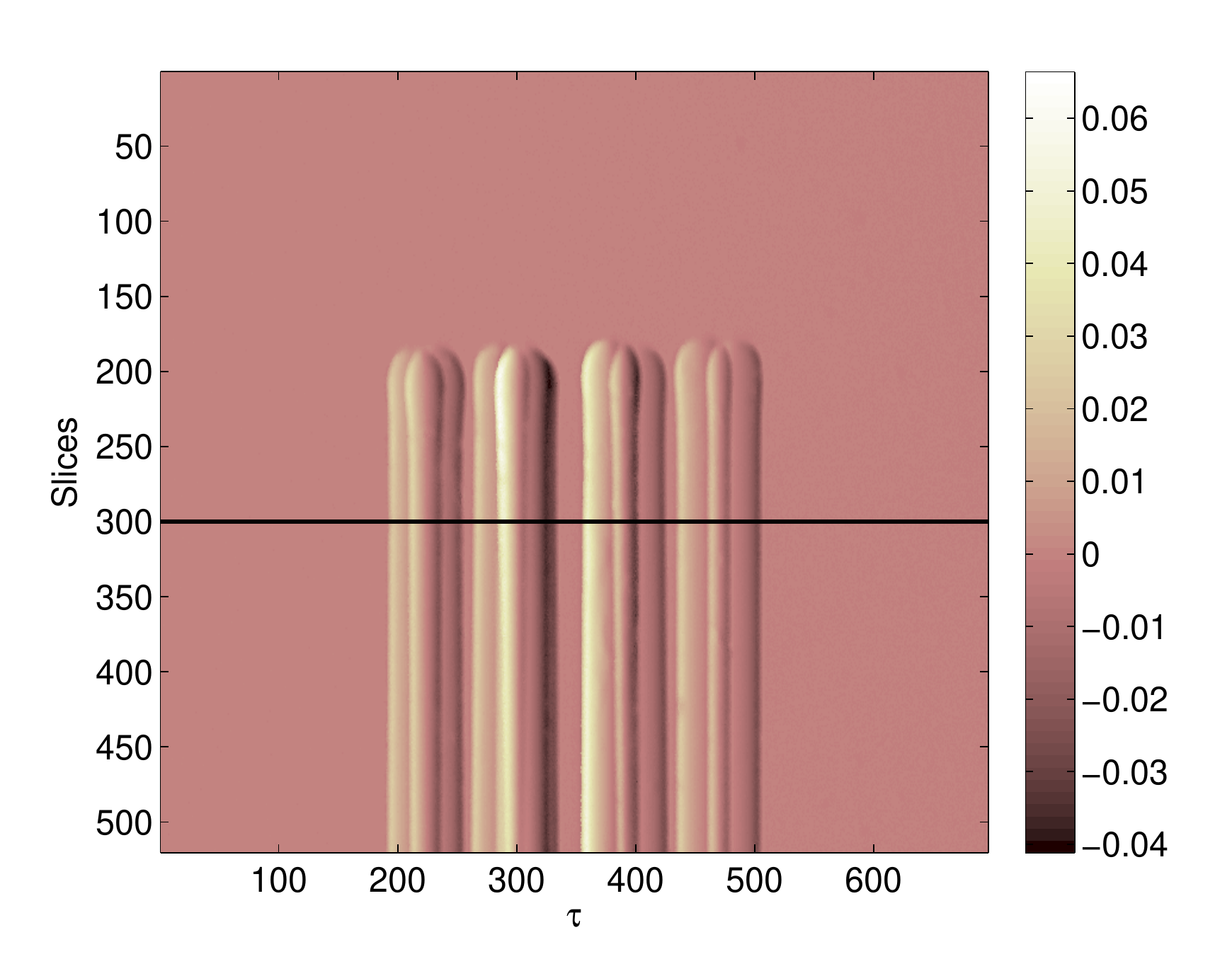}
	\caption{Computed deflection angles (in radians) from CCD camera observations for
          $\theta = 0^\circ$, corresponding to (left) the homogeneous sphere and (right) the bundle of
          fibers. The 300$^{\rm th}$ slice
          (over the 512 available slices) used for the reconstruction
          is indicated in the figures using a black line. In this
          specific slice, the refractive index maps are expected to have negligible
          variations along the vertical direction ($\bs e_3$).}
	\label{fig:Experimental_Fibers_Obs}
\end{figure}

The experimental configuration leads to a calibration problem. As the
object is rotating, the rotation center is modified within a small
range and the origin of the affine parameter $\tau$ is altered. A
post-acquisition calibration method was therefore implemented for
correcting this effect. In short, for each angle, the method estimates the true
centrum location by averaging the locations of the maximum and
minimum deflection values along $\tau$. 

The two next paragraphs present the characteristics of the two objects of
interest and the reconstruction results obtained for the three tested
methods (FBP, ME and TV-$\ell_2$) from the collected experimental observations. In
all these experiments, and as discussed in Sec.~\ref{sec:noise-sources}, the TV-$\ell_2$ method was considered in the
context of a 10 dB modeling noise. This choice seems somehow optimal
for our experimental conditions. In several tests not reported here,
higher or lower values of noise SNR lead either to severe artifacts in areas where no object is expected 
or to a significant loss in the expected RIM dynamics.

\subsubsection*{Homogeneous Sphere}

The first observed object consists of a homogeneous sphere with a
diameter of $1.5$ mm. The difference of refractive index between the
sphere and the optical fluid where it is immersed is $\delta \imap =
2.8 \times 10^{-3}$. The deviation map was measured for $N_\theta =
45$ angular positions over 360 degrees (\ie 13\% of
measurements). The most important noise in the measurements is the modeling noise obtained by assuming MSNR = 10 dB \eqref{eq:epsilon_model}, which provides $\varepsilon_{\rm model} = 0.035$; while the estimated observation noise provides $\varepsilon_{\rm obs} = 0.008$. The NFFT interpolation noise is considerably smaller, with $\varepsilon_{\rm nfft} = 4.24 \times 10^{-16}$ . This noise estimation provides a MSNR = $9.79$\,dB. Fig.~\ref{fig: Experimental_Sphere} shows the
reconstruction results obtained when using FBP, ME and TV-$\ell_2$
reconstruction methods, for the 300$^{\rm th}$ 2-D slice of the
  observed 3-D object. The results are shown for a threshold $\rm \mathtt{Th} = 10^{-5}$,
where ME converges in $7180$ iterations and TV-$\ell_2$ in
$5180$ iterations.

\begin{figure}[tbh]
	\centering 
	\includegraphics[height=1.21in]{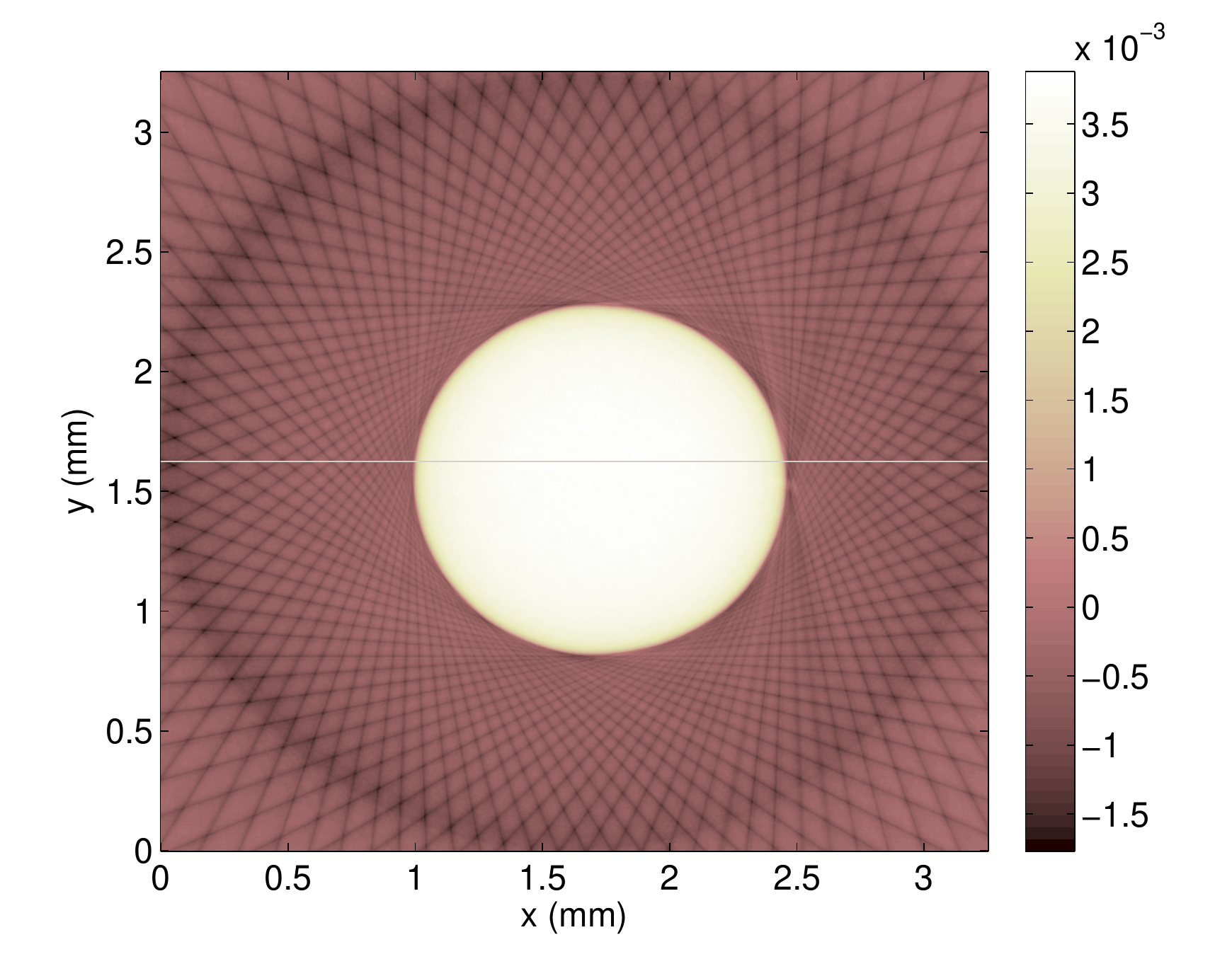}
	\includegraphics[height=1.21in]{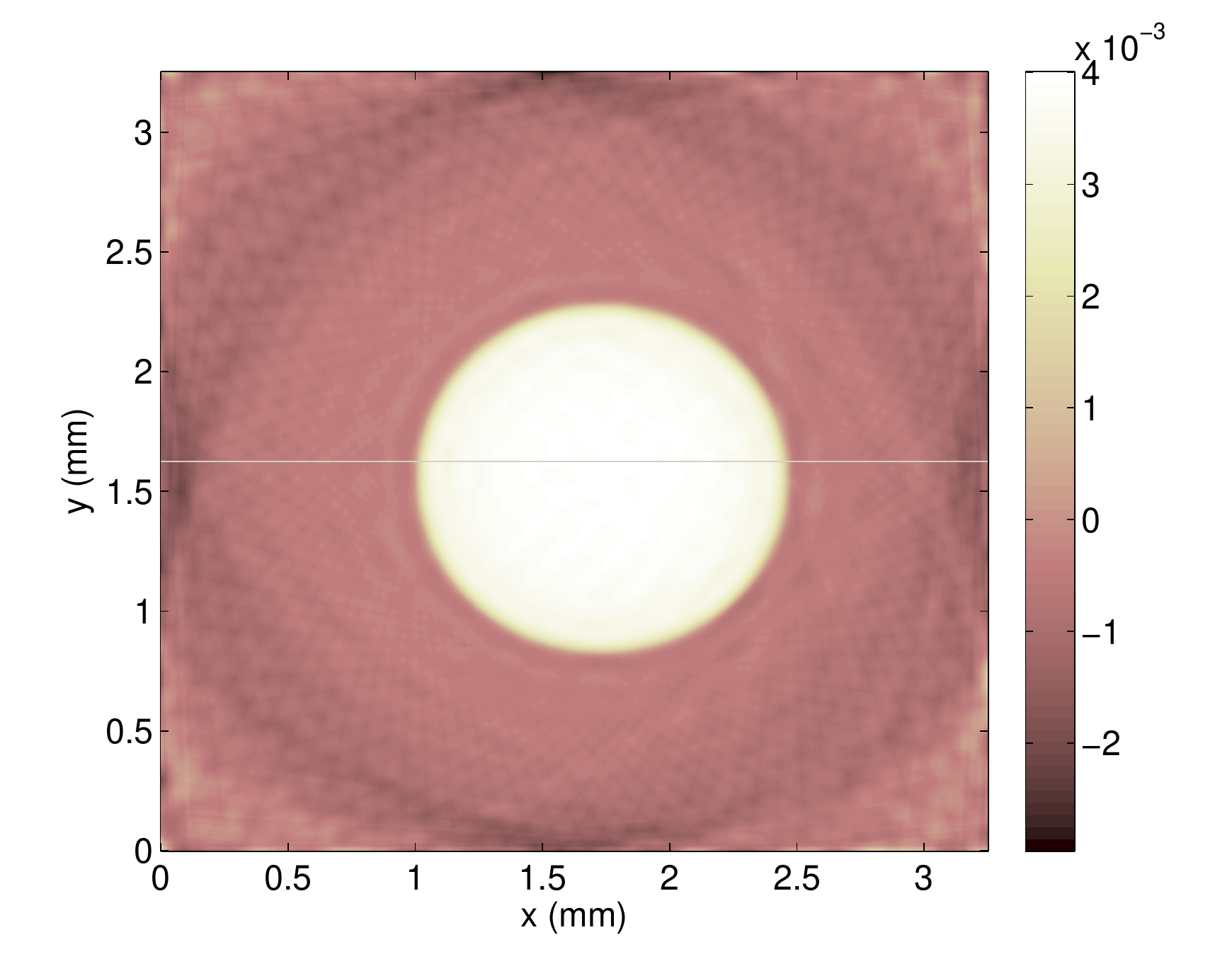}
	\includegraphics[height=1.21in]{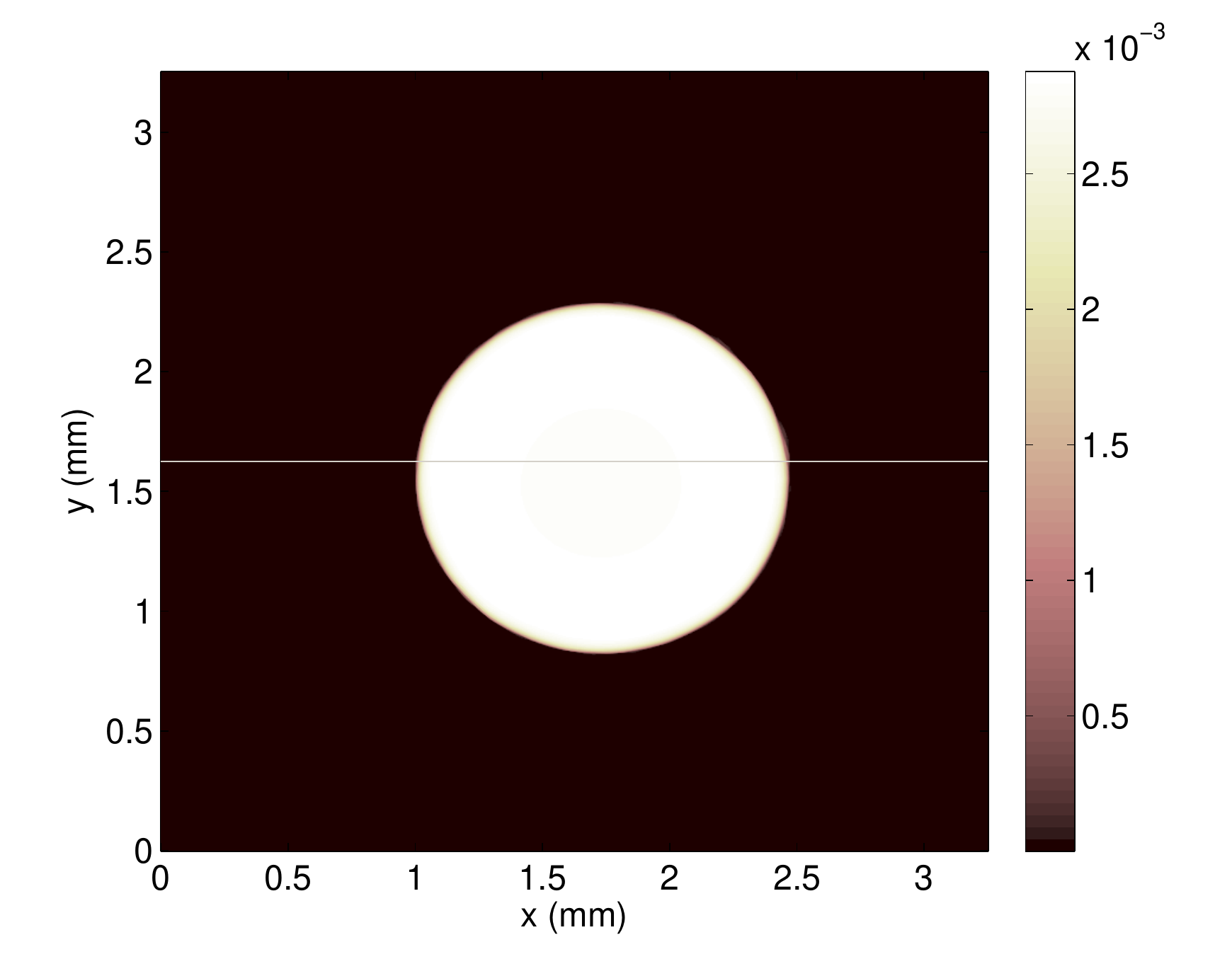} 
	\includegraphics[height=1.21in]{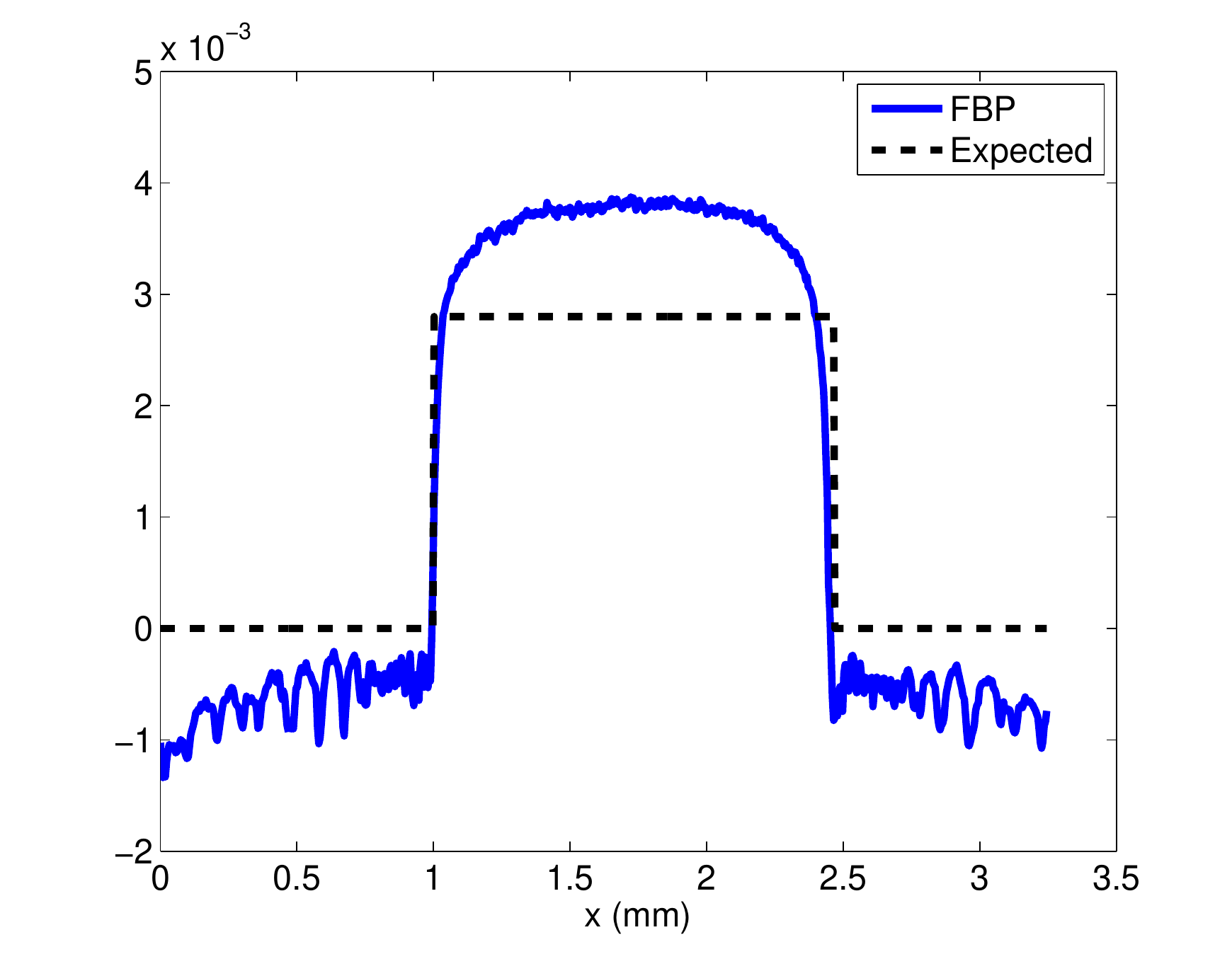} 
	\includegraphics[height=1.21in]{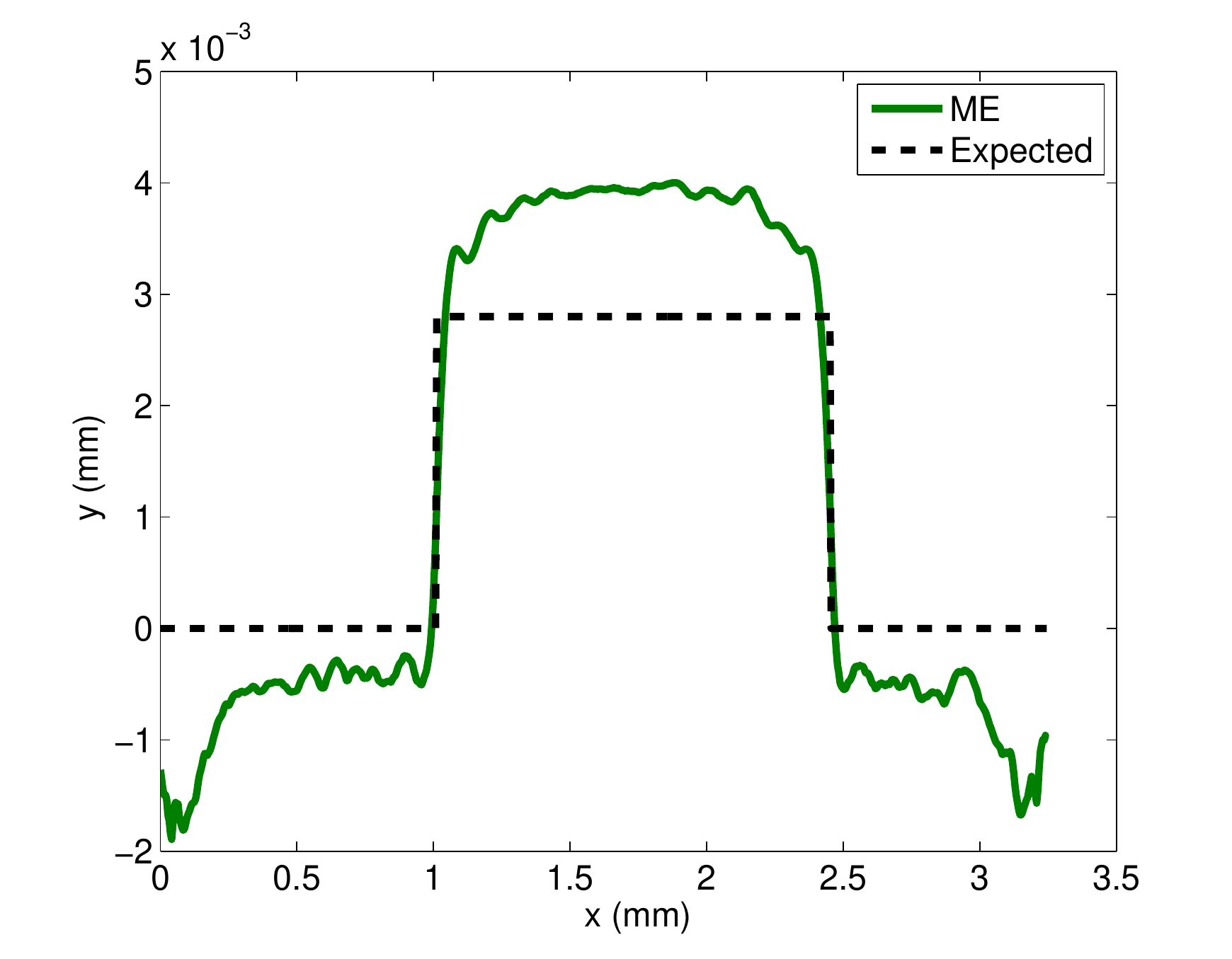}
	\includegraphics[height=1.21in]{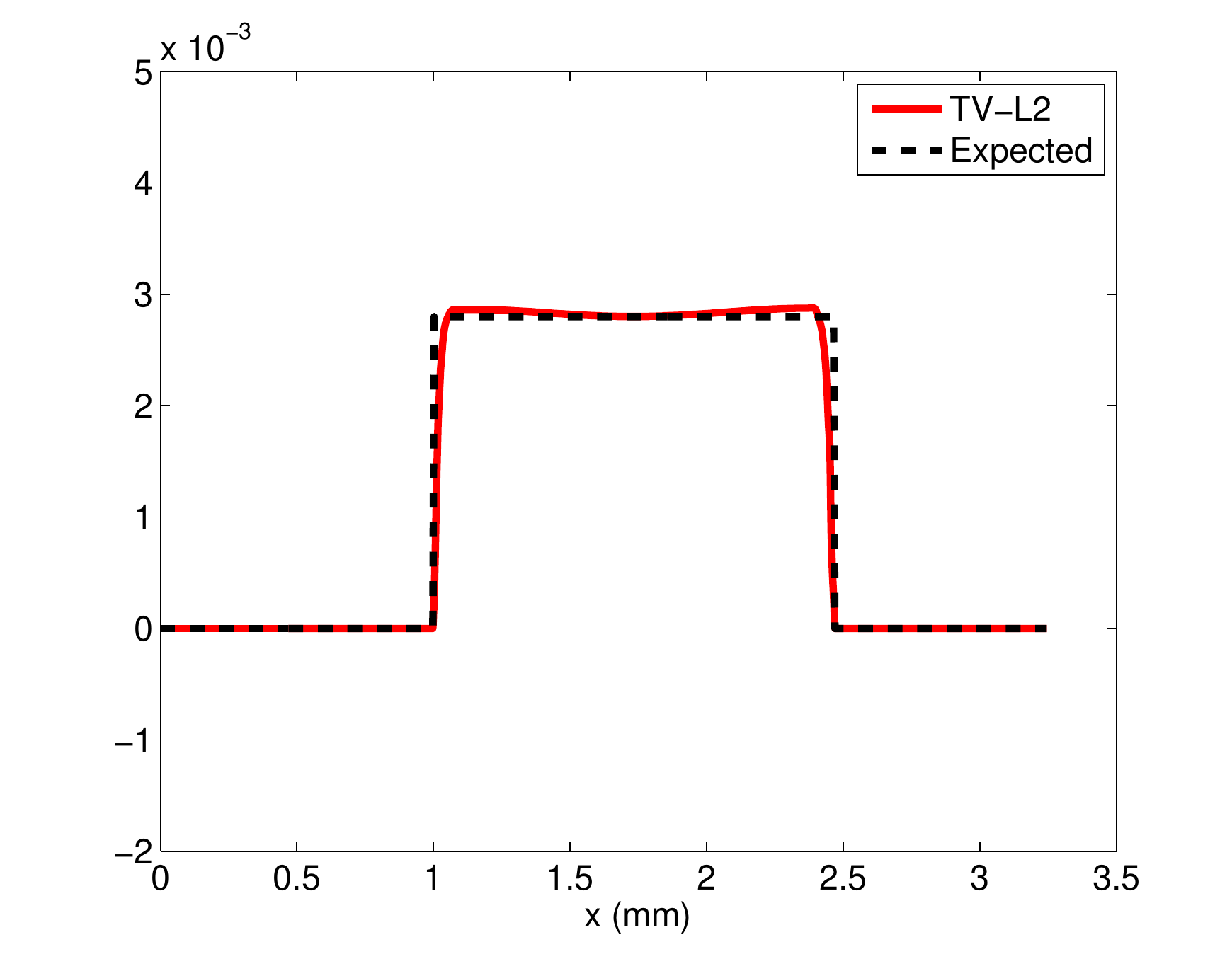} 
	\caption{Reconstructed sphere for 45 angular positions for the 300$^{\rm th}$ 2-D slice. 2-D distribution using (top left) FBP, (top center) ME and (top right) TV-$\ell_2$. 1D profile along y~=~1.625 mm using (bottom left) FBP (bottom center) ME and (bottom right) TV-$\ell_2$.}
	\label{fig: Experimental_Sphere}
\end{figure}

We observe a similar behavior to the one found in the synthetic
reconstruction. Compared to FBP and ME results, the sphere frontier is sharper with the TV-$\ell_2$
estimation and the RIM vanishes on the background. The image dynamics
recovery is also more accurate using
our regularized method, whereas with FBP and ME the reconstructions present
several artifacts with implausible negative
values. It is important to notice that the preservation of the image
dynamics depends mainly on the noise estimation and on the proper
definition of the constants included in the operator $\bs D$ (see
Sec.~\ref{sec:processing}). When considering the appropriate constants,
we are able to make an equivalence between the physical problem and
its discrete mathematical formulation.

\subsubsection*{Bundle of fibers}

The second measured object is a bundle of 10 fibers of 200 $\mu$m
diameter each. The refractive index difference with respect to the
solution where the fibers are immersed is $\delta \imap = 12.1 \times
10^{-3}$. The experimental data was measured for 60 angular positions
over 360 degrees (\ie 17\% of
measurements). As in the case of the sphere, the noise that has more influence in the measurements is the modeling noise obtained by assuming MSNR = 10 dB \eqref{eq:epsilon_model}, which provides 
$\varepsilon_{\rm model} = 0.093$; while the observation noise
provides $\varepsilon_{\rm obs} = 0.005$. The NFFT interpolation noise
is considerably smaller ($\varepsilon_{\rm nfft} = 1 \times
10^{-15}$). This noise estimation provides MSNR = $9.98$\,dB.

Fig.~\ref{fig:Experimental_Fibers} shows the
reconstruction results obtained using FBP, ME and TV-$\ell_2$
reconstruction methods, for the 300$^{\rm th}$ 2-D slice.
For a threshold of $10^{-5}$, ME converges in 33920 iterations and TV-$\ell_2$ in 4560 iterations.
Compared to FBP and ME reconstructions, TV-$\ell_2$ provides
  a much shaper estimation of the true RIM and the background is
  correctly estimated to 0. However, the image dynamics is not properly recovered. 
Such reconstruction error is present for the fibers because the 
refractive index difference between the material and the optical fluid 
is higher than for the sphere. In this case, the deflection angle is higher 
and it causes the modeling error to increase. The actual light ray trajectory could 
be estimated and inserted in an iterative process as done by Antoine
et al.~\cite{optimess2012}. However, the forward model could not be 
represented in the frequency domain and we would loose its fast
computation. 

\begin{figure}[tbh]
	\centering 
	\includegraphics[height=1.21in]{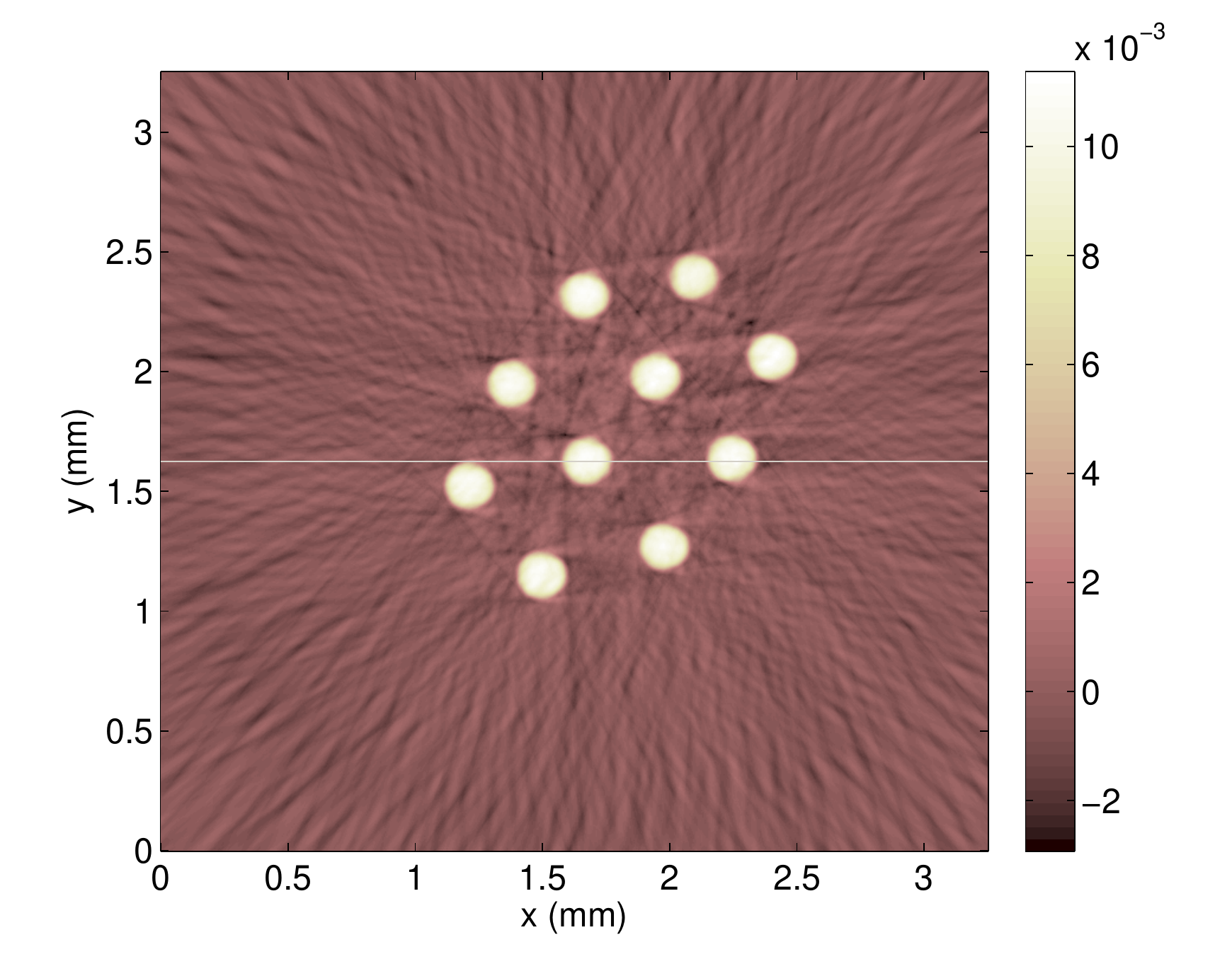} 
	\includegraphics[height=1.21in]{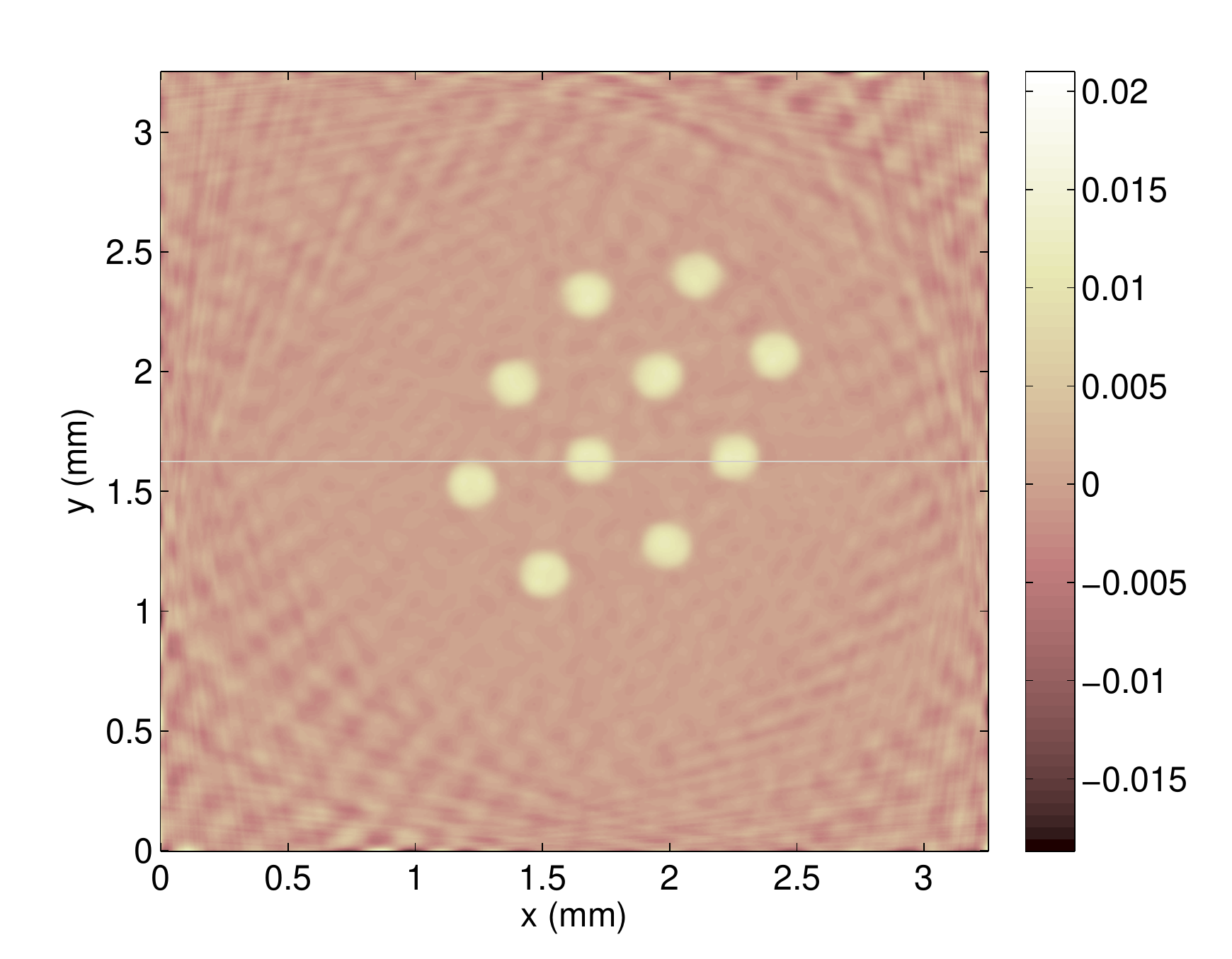}
	\includegraphics[height=1.21in]{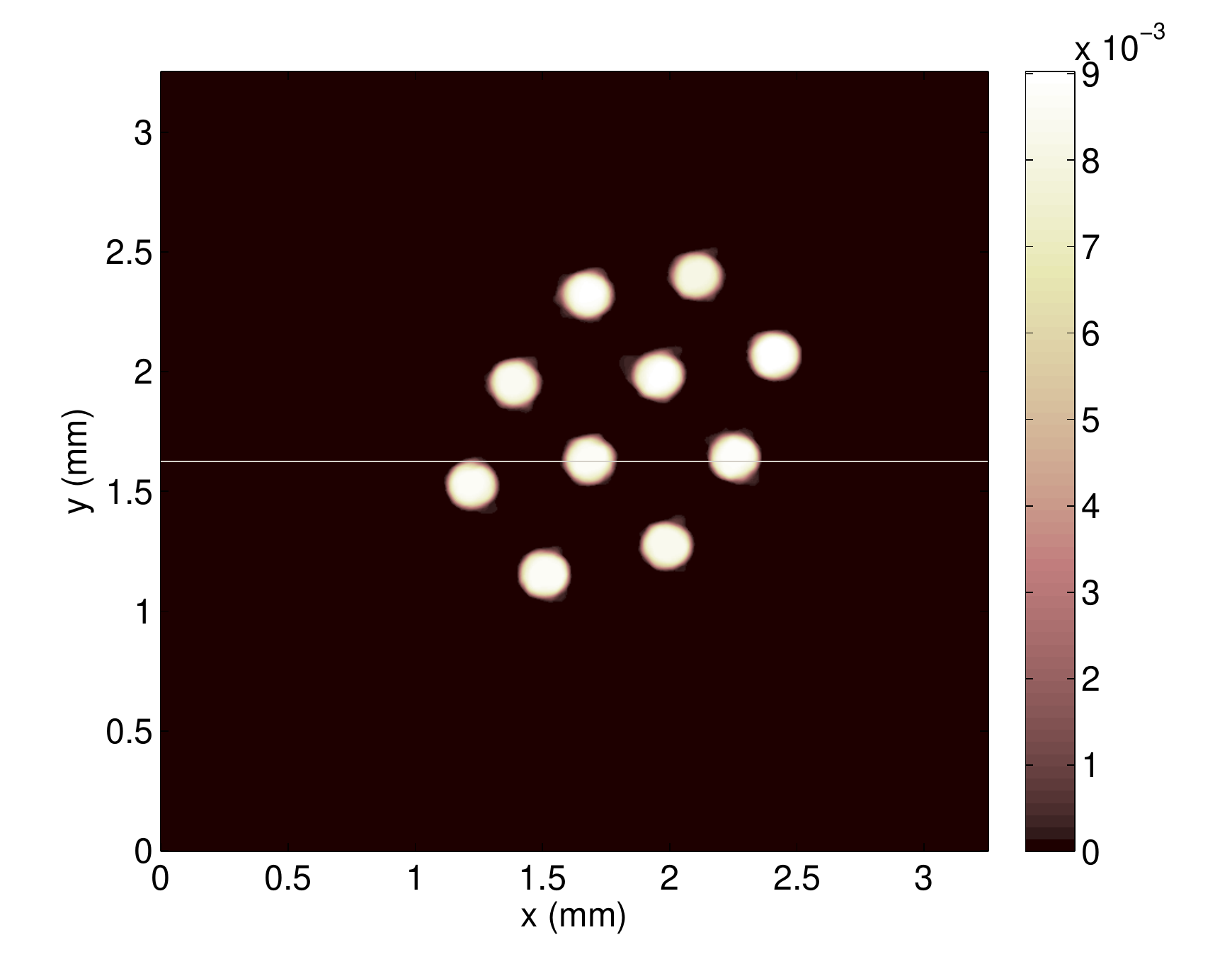}
	\includegraphics[height=1.21in]{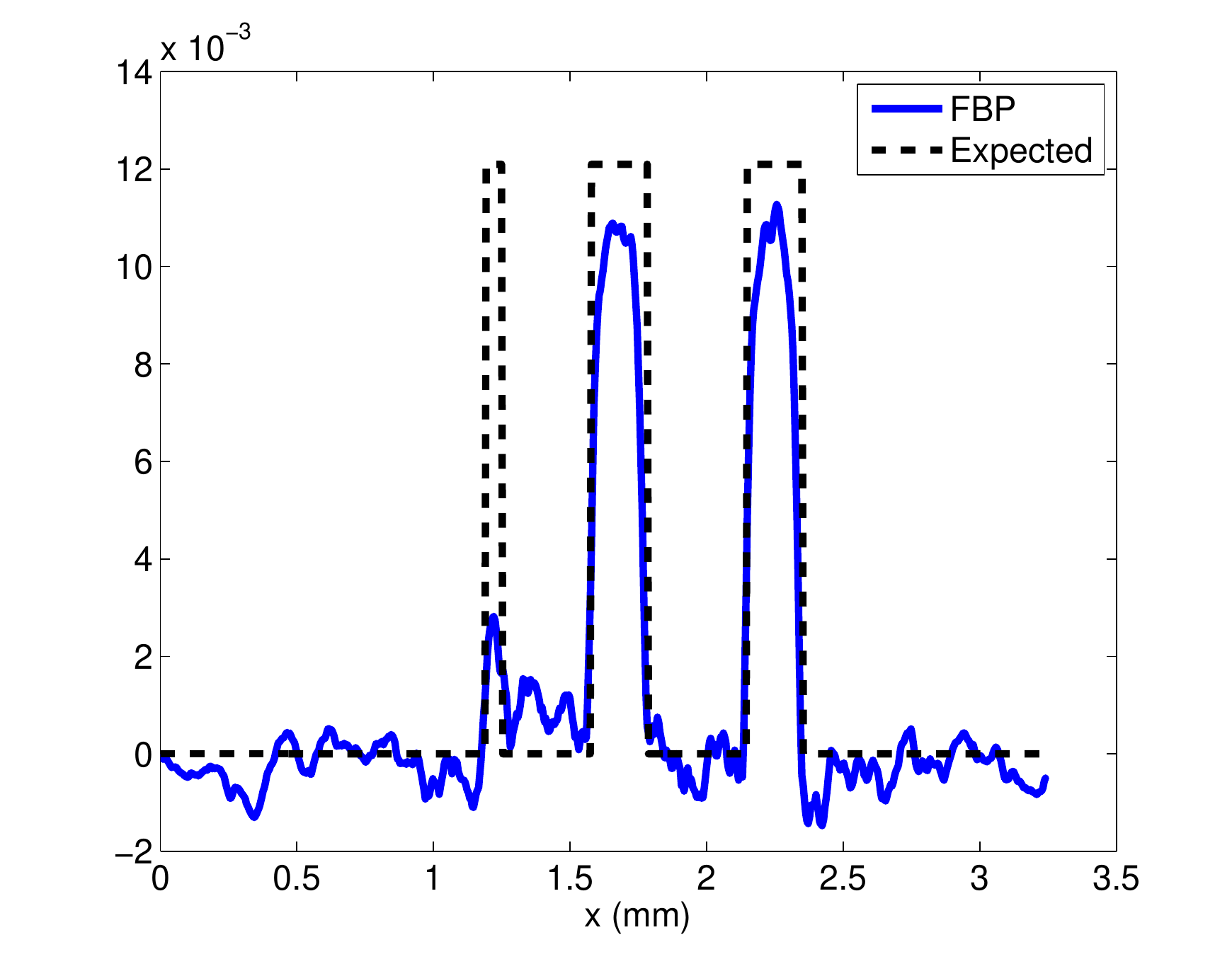} 
	\includegraphics[height=1.21in]{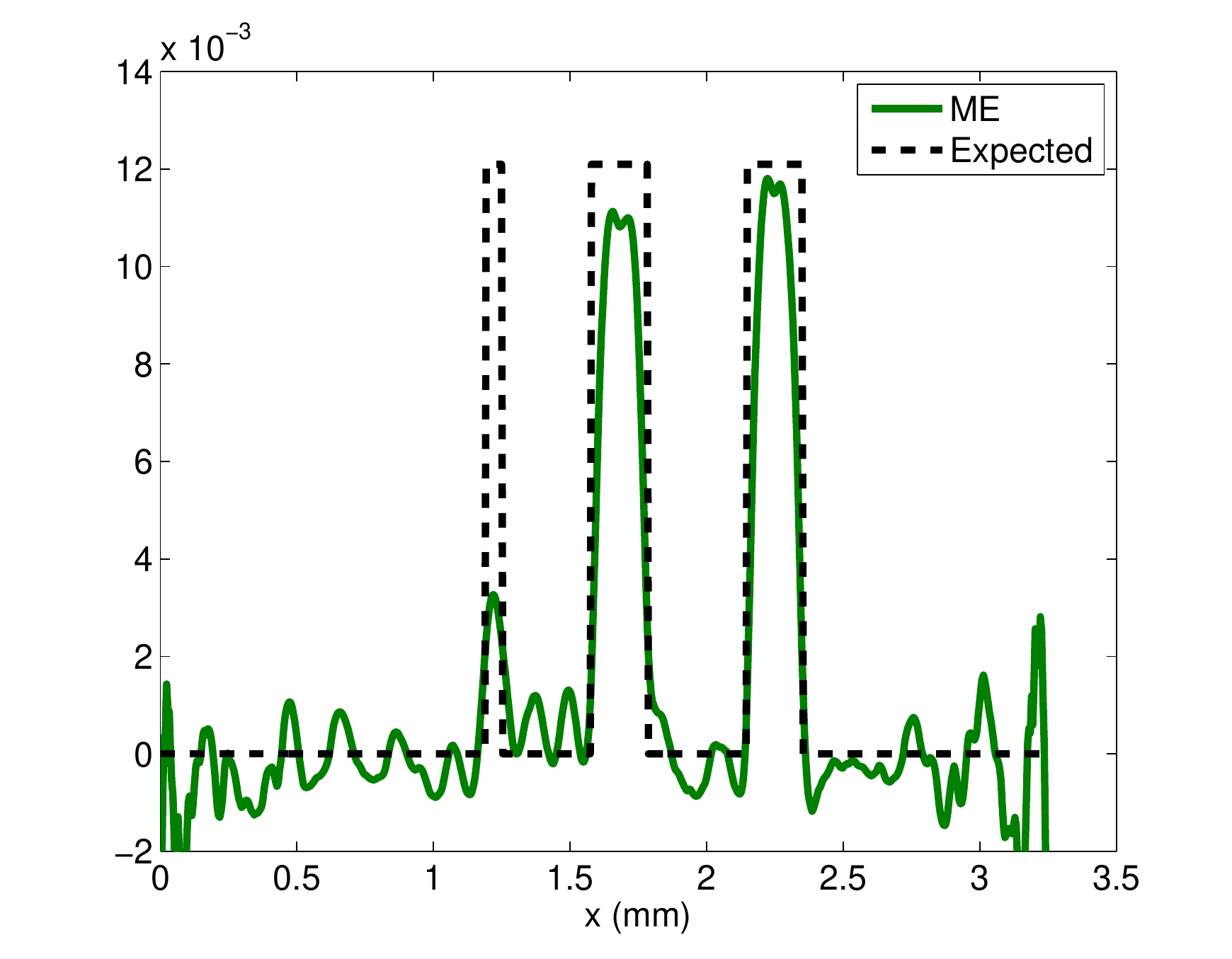}
	\includegraphics[height=1.21in]{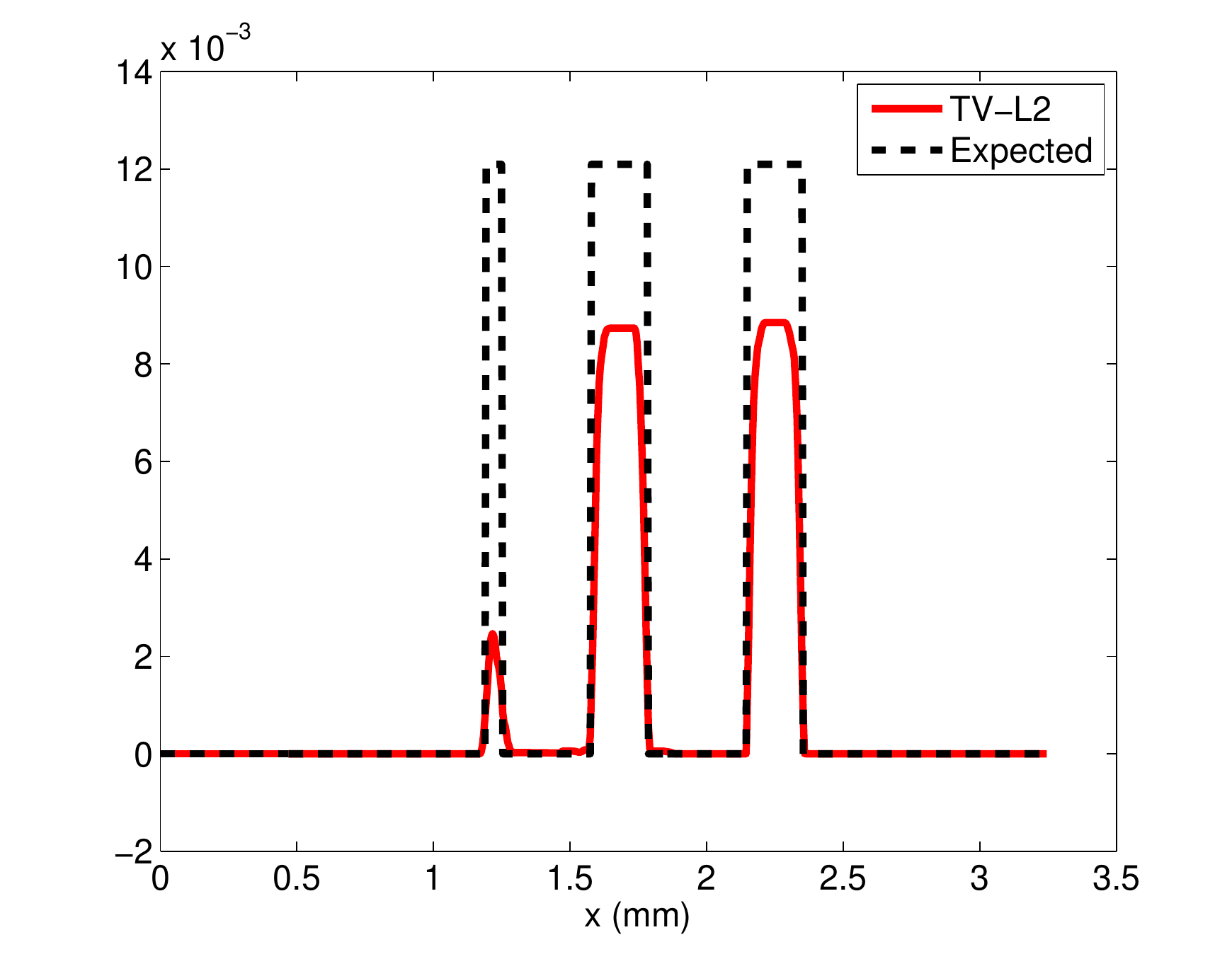} 
	\caption{Reconstructed bundle of 10 fibers for 60 angular
          positions for the 300$^{\rm th}$ 2-D slice. 2-D distribution using (top left) FBP, (top center) ME and (top right) TV-$\ell_2$. 1D profile along y~=~1.625 mm using (bottom left) FBP, (bottom center) ME and (bottom right) TV-$\ell_2$.}
	\label{fig:Experimental_Fibers}
\end{figure}

We can also notice that a section of the bundle of fibers represents a complex image to
reconstruct, since the light enters and comes out of multiple fibers,
and this amplifies the modeling error.

Furthermore, investigating the TV-$\ell_2$ program \eqref{eq:inv-problem-deflecto-solving} 
and assuming that the unknown true RIM is a feasible point of the
constraints, \ie the noise power $\varepsilon$ is correctly bounded, the TV-norm of
the solution is necessarily smaller than the one of the true map.
The norm reduction actually increases with the noise power 
since the optimization has then more freedom to reduce the
TV-norm of the solution.  A direct impact for cartoon shape maps is thus a reduction
of RIM dynamics in the reconstruction compared to the expected
one. Studying if this dynamics loss could be limited 
(\eg by constraining the mean) is a matter of future study.

\section{CONCLUSIONS}

We have demonstrated how regularized reconstruction methods, such as
TV-$\ell_2$, can be used in the framework of Optical Deflectometric
Tomography in order to tackle the lack of deflectometric observations
and to make the ODT imaging process robust to noise. The proposed
constrained optimization problem shows significant improvements in the
reconstructions, compared to the well-known Filtered Back
Projection and Minimum Energy methods. The results confirm that, when dealing with a compressive
setting, the Total Variation regularization and the prior information
constraints (positivity and FoV restriction) help in providing a
unique and accurate estimation of the RIM, promoting sharp edges and
preserving the image dynamics (when the modeling noise is limited). By working with the Chambolle-Pock
algorithm we exploit the advantages of proximal operators and of
primal-dual algorithms, and their flexibility to integrate multiple
constraints. We have also shown that the use of the fast NFFT
algorithm efficiently approximates (with a controlled error) the ODT
sensing model involving the polar NDFT.

Noticeably, there still exist some artifacts in the experimental data
reconstruction, coming from the modeling error. In order to handle
this problem, we could no longer assume a linear light
propagation. The actual curved light trajectories, depending
themselves on the RIM through the eikonal equation, could be
then traced. However, such a model improvement quickly leads to
nonconvex optimization procedures and it breaks the fast
computability of the forward sensing operator through the Fourier
domain. Knowing how to insert this new scheme in our reconstruction
method is a matter of future works.

Following the guidelines given in \cite{goldstein2013}, the
  algorithm convergence is greatly improved by making CP adaptive. 
However, as the CP algorithm has been proved to work slowly when the problem is
badly conditioned~\cite{Becker2012,Preconditioning}, preconditioning our global 
operator $\bs K$ can stabilize it and thus improve convergence results. 
We note also that most of the algorithms are implemented in
Matlab, except for the NFFT algorithm which is compiled in
C++. A faster implementation could be reached for instance by
the adopting parallel (GPU) processing techniques, \ie a framework particularly adapted to proximal optimization (\eg see the UNLocBoX project\footnote{\url{http://unlocbox.sourceforge.net}}).  
Such a numerical improvement will be mandatory for addressing 3-D RIM estimation. In this case, the
whole (high dimensional) optimization driving the RIM reconstruction must to be considered in a
(discretized) three dimensional space and fast forward model estimation will have to be developed.

We finally remark that the Optical Deflectometric framework treated in
this paper can be also applied to other imaging techniques, such as
the X-ray phase contrast tomography~\cite{Pfeiffer2007}. With the use
of the energetic X-ray light, the first order paraxial approximation
involving linear light trajectory is no longer needed and the proposed
methods are expected to provide better results.

\section*{ACKNOWLEDGEMENTS} 

Part of this work is funded by the DETROIT project
(WIST3/SPW, Belgium). LJ and CDV are funded by the F.R.S-FNRS.


\appendix

\section{DEFLECTOMETRIC FOURIER SLICE THEOREM} 
\label{app:DFST-proof}

Using the notations of Sec.~\ref{sec:principles}, we must prove that
\begin{equation*}
y(\omega,\theta)\ =\ \tfrac{2\pi i \omega}{\imap_{\rm r}}\,\widehat{\imap}\big(\omega\,\vec p_{\theta}\big).
\end{equation*}

By definition of $y$ in (\ref{eq:radial-fourier-deflecto}), we have
\begin{align*}
y(\omega,\theta)&=\tinv{\imap_{\rm r}}\int_{\Rbb} \int_{\Rbb^2} \big(\vec\nabla \imap(\vec{r})
\cdot \vec p_{\theta}\big)
\ \delta(\tau - \vec{r}\cdot\vec p_{\theta})\, e^{-2\pi i\, \tau\omega} \,\ud^2\vec r\,\ud
\tau\\
&=\tinv{\imap_{\rm r}}\int_{\Rbb^2} \big(\vec\nabla \imap(\vec{r})
\cdot \vec p_{\theta}\big)
\, e^{-2\pi i\, \vec{r}\cdot (\omega \vec p_{\theta})} \,\ud^2\vec r. 
\end{align*}
However, for any function $f:\Rbb\to\Cbb$ integrable on $\Rbb$, 
$$
\widehat{\tfrac{\ud f}{\ud t}}(\omega) = \int_{\Rbb} \tfrac{\ud}{\ud
t}f(t)\, e^{-2\pi i\,t\omega}\ \ud t = (2\pi i)\,\omega\,\widehat f(\omega).
$$
Therefore, we compute easily for any $\bs a,\bs \xi\in\Rbb^2$, 
$$
\int_{\Rbb^2} \big(\vec\nabla \imap(\vec{r})
\cdot \vec a\big)
\, e^{-2\pi i\, \vec{r}\cdot \vec \xi} \,\ud^2\vec r = (2\pi i)\,(\bs \xi
\cdot \bs a)\,\widehat{\imap}(\bs \xi).
$$
Setting $\bs a=\bs p_\theta$ and $\bs \xi = \omega \bs p_\theta$, we
find finally
$$
y(\omega,\theta) = \tinv{\imap_{\rm r}}\int_{\Rbb^2} \big(\vec\nabla \imap(\vec{r})
\cdot \vec p_{\theta}\big)
\, e^{-2\pi i\, \vec{r}\cdot (\omega \vec p_{\theta})} \,\ud^2\vec r =
\tfrac{2\pi i \omega}{\imap_{\rm r}}\,\widehat \imap(\omega \bs p_\theta). 
$$

\section{NON-EQUISPACED FOURIER TRANSFORM} 
\label{app:NFFT}

The non-equispaced Fourier Transform (NFFT) allows us to compute
rapidly, \linebreak \ie in $\cl O(N \log N)$, the NDFT defined in
(\ref{eq:2DNFFT}) of a function defined on $\cl C_N$.  This
computation is performed with a controllable error which can be
further reduced by increasing the computational time.

In a nutshell, the NFFT algorithm replaces the equivalent matrix
multiplication $\bs{\widehat f}=\bs F \bs f$ of (\ref{eq:2DNFFT-matrix})
by $\bs{\widehat s}=\widetilde{\bs F} \bs f$, where $\tilde{\bs F}$ has
fast matrix-vector multiplication computation.  In this scheme, the
discrepancy between $\bs{\widehat f}$ and $\bs{\widehat s}$, as measured by
$E_\infty(\bs f):=\|\bs{\widehat f}- \bs{\widehat s}\|_\infty$, is controlled
and kept small.

More precisely, the matrix $\tilde {\bs F}$ starts by embedding
$\Rbb^N$ in a bigger regular space $\Rbb^n$ of size $n=n_0^2$, with
$n_0> N_0$. This is obtained from the multiplication of
$\bs f$ with a matrix $\bs W \in \Rbb^{n \times N}$ that performs a
weighting of $\bs f$ by a vector $\bs w \in \Rbb_+^N$ (component wise) followed by
a symmetric zero-padding on each side of the function domain $\cl
C_N$. Once in $\Rbb^n$, the common DFT matrix $\bs F_n$ is applied. It
can be computed with the FFT algorithm in order to obtained an
oversampled Fourier transform of $\bs f \in \Rbb^N$.  Finally, a
sparse matrix $\bs V \in \Rbb^{M\times n}$ multiplies the output of
$\bs F_n$ in order to end in the $M$ dimensional space $\widehat{\cl
  P}$. Each row of $\bs V$ corresponds to the translation in the 2-D
Fourier domain of a compact and separable 2-D filter $\psi$ on one
specific point of the non-regular grid $\widehat{\cl P}_M$.  As a final
result, the matrix $\tilde{\bs F}$ is thus factorized in
\cite{KeinerNFFT}
$ 
\tilde{\bs F} = \bs V \bs F_n \bs W.
$

Without entering into unnecessary technicalities, the NFFT scheme is
characterized by a precise connection between the function $\psi$
defining $\bs V$ and the weighting performed by $\bs W$. In
particular, each component of $\bs w$ is actually set to the inverse
of the Fourier transform of a filter $\varphi$, while $\psi$ is a
periodization of (a truncation of) the same filter.

There exist several choices of windows $\varphi$/$\psi$ associated to
different numerical properties (\eg localized support in frequency and
time, simple precomputations of the windows, ...). We select here the
\emph{translates of Gaussian bells} \cite{steidl1998note}, involving a
Gaussian behavior for $\widehat{\varphi}(k)$, which provides fast error
decay for $E_\infty(\bs f)$.

In particular, denoting by $\kappa = n/N \geq 1$ the oversampling
factor and using the FFT for matrix-vector multiplications involving
$\bs F_n$, the total complexity $\cl T$ of the multiplications of
$\tilde{\bs F}$ or $\tilde{\bs F}^\star$ with vectors is
$$
\cl T(\tilde{\bs F}) = \cl O(n \log n + N
(\tfrac{2\kappa -1}{2\kappa -2}) \log 1/\epsilon),
$$
if we impose $E_\infty(\bs f) \leq 4 \|\bs f\|_1\,\epsilon$. For a
fixed $\kappa > 1$, this reduces to $\cl T(\tilde{\bs F}) = \cl O(N \log
N/\epsilon)$, which is far less than a direct computation of the DFT
in $O(MN)$ computation even for small value of $\epsilon$. 

\section{CONVEX CONJUGATE FUNCTIONS}
\label{app:dual-functions}

\subsection{Convex conjugate of the function $G$.}
\label{sec:conv-conj-funct}

Given a vector $\bs x' = (\bs x_1^T, \bs x_2^T)^T\in\Rbb^{2N}$, we can define the
function $G(\bs x') = H(\bs x_1) + \imath_{\Pi_{1,2}} (\bs x_1, \bs x_2)$, with the bisector plane $\Pi_{1,2} = \{ \bs x' : \bs x_1 = \bs x_2 \}$.  Therefore, for a vector $\bs u = (\bs u_1^T, \bs u_2^T)^T\in\Rbb^{2N}$, the dual function $G^\star(\bs x')$ can be computed as follows:
\begin{align*}
  G^\star(\bs x') & = \max_{\bs u} \ \langle \bs u, \bs x' \rangle - G(\bs u) \
   = \max_{\bs u:\ \bs u_1 = \bs u_2} \ \langle \bs u_1, \bs x_1 + \bs x_2 \rangle - H(\bs u_1) \\ \vspace{2mm}
	& = \begin{pmatrix} H^\star(\bs x_1 + \bs x_2)\\H^\star(\bs x_1 + \bs x_2)
  \end{pmatrix} \ = \begin{pmatrix}\Id^N\\\Id^N
  \end{pmatrix} H^\star(\bs x_1 + \bs x_2). 
\end{align*}

\subsection{Convex Conjugate of the indicator function $\imath_{\cl C}$.}
\label{sec:conv-conj-indic}

Given the indicator function $\imath_{\cl C}(\bs u)$ of the convex set
$\cl C = \{ \bs v\in\Rbb^M: \| \bs v - \bs y \| \leq \varepsilon \}$, its dual
function can be computed via the Legendre transform as follows:
\begin{align*}
  \imath_{\cl C}^\star(\bs v)&= \max_{\bs u} \ \scp{\bs v}{\bs u} -
  \imath_{\cl C}(\bs u)\ 
  = \scp{\bs v}{\bs y} + \max_{\bs u : \| \bs u - \bs y \| \leq \varepsilon} \ \scp{\bs v}{\bs u - \bs y}\\
  &= \scp{\bs v}{\bs y} + \max_{\bs b : \| \bs b \| \leq \varepsilon} \ \scp{\bs v}{\bs b}.
\end{align*}
The value of $\{\bs b : \| \bs b \| \leq \varepsilon\}$ that maximizes
the last expression is $\bs b = \frac{\bs v}{\| \bs v \|}
\varepsilon$, and we get $\imath_{\cl C}^\star(\bs v) = \scp{\bs v}{\bs y} +
\varepsilon \| \bs v \|$.

\section{DETAILS ON THE PRODUCT SPACE OPTIMIZATION}
\label{app:Details_ProductSpace}

\subsection{Proximal operator of the function $G$}
\label{app:proxg}

For every $\bs x' = (\bs x_1, \cdots, \bs x_p)\in\Rbb^{pN}$, the function $G(\bs x')$ is defined as: 
$$ G(\bs x') = \sum_{j=2}^{p} \imath_{\Pi_{1,j}} (\bs x') + H(\bs x_1), $$  
where, for $j\in [p]$, $\Pi_{1,j}$ denotes the bisector plane
$\Pi_{1,j} = \{ \bs x' \in\Rbb^{pN}: \bs x_1 = \bs x_j\}$.

Then, for $\bs \zeta = (\bs \zeta_1, \cdots, \bs \zeta_p)\in\Rbb^{pN}$, the proximal operator of $G$ reduces to 
\begin{align*}
  \tilde{\bs x}' = \prox_{\mu G} \bs \zeta &= \argmin_{\bs x':\ \bs x_1 =
  \cdots = \bs x_p}\ \mu H(\bs x_1)\, +\, \tfrac{1}{2}\| \bs \zeta
  - \bs x'\|^2.
\end{align*}
As all the $\bs x_j$ are equal, we necessarily have $ \tilde{\bs x}' =
(\bs \Id^N, \cdots, \bs \Id^N)^T \tilde{\bs u}$ with $\tilde{\bs
u}\in\Rbb^N$ defined by 
\begin{align*}
\tilde{\bs u}& = \argmin_{\bs u\in\Rbb^{N}} \mu H(\bs u) + \tfrac{1}{2} \sum_j \| \bs \zeta_j - \bs u \|^2 \\
&= \argmin_{\bs u\in\Rbb^{N}} \mu H(\bs u) + \tfrac{1}{2} \big [ \| \bs \zeta_1 \|^2
+ \cdots + \| \bs \zeta_p \|^2 - 2 p\,\bar{\bs \zeta}^T\bs u + p \| \bs u \|^2 \big ] \\
&= \argmin_{\bs u\in\Rbb^{N}} \mu H(\bs u) + \tfrac{1}{2} \big [ p\|
\bar{\bs \zeta}\|^2 - 2 p\,\bar{\bs \zeta}^T\bs u + p \| \bs u \|^2 \big ]\\
&= \prox_{(\mu/p) H} \bar{\bs \zeta}.
\end{align*}
with $\bar{\bs \zeta} = \tinv{p} \sum_j \bs \zeta_j \in \Rbb^N$ and
where we used the fact that we can always subtract or add constants in
the minimization without disturbing its solution.  Denoting by
$\Id_{p}^N = (\Id^N, \cdots, \Id^N)\in\Rbb^{N\times pN}$ the operator
such that $\bs \Id_{p}^N \bs \zeta = p \bar{\bs \zeta}$, this provides
finally the compact notation
$$ \prox_{\mu G} \bs \zeta  = (\bs \Id_{p}^N)^T\, \prox_{\frac{\mu}{p} H}\, \tinv{p} \bs \Id_{p}^N \bs \zeta. $$

\subsection{Formulation of Chambolle-Pock algorithm in OD}
\label{app:FinalFormulation}

We take the CP algorithm as described in Eq.~\eqref{eq:CP_Algorithm}
and we translate it into our OD problem in the product space, to
obtain the following:
$$
\begin{cases}
  \bs s^{(k+1)}\!\!\!\!&= \begin{pmatrix}
    \bs s_1^{{(k+1)}} \\
    \bs s_2^{{(k+1)}}
      \end{pmatrix}
      = \begin{pmatrix}
        \prox_{\nu F_1^\star} \big(\bs s_1^{{(k)}} + \nu \bs \nabla \bar{\bs x}_1^{{(k)}} \big) \\
        \prox_{\nu F_2^\star} \big(\bs s_2^{{(k)}} + \nu \bs \Phi \bar{\bs x}_2^{{(k)}} \big)
      \end{pmatrix},
      \\[4mm]
      \bs x'^{(k+1)}\!\!\!\!&= \begin{pmatrix}
        \bs x_1^{{(k+1)}} \\
        \bs x_2^{{(k+1)}}
      \end{pmatrix}
      = \begin{pmatrix}
        \bs\Id^N \\
        \bs\Id^N
      \end{pmatrix}
      \prox_{\tfrac{\mu}{2} H} \tinv{2}(\bs x_1^{{(k)}} - \mu \bs \nabla^* \bs s_1^{{(k+1)}} + \bs x_2^{{(k)}} - \mu \bs \Phi^* \bs s_2^{{(k+1)}}),\\[4mm]
      \bar{\bs x}'^{(k+1)}\!\!\!\!&= \begin{pmatrix}
        \bar{\bs x}_1^{{(k+1)}} \\
        \bar{\bs x}_2^{{(k+1)}}
      \end{pmatrix}
      = 2 \begin{pmatrix}
        \bs x_1^{{(k+1)}} \\
        \bs x_2^{{(k+1)}}
      \end{pmatrix}
      -  \begin{pmatrix}
        \bs x_1^{{(k)}} \\
        \bs x_2^{{(k)}}
      \end{pmatrix}.
    \end{cases}
$$
We have the function $H(\bs x_1) = \imath_{\cl P_0} (\bs x_1)$ and its proximal operator $\prox_{\frac{\mu}{2} H} \bs \zeta = \textrm{proj}_{\cl P_0} \bs \zeta$. As $\bs x_1 = \bs x_2$, we can relabel the variable as $\bs x^{(k)} = \bs x_1^{{(k)}} = \bs x_2^{{(k)}}$ and $\bs x^{(k+1)} = \bs x_1^{{(k+1)}} = \bs x_2^{{(k+1)}}$. In the same way, $\bar{\bs x}_1 = \bar{\bs x}_2$, thus we can relabel the variable as $\bar{\bs x}^{(k)} = \bar{\bs x}_1^{{(k)}} = \bar{\bs x}_2^{{(k)}}$ and $\bar{\bs x}^{(k+1)} = \bar{\bs x}_1^{{(k+1)}} = \bar{\bs x}_2^{{(k+1)}}$. We obtain the
following algorithm:

$$
\begin{cases}
  \bs s_1^{{(k+1)}}&= \prox_{\nu F_1^\star} (\bs s_1^{{(k)}} + \nu \bs \nabla \bar{\bs x}^{(k)} ),\\[2mm]
  \bs s_2^{{(k+1)}}&= \prox_{\nu F_2^\star} (\bs s_2^{{(k)}} + \nu \bs \Phi \bar{\bs x}^{(k)} ),\\[2mm]
  \bs x^{(k+1)}&= \textrm{proj}_{\cl P_0} \big(\bs x^{(k)} - \frac{\mu}{2} (\bs \nabla^* \bs s_1^{{(k+1)}} + \bs \Phi^* \bs s_2^{{(k+1)}}) \big),\\[2mm]
  \bar{\bs x}^{(k+1)}&= 2 \bs x^{(k+1)} - \bs x^{(k)}.
\end{cases}
$$

\end{document}